\documentclass[11pt]{article}

\usepackage{microtype}
\usepackage{graphicx}
\usepackage{subfigure}
\usepackage{booktabs} 
\usepackage{natbib}
\usepackage{epsfig,epsf,fancybox}
\usepackage{amsmath}
\usepackage{mathrsfs}
\usepackage{amssymb}
\usepackage{color}
\usepackage{multirow}
\usepackage{paralist}
\usepackage{verbatim}
\usepackage{algorithm}
\usepackage{algorithmic}
\usepackage{galois}
\usepackage{boxedminipage}
\usepackage{accents}
\usepackage{stmaryrd}
\usepackage[bottom]{footmisc}

\usepackage{natbib}

\usepackage{url}
\usepackage[colorlinks,linkcolor=magenta,citecolor=blue, pagebackref=true,backref=true]{hyperref}
\renewcommand*{\backrefalt}[4]{%
    \ifcase #1 \footnotesize{(Not cited.)}%
    \or        \footnotesize{(Cited on page~#2.)}%
    \else      \footnotesize{(Cited on pages~#2.)}%
    \fi}

\newtheorem{theorem}{Theorem}[section]
\newtheorem{corollary}[theorem]{Corollary}
\newtheorem{lemma}[theorem]{Lemma}

\newtheorem{definition}{Definition}[section]

\newtheorem{assumption}[theorem]{Assumption}

\usepackage{hyperref}

\newcommand{\EE}{\mathbb{E}}

\newcommand{\sign}{\textnormal{sign}}

\newcommand{\RR}{\mathbf R}

\newcommand{\w}{\mathbf w}

\newcommand{\argmin}{\mathop{\rm argmin}}

\newcommand{\ba}{\begin{array}}
\newcommand{\ea}{\end{array}}

\newcommand{\red}{\color{red}}

\numberwithin{equation}{section}
\def\red#1{}\def\pb{}
\def\blue#1{\textcolor{black}{#1}}
\usepackage{enumitem} 
\def\beq{\begin{equation} }\def\eeq{\end{equation} }\def\1{\mathbf{1}}\usepackage{bm} 
\newcommand{\Ab}{\bm{A}}
\newcommand{\bX}{\bm{X}}
\newcommand{\bY}{\bm{Y}}
\newcommand{\bZ}{\bm{Z}}
\newcommand{\cD}{\mathscr{D}}
\newcommand{\ub}{\mathbf{u}}
\def\cI{\mathcal{I}}
\newcommand{\Ib}{\mathbf{I}}
\newcommand{\cH}{\mathcal{H}}
\newcommand{\PP}{\mathbb{P}}
\newcommand{\cO}{O}
\newcommand{\vb}{\mathbf{v}}
\def\Pb{\mathbf{P}}
\newcommand{\eb}{\mathbf{e}}
\newcommand{\cT}{{\mathcal{T}}}
\newcommand{\cF}{{\mathcal{F}}}
\def\ab{\mathbf{a}}
\def\cS{\mathcal{S}}
\def\rB{\mathscr{B}}
\newcommand{\ud}{{\mathrm{d}}}
\def\u{u}
\def\cM{\mathcal{M}}
\def\cB{\mathcal{B}}
\def\cG{\mathcal{G}}
\begin{document}


\begin{center}

{\bf{\LARGE{Stochastic Approximation for Online Tensorial Independent Component Analysis}}}

\vspace*{.2in}
{\large{
\begin{tabular}{cccc}
Chris Junchi Li$^{\diamond}$
&
Michael I.~Jordan$^{\diamond, \dagger}$ \\
\end{tabular}
}}

\vspace*{.2in}

\begin{tabular}{c}
Department of Electrical Engineering and Computer Sciences$^\diamond$\\
Department of Statistics$^\dagger$ \\ 
University of California, Berkeley
\end{tabular}

\vspace*{.2in}

\today

\vspace*{.2in}

\begin{abstract}
Independent component analysis (ICA) has been a popular dimension reduction tool in statistical machine learning and signal processing. In this paper, we present a convergence analysis for an online tensorial ICA algorithm, by viewing the problem as a nonconvex stochastic approximation problem. For estimating one component, we provide a dynamics-based analysis to prove that our online tensorial ICA algorithm with a specific choice of stepsize achieves a sharp finite-sample error bound. In particular, under a mild assumption on the data-generating distribution and a scaling condition such that $d^4/T$ is sufficiently small up to a polylogarithmic factor of data dimension $d$ and sample size $T$, a sharp finite-sample error bound of $\tilde{O}(\sqrt{d/T})$ can be obtained.
\end{abstract}
\end{center}

\paragraph{Keywords:}
Independent component analysis, tensor decomposition, non-Gaussianity, finite-sample error bound, online learning

\pb\section{Introduction}\label{sec:intro}

Independent Component Analysis (ICA) is a widely used dimension reduction method with diverse applications in the fields of statistical machine learning and signal processing \citep{HYVARINEN-KARHUNEN-OJA,stone2004independent,samworth2012independent}.
Let the data vector be modeled as $\bX = \Ab \bZ$, where $\Ab \equiv (\ab_1,\dots,\ab_d) \in \RR^{d \times d}$ is a full-rank mixing matrix whose columns are orthogonal components in $\RR^d$, and $\bZ \in \RR^d$ is a non-Gaussian latent random vector consisting of independent entries $\bZ = (Z_1, \dots, Z_d)^\top$.
The goal of ICA is to recover one or multiple columns among $(\ab_1,\dots,\ab_d)$ from independent observations of $\bX = (X_1, \dots, X_d)^\top$.
Following standard practice, we assume that the random vector $\bX$ has been whitened in the sense that it has zero mean and an identity covariance matrix \citep{HYVARINEN-KARHUNEN-OJA}, and we focus on the case where the distributions of $Z_1,\dots,Z_d$ share a fourth moment $\mu_4 \ne 3$ (i.e., they are of identical kurtosis, $\mu_4-3$).
These assumptions restrict the search of mixing matrix $\Ab$ to the space of orthogonal matrices and guarantee its identifiability up to signed permutations of its columns $(\ab_1, \dots, \ab_d)$ \citep{comon1994independent,frieze1996learning,HYVARINEN-KARHUNEN-OJA}.

In this paper, we study a stochastic algorithm that estimates independent components for streaming data.
Such an algorithm processes and discards one or a small batch of data observations at each iterate and enjoys reduced storage complexity.
To begin with, we cast the tensorial ICA as the problem of optimizing an objective function based on the fourth-order cumulant tensor over the unit sphere $\cD_1 \equiv \{\ub\in\RR^d: \|\ub\| = 1\}$, where $\|\cdot\|$ denotes the Euclidean norm.
The optimization problem is as follows:
\beq\label{eq:opt}
\begin{aligned}
&
\text{min }
- \sign(\mu_4 - 3) \EE(\ub^\top \bX)^4
\\&
\text{subject to }~
\ub\in \cD_1.
\end{aligned}\eeq
Such an objective function is referred to as a \emph{non-Gaussianity contrast function} \citep{HYVARINEN-KARHUNEN-OJA}.
The formulation \eqref{eq:opt} of ICA via its fourth-order cumulant tensor was initiated by \citet{comon1994independent} and \citet{frieze1996learning} in the batch case.
\blue{The landscape of objective \eqref{eq:opt} is featured by its nonconvexity, in the sense that it presents $2d$ local minimizers $\pm \ab_1, \dots, \pm \ab_d$ and exponentially many unstable stationary points in terms of $d$ \citep{ge2015escaping,li2016online,sun2015nonconvex}.\footnote{\blue{In contrast, the rank-one eigenvalue problem presents $2$ local minimizers and $2d-2$ unstable stationary points in the typical setting of distinct eigenvalues.}}
Here, by unstable stationary points we refer to those points with zero gradient and at least one negative Hessian direction, and hence includes the collection of saddle points and local maximizers.}
Algorithms that find local minimizers of \eqref{eq:opt} allow us to estimate the columns of the mixing matrix $\Ab$.
We analyze and discuss the following stochastic approximation method, which we refer to as \emph{online tensorial ICA}   
\citep{ge2015escaping,li2016online,wang2017scaling}.
Initialize a unit vector $\ub^{(0)}$ appropriately, at step $t = 1,2,\dots, T$ the algorithm processes an observation $\bX^{(t)}$ by performing the following update:
\beq\label{tensorSGD}
\ub^{(t)}
 = 
\Pi_1\left\{
\ub^{(t - 1)}
+
\eta^{(t)} \cdot \sign(\mu_4 - 3) \left((\ub^{(t - 1)})^\top \bX^{(t)}\right)^3 \bX^{(t)}
\right\}
,
\eeq
where $\eta^{(t)}$ is a positive stepsize, and the operator $\Pi_1[\bullet] := \bullet / \|\bullet\|$ projects a nonzero vector onto the unit sphere $\cD_1$ centered at the origin.

\blue{Due to the nonconvexity of the optimization problem \eqref{eq:opt}, a key issue that arises in analyzing the iteration \eqref{tensorSGD} is the avoidance of unstable stationary points including saddle points and local maximizers}.
In earlier work, \citet{ge2015escaping} introduced an additional artificial noise injection step to the algorithm and developed a hit-and-escape convergence analysis, obtaining a polynomial convergence rate for the online tensorial ICA problem and beyond (for a more general class of nonconvex optimization problems).
In contrast, we present a \blue{dynamics-based} analysis that requires no noise injection step.
We show that a uniform initialization on the unit sphere $\cD_1$ along with mild scaling conditions, is sufficient to ensure that the algorithm iterate enters a basin of attraction and finds a local (and also global) minimizer with high probability.

\paragraph{Overview of the Main Result and Contributions}
Our main result states that for estimating a single independent component, the iteration \eqref{tensorSGD} with uniform initialization and carefully chosen stepsize $\eta^{(t)}$ enters the basin of attraction of a uniformly drawn independent component $\pm \ab_\cI$ and achieves a convergence rate of $\tilde{O}(\sqrt{d/T})$ in terms of angle with respect to an independent component, up to a polylogarithmic factor of $d$ and $T$.
Such a convergence result holds under mild distributional assumptions and scaling conditions that involve the data dimension $d$ and sample size $T$.
Informally, this result is stated as follows.

\vspace{-.1in}
\begin{theorem}[Informal version of Corollary \ref{coro:two_phase} in \S\ref{sec:uniformsec}]\label{theo:informal}
Let the initial $\ub^{(0)}$ be  sampled uniformly from the unit sphere $\cD_1$, and let appropriate \blue{light-tailed} distributional assumptions hold.
Then, for any fixed positive $\epsilon \in (0, 1 / 5]$ satisfying the scaling condition:
$$
d
\ge
2\sqrt{2\pi e} \log\epsilon^{-1} + 1
,\qquad
C_{\ref{theo:informal}, T}
\log^8(C_{\ref{theo:informal}, T} \epsilon^{-1} d T)
\cdot
\frac{B^8}{|\mu_4 - 3|^2}
\cdot
\frac{d^4 \log^2 T}{T}
\le
\frac{\epsilon^2}{ \log^2\epsilon^{-1}}
,
$$
there exists a uniformly distributed random variable $\cI \in [d]$ such that with probability at least $1 - 5\epsilon$, iteration $\ub^{(t)}$ of \eqref{tensorSGD} with appropriate choice of stepsizes $\eta^{(t)}$ satisfies
$$
\left| \tan \angle\left(\ub^{(T)}, \ab_\cI\right) \right|
	\le
C_{\ref{theo:informal}, T} \log^{5 / 2}(C_{\ref{theo:informal}, T} \epsilon^{-1} d)
\cdot
\frac{B^4}{|\mu_4 - 3|}
\cdot
\sqrt{\frac{d \log^2 T}{T}}
,
$$
where $C_{\ref{theo:informal}, T}$ is a positive, absolute constant.
\end{theorem}
\blue{
To the best of our knowledge, this provides the first rigorous analysis of an online tensorial ICA algorithm that achieves an $\tilde{O}(\sqrt{d/T})$ finite-sample convergence rate, under mild distributional assumptions and scaling conditions.
The contributions of this work lie in several aspects:
\begin{enumerate}[label=(\roman*)]
\item
Partly adapting from the analysis of online principal component estimation in \citet{li2018near}, we provide a per-coordinate analysis of the warmly-initialized online tensorial ICA algorithm that achieves a sharp $\tilde{O}(\sqrt{d/T})$ convergence rate
[Theorem \ref{theo:finite_sample2}].
\item
In contradistinction to existing saddle-point-escaping analysis in \citet{ge2015escaping,jin2017escape,li2018near}, we developed a novel coordinate-pair analysis of the uniformly-initialized online tensorial ICA algorithm based on our dynamics-based characterization
[Theorem \ref{theo:finite_sample}].
\item
We combine these two analyses to conclude that a two-stage training procedure provides a finite-sample error bound of $\tilde O(\sqrt{d / T})$ for the uniform initialization case under a mild assumption on the data-generating distribution and also a scaling condition of $d^4 / T$ being sufficiently small, up to a polylogarithmic factor of $d, T$
[Theorem \ref{theo:informal}, presented formally in Corollary \ref{coro:two_phase}].
\end{enumerate}
}

\paragraph{Organization}
The rest of this paper is organized as follows.
\S\ref{sec:warmsec} presents our main convergence results and finite-sample error bounds for the warm initialization case for estimating one single component.
\S\ref{sec:uniformsec} presents the corresponding results for the uniform initialization case.
\S\ref{sec:related} discusses additional related literature.
\S\ref{sec:summary} summarizes our results.
Limited by space we relegate to the Appendix the proofs of main results, all secondary lemmas and technical results along with preliminary simulation results.

\begin{algorithm}[!tb]
\caption{Online Tensorial ICA}
\label{algo:ICA}
\begin{algorithmic}
\STATE Initialize $\ub^{(0)}$ and select stepsize $\eta^{(t)}$ appropriately (to elaborate later)
\FOR{$t = 1, 2, \dots$}
\STATE Draw one observation $\bX^{(t)}$ from streaming data, and update iteration $\ub^{(t)}$ via
\beq\label{tensorialSGD}
\ub^{(t)}
=
\Pi_1\left\{
\ub^{(t - 1)} + \eta \cdot \sign(\mu_4 - 3) \left((\ub^{(t - 1)})^\top \bX^{(t)}\right)^3 \bX^{(t)}
\right\}
\eeq
where $\Pi_1\{\bullet\} = \|\bullet\|^{-1} \bullet$ denotes the projection operator onto the unit sphere centered at the origin $\cD_1$
\ENDFOR
\end{algorithmic}
\end{algorithm}

\pb\section{Estimating a Single Component: Warm Initialization Case}\label{sec:warmsec}
For the purpose of estimating a single independent component, we introduce our settings and assumptions for tensorial ICA and its stochastic approximation algorithm \eqref{tensorSGD}, formally stated in Algorithm \ref{algo:ICA}.
Let the dimension $d \ge 2$, let $\bX$ be the data vector of which $\bX^{(1)}, \bX^{(2)}, \ldots \in \RR^d$ are independent draws, and assume the following for the distribution of $\bX$:

\begin{assumption}[Data vector distribution]\label{assu:distribution}
Let $\bX = \Ab \bZ$, where $\Ab\in \RR^{d\times d}$ is an orthogonal matrix with $\Ab \Ab^\top = \Ib$ and $\bZ \in \RR^{d}$ is a random vector satisfying 
\begin{enumerate}[label=(\roman*)]
\item\label{ass_fir}
The $Z_i, i=1,\dots,d$ are independent with identical $j$th-moment for $j=1,2,4$, denoted as $\mu_j \equiv \EE Z_i^j$;
\item\label{ass_sec}
The $\mu_1 = \EE Z_i = 0$, $\mu_2 = \EE Z_i^2 = 1$, $\mu_4 = \EE Z_i^4 \ne 3$;
\item\label{ass_sub}
For all $i \in [d]$, $Z_i$ has an Orlicz-$\psi_2$ norm bounded by $\sqrt{3/8} B$.\footnote{
A random variable $Z$ with mean 0 is light-tailed if there is a positive number $\sigma$ such that
$\EE \exp(\lambda Z) \le \exp(\sigma^2 \lambda^2 / 2)$ for all $\lambda \in \RR$.
\blue{The smallest possible $\sigma > 0$ is referred to as the Orlicz-$\psi_2$ norm or sub-Gaussian parameter \citep{wainwright2019high}.
Readers shall be warned that the term ``sub-Gaussian'' here indicates the light-tailed condition and should be distinguished from $\mu_4<3$ (resp.~$\mu_4>3$) of sub-Gaussianity (resp.~super-Gaussianity) in ICA \citep{stone2004independent}. }}
\end{enumerate} 
\end{assumption}
Note that Assumption \ref{assu:distribution}(i) requires the distribution for all independent components to admit identical first, second and fourth moments.
As indicated in Assumption \ref{assu:distribution}(ii), the data vectors are assumed to be whitened first in the sense that $\mu_2 = 1$.
The sign of our excess kurtosis $\mu_4 - 3$ determines the direction of stochastic gradient update, and, as  will be seen later, the magnitude of the excess kurtosis $|\mu_4 - 3|$ plays an important role in our convergence analysis.\footnote{
When the excess kurtosis $\mu_4 - 3$ is equal to zero, for instance when $\bZ$ follows an i.i.d.~standard normal distribution, the matrix $\Ab$ is non-identifiable in our framework.
Non-gaussian independent components with $\mu_4 = 3$ can be studied via higher-order tensor decomposition with a different contrast function, but is beyond the scope of this paper.
}
Assumption \ref{assu:distribution}(iii) generalizes the boundedness assumption $\|Z_i\|_\infty \le O(B)$ made in recent work in order to cover Gaussian mixtures and Bernoulli-Gaussians, which are typical application cases for the tensorial ICA estimation problem.
Note that we include a factor of $\sqrt{3/8}$ in the \blue{Orlicz-$\psi_2$ norm (sub-Gaussian parameter)} purely for notational simplicity in our analysis.

We target to study the convergence of tensorial ICA under certain initialization conditions.
For each initialization condition, we first analyze the convergence result for any fixed, plausible stepsizes, and then (by choosing the stepsize according to the number of observations) obtain the finite-sample error bound.
We focus in this section the warm initialization condition as any $\ub^{(0)}$ satisfying, for some integer $i\in [d]$,
\beq\label{eq:lemm_warm}
\ub^{(0)}\in \cD_1
	\quad\text{and}\quad
\left| \tan \angle\left(\ub^{(0)}, \ab_i\right) \right| \le \frac{1}{\sqrt{3}}
.
\eeq
For any fixed $\tau > 0$, we define a rescaled time variable $T_{\eta, \tau}^*$ as
\beq\label{eq:Tetatau*}
T_{\eta, \tau}^* \equiv \left\lceil
	\frac{\tau \log\left(\frac{|\mu_4 - 3|}{B^8} \cdot \eta^{-1}\right)}{- \log\left( 1 - \frac{\eta}{3} |\mu_4 - 3| \right)}
\right\rceil
. 
\eeq
Then under warm initialization condition \eqref{eq:lemm_warm}, we have the following convergence lemma.

\begin{lemma}[Convergence Result with Warm Initialization]\label{lemm:convergence_result2}
Let the dimension $d \ge 2$, let Assumption \ref{assu:distribution} hold, and let initialization $\ub^{(0)}$ satisfy condition \eqref{eq:lemm_warm} for some integer $i\in [d]$.
Then, for any fixed positive numbers $\tau, \eta$ and $\delta \in (0, e^{-1}]$ satisfying the scaling condition,
\beq\label{eq:warm_scaling}
C_{\ref{lemm:convergence_result2}, L}^* \log^8(T_{\eta, 1}^* \delta^{-1})
\cdot
\frac{B^8}{|\mu_4 - 3|}
\cdot
d \eta \log\left(\frac{|\mu_4 - 3|}{B^8}\eta^{-1}\right)
\le
\frac{1}{\tau + 1}
,\qquad
\eta
<
\min\left(
\frac{1}{|\mu_4 - 3|}
,
\frac{|\mu_4 - 3|}{B^8} e^{-1}
\right)
,
\eeq
there exists an event $\cH_{\ref{lemm:convergence_result2}, L}$ with
$
\PP(\cH_{\ref{lemm:convergence_result2}, L})
\ge
1 - \left(
6\tau + 12 + \frac{5184}{\log^5\delta^{-1}}
\right) d \delta
,
$
such that on $\cH_{\ref{lemm:convergence_result2}, L}$, iteration $\ub^{(t)}$ of Algorithm \ref{algo:ICA} satisfies 
\beq\label{eq:lemm:convergence_result-tanbdd}
\begin{aligned}
\left| \tan \angle\left(\ub^{(t)}, \ab_i\right) \right|
&\le
\left| \tan \angle\left(\ub^{(0)}, \ab_i\right) \right|
\left(  1 - \frac{\eta}{3} |\mu_4 - 3| \right)^t
\\&\quad+
\sqrt{\tau + 1} C_{\ref{lemm:convergence_result2}, L} \log^{5 / 2}\delta^{-1}
\cdot
\frac{B^4}{|\mu_4 - 3|^{1 / 2}}
\cdot
\sqrt{
d \eta \log\left(\frac{|\mu_4 - 3|}{B^8}\eta^{-1}\right)
}
,
\end{aligned}\eeq
for all $t \in [0, T_{\eta, \tau}^*]$, where $C_{\ref{lemm:convergence_result2}, L}$ and $C_{\ref{lemm:convergence_result2}, L}^*$ are positive, absolute constants.
\end{lemma}
Lemma \ref{lemm:convergence_result2} provides, under the warm initialization assumption \eqref{eq:lemm_warm} and scaling condition $\eta = \tilde\cO(d^{-1})$, an upper bound for $\left| \tan \angle\left(\ub^{(t)}, \ab_i\right) \right|$ which is the sum of two terms:
the first term on the right hand side of \eqref{eq:lemm:convergence_result-tanbdd} decays geometrically from $\left| \tan \angle\left(\ub^{(0)}, \ab_i\right) \right|$ at rate $1 - |\mu_4 - 3| \eta / 3$, and the second term $\tilde\cO(\sqrt{d\eta})$ is induced by the noise.
To balance these two terms, when we know in advance the sample size $T$ of online data satisfying some scaling condition $T = \tilde\Theta(d)$, we choose a constant stepsize $\eta = \tilde\Theta(\log T / T)$ and obtain  a finite-sample error bound:

\begin{theorem}[Finite-Sample Error Bound with Warm Initialization]\label{theo:finite_sample2}
Let the dimension $d \ge 2$, let Assumption \ref{assu:distribution} hold, and let initialization $\ub^{(0)}$ satisfy condition \eqref{eq:lemm_warm} for some integer $i\in [d]$.
For sample size $T$ set the stepsize $\eta$ as
\beq\label{etaPickT}
\eta(T) = \frac{9 \log\left(\frac{2 |\mu_4 - 3|^2}{9 B^8} T\right)}{2 |\mu_4 - 3| T}
.
\eeq
Then, for any fixed positive numbers $T \ge 100, \epsilon \in (0, 1]$ satisfying the scaling condition 
\beq\label{eq:warm_scaling2}
C_{\ref{theo:finite_sample2}, T}^* \log^8(C_{\ref{theo:finite_sample2}, T}' \epsilon^{-1} d T)
\cdot
\frac{d \log^2 T}{T}
\le
\frac{|\mu_4 - 3|^2}{B^8}
,
\eeq
there exists an event $\cH_{\ref{theo:finite_sample2}, T}$ with $\PP(\cH_{\ref{theo:finite_sample2}, T}) \ge 1 - \epsilon$ such that on $\cH_{\ref{theo:finite_sample2}, T}$, iteration $\ub^{(t)}$ of Algorithm \ref{algo:ICA} satisfies
$$
\left| \tan \angle\left(
\ub^{(T)}, \ab_i
\right) \right|
\le
C_{\ref{theo:finite_sample2}, T} \log^{5 / 2}(C_{\ref{theo:finite_sample2}, T}' \epsilon^{-1} d)
\cdot
\frac{B^4}{|\mu_4 - 3|}
\cdot
\sqrt{
\frac{d \log^2 T}{T}
}
,
$$
where $C_{\ref{theo:finite_sample2}, T}, C_{\ref{theo:finite_sample2}, T}^*, C_{\ref{theo:finite_sample2}, T}'$ are positive, absolute constants.
\end{theorem}
Theorem \ref{theo:finite_sample2} indicates an $\tilde\cO\left( \sqrt{d/T} \right)$ finite-sample error bound for online tensorial ICA when it is warmly initialized in the sense that \eqref{eq:lemm_warm} holds for $\ub^{(0)}$ for some integer $i\in[d]$.
In the rest of this section, we prove Lemma \ref{lemm:convergence_result2} and Theorem \ref{theo:finite_sample2} for the warm initialization case.  The section is organized as follows:
\S\ref{sec:warm_key} analyzes our algorithm and provide a key lemma [Lemma \ref{lemm:fast_expo}] when it is warmly initialized.
Proof of the key Lemma \ref{lemm:fast_expo} and all proofs of secondary lemmas are deferred to \S\ref{sec:warm_lemm} and \S\ref{sec:warm_aux} in Appendix.

\pb\subsection{Key Lemma in the Warm Initialization Analysis}\label{sec:warm_key}
To simplify our problem, we set $i\in [d]$ as the integer such that $\ub^{(0)}$ satisfies condition \eqref{eq:lemm_warm} and define the rotated iteration $\{\vb^{(t)}\}_{t \ge 0}$ as
\beq\label{eq:v_def}
\vb^{(t)}
\equiv
\Pb \Ab^\top \ub^{(t)}
,
\eeq
where $\Pb \in \RR^{d\times d}$ is the permutation matrix corresponding to the cycle $(1i)$; i.e.,~$\Pb(i;1) = \Pb(1;i) = 1$, $\Pb(j;j) = 1$ for $j\ne 1,i$ and all other elements being zero.
Such a matrix, as an operator, maps the component vectors to coordinate vectors and ensures that $\pm \eb_1$ is the closest independent components pair at initialization and (with high probability) at convergence.
Furthermore, transforming to rotated observations $\bY^{(t)} = \Pb \Ab^\top \bX^{(t)}$ allows us to equivalently translate our online tensorial ICA iteration \eqref{tensorialSGD} into an analogous form:
\beq\label{eq:rescaled_update}
\vb^{(t)}
=
\Pi_1\left\{
\vb^{(t - 1)} + \eta\cdot \sign(\mu_4 - 3) \left( (\vb^{(t - 1)})^\top \bY^{(t)} \right)^3 \bY^{(t)}
\right\}
.
\eeq
It is easy to verify that the rotated iterations $\{\vb^{(t)}\}_{t \ge 0}$ and $\{\ub^{(t)}\}_{t \ge 0}$ satisfy $\ab_i^\top \ub^{(t)} = \eb_1^\top \vb^{(t)}$, and hence for all $t \ge 0$
\beq\label{eq:uv}
\tan \angle\left( \vb^{(t)}, \eb_1\right)
=
\frac{\sqrt{1 - (\eb_1^\top \vb^{(t)})^2} }{\eb_1^\top \vb^{(t)}}
=
\frac{\sqrt{1 - (\ab_i^\top \ub^{(t)})^2} }{\ab_i^\top \ub^{(t)}}
=
\tan \angle\left( \ub^{(t)}, \ab_i\right)
.
\eeq

Now, we let the warm initialization region be
\beq\label{eq:Swarm}
\cD_{\text{warm}}
 =
\left\{
\vb \in \cD_1 : v_1^2 \ge \frac34
\right\}
 =
\left\{
\vb \in \cD_1 :
\big| \tan \angle\left(\vb, \eb_1\right) \big| \le \frac{1}{\sqrt{3}}
\right\}
.
\eeq
Note that the warm initialization condition in \eqref{eq:lemm_warm} is equivalent to $\vb^{(0)}\in \cD_{\text{warm}}$.
Analogous to the warm initialization studied in the setting of principal component estimation \citep{li2018near}, the iteration we study in our warmly initialized online tensorial ICA is
\beq\label{eq:Uk}
U_k^{(t)} \equiv \frac{v_k^{(t)}}{v_1^{(t)}}
.
\eeq
To prevent iteration $U_k^{(t)}$ from diverging from the warm initialization region, we also define a larger warm-auxiliary region as
$
\cD_{\text{warm-aux}}
\equiv
\left\{\vb \in \cD_1:  v_1^2 \ge \frac23\right\}
=
\left\{ \vb \in \cD_1:  \big| \tan \angle\left(\vb, \eb_1\right) \big|  \le  \frac{1}{\sqrt{2}} \right\}
$.
Suppose the process $\{\vb^{(t)}\}$ is initialized at $\vb^{(0)} \in \cD_{\text{warm}}$, and we define the first time $\vb^{(t)}$ exits the warm-auxiliary region as
\beq\label{eq:cTx}
\cT_x
\equiv
\inf \left\{t\ge 1: \vb^{(t)}\in \cD_{\text{warm-aux}}^c\right\}
.
\eeq
We state the key lemma as

\begin{lemma}\label{lemm:fast_expo}
Let the settings in Lemma \ref{lemm:convergence_result2} hold, and fix the coordinate $k\in [2,d]$ and value $\tau > 0$.
Then for any fixed positives $\eta, \delta$ satisfying the scaling condition \eqref{eq:warm_scaling} along with the warm initialization condition $\vb^{(0)}\in \cD_{\text{warm}}$, there exists an event $\cH_{k; \ref{lemm:fast_expo}, L}$ satisfying
$
\PP(\cH_{k; \ref{lemm:fast_expo}, L})
\ge
1 - \left(6\tau + 12 + \frac{5184}{\log^5\delta^{-1}}\right)\delta
,
$
such that on event $\cH_{k; \ref{lemm:fast_expo}, L}$ the following holds:
\beq\label{eq:prop:fast_expo1}
\sup_{t\le T_{\eta, \tau}^*  \land \cT_x} \left|
 U_k^{(t)}  - U_k^{(0)} \prod_{s = 0}^{t - 1} \left[ 1 - \eta |\mu_4 - 3| \left( (v_1^{(s)})^2 - (v_k^{(s)})^2 \right) \right]
\right| 
\le 
2 C_{\ref{lemm:fast_expo}, L} \log^{5 / 2}\delta^{-1}
\cdot
B^4
\cdot
\eta (T_{\eta, \tau}^*)^{1 / 2}
,
\eeq
where $C_{\ref{lemm:fast_expo}, L}$ is a positive, absolute constant.
\end{lemma}
This lemma shows that for each coordinate $k\in [2,d]$, with high probability, the \blue{dynamics} of $U_k^{(t)}$ is tightly controlled within a deterministic vessel whose center converges to zero at least exponentially fast.
As we will later see in the proofs of Lemma \ref{lemm:convergence_result2} and Theorem \ref{theo:finite_sample2}, this guarantees that with high probability $\vb^{(t)}$ will not exit the warm-auxilliary region $\cD_{\text{warm-aux}}$ and will stay within a small neighborbood of $\pm \eb_1$ after $T_{\eta, 0.5}^*$ iterates, where the rescaled time $T_{\eta, \tau}^*$ was earlier defined in \eqref{eq:Tetatau*}.

\pb\section{Estimating Single Component: Uniform Initialization Case}\label{sec:uniformsec}
It is often the case that we are unable to obtain a warm initialization for online tensorial ICA as required in \S\ref{sec:warmsec}, in which case we resort to initializing $\ub^{(0)}$ uniformly at random from the unit sphere $\cD_1$.
To analyze such a case, we define a new rescaled time variable, parameterized by $\tau > 0$, as follows: 
\beq\label{Tetatau}
T_{\eta, \tau}^o
\equiv
\left\lceil
\frac{\tau \log\left(\frac{|\mu_4 - 3|}{B^8} \cdot \eta^{-1}\right)}{- \log\left( 1 - \frac{\eta}{2d} |\mu_4 - 3| \right)}
\right\rceil
.
\eeq
Then under uniform initialization, we have the following convergence result:

\begin{lemma}[Convergence Result with Uniform Initialization]\label{lemm:convergence_result}
Let the dimension $d \ge 2$, let Assumption \ref{assu:distribution} hold, and let $\ub^{(0)}$ be uniformly sampled from the unit sphere $\cD_1$.
Then for any fixed positive numbers $d \ge 2, \tau > 0.5, \eta > 0, \delta \in (0, e^{-1}], \epsilon \in (0, 1 / 3]$ satisfying the scaling condition 
\beq\label{eq:uniform_scaling}
\begin{aligned}
&\hspace{0.4in}
d
\ge
2\sqrt{2\pi e} \log\epsilon^{-1} + 1
,\qquad
\eta
<
\min\left(
\frac{1}{|\mu_4 - 3|}
,
\frac{|\mu_4 - 3|}{B^8} e^{-1}
\right)
,\quad\text{and}
\\&
C_{\ref{lemm:convergence_result}, L}^*  \log^8(T_{\eta, 1}^o \delta^{-1})
\cdot
\frac{B^8}{|\mu_4 - 3|}
\cdot
d^2 \log^2 d
\cdot
\eta \log\left(\frac{|\mu_4 - 3|}{B^8}\eta^{-1}\right)
\le
\frac{\epsilon^2}{(\tau + 1) \log^2\epsilon^{-1}}
,
\end{aligned}\eeq
there exists a uniformly distributed random variable $\mathcal{I} \in \{1,\dots,d\}$ and an event $\cH_{\ref{lemm:convergence_result}, L}$ with
$$
\PP(\cH_{\ref{lemm:convergence_result}, L})
\ge
1 - \left(
6\tau + 27 + \frac{10368}{\log^5\delta^{-1}}
\right) d \delta - 3 \epsilon
,
$$
such that on $\cH_{\ref{lemm:convergence_result}, L}$, iteration $\ub^{(t)}$ of Algorithm \ref{algo:ICA} satisfies for $t \in [T_{\eta, 0.5}^o, T_{\eta, \tau}^o]$
\beq\label{eq:uniform_cr}
\begin{aligned}
\left| \tan \angle\left(\ub^{(t)}, \ab_\cI \right) \right|
&\le
\sqrt{d} \cdot \left(
1 - \frac{\eta}{2d}|\mu_4 - 3|
\right)^{t - T_{\eta, 0.5}^o}
\\&\quad+
\sqrt{\tau + 1} C_{\ref{lemm:convergence_result}, L} \log^{5 / 2}\delta^{-1}
\cdot
\frac{B^4}{|\mu_4 - 3|^{1 / 2}}
\cdot
\sqrt{d^3 \eta \log\left(\frac{|\mu_4 - 3|}{B^8}\eta^{-1}\right)}
,
\end{aligned}\eeq
where $C_{\ref{lemm:convergence_result}, L}, C_{\ref{lemm:convergence_result}, L}^*$ are positive, absolute constants.
\end{lemma} 
Under the uniform initialization assumption, Lemma \ref{lemm:convergence_result} shows that the term $\left| \tan \angle\left(\ub^{(t)}, \ab_\cI \right) \right|$ can be upper bounded by the sum of two summands, with
the first term geometrically decaying from $\sqrt{d}$ at rate $1-\eta|\mu_4-3| / (2d)$, and the second being the noise-induced error $\tilde\cO(\sqrt{d^3\eta})$.
The key idea behind the proof is that the scaling condition \eqref{eq:uniform_scaling} ensures that the iterate avoids the set where the unstable stationary points lie, and hence efficiently contracts to the basin of attraction of the independent component pairs.

Analogous to Theorem \ref{theo:finite_sample2}, when the sample size $T$ satisfies the scaling condition $T = \tilde\Omega(d^3)$, one can carefully choose a stepsize $\eta = \tilde\Theta(d \log T / T)$ and establish a finite-sample bound in $T$ based on Lemma \ref{lemm:convergence_result}. We formulate this fact as our second main theorem.

\begin{theorem}[Finite-Sample Error Bound with Uniform Initialization]\label{theo:finite_sample}
Let the dimension $d \ge 2$, let Assumption \ref{assu:distribution} hold, and let initialization $\ub^{(0)}$ be uniformly sampled from the unit sphere $\cD_1$.
For sample size $T \ge 100$, set the stepsize $\eta$ as
$
\eta(T)
=
\dfrac{4 d \log\left(\frac{|\mu_4 - 3|^2}{4 B^8 d} T\right)}{|\mu_4 - 3| T}
.
$
Then, for any fixed positive numbers $d \ge 2, T \ge 100, \epsilon \in (0, 1 / 4]$ satisfying the scaling condition
\beq\label{eq:uniform_scaling2}
d \ge 2\sqrt{2\pi e} \log\epsilon^{-1} + 1
,\qquad
C_{\ref{theo:finite_sample}, T}^* \log^8(C_{\ref{theo:finite_sample}, T}' \epsilon^{-1} d T)
\cdot
\frac{B^8}{|\mu_4 - 3|^2}
\cdot
\frac{d^3 \log^2 d \log^2 T}{T}
\le
\frac{\epsilon^2}{\log^2\epsilon^{-1}}
,
\eeq
there exists a uniformly distributed random variable $\mathcal{I} \in \{1,\dots,d\}$ and an event $\cH_{\ref{theo:finite_sample}, T}$ with $\PP(\cH_{\ref{theo:finite_sample}, T}) \ge 1 - 4 \epsilon$ such that on $\cH_{\ref{theo:finite_sample}, T}$, iteration $\ub^{(t)}$ of Algorithm \ref{algo:ICA} satisfies
$$
\left| \tan \angle\left(\ub^{(T)}, \ab_\cI\right) \right|
\le
C_{\ref{theo:finite_sample}, T} \log^{5 / 2}(C_{\ref{theo:finite_sample}, T}' \epsilon^{-1} d)
\cdot
\frac{B^4}{|\mu_4 - 3|}
\cdot
\sqrt{
\frac{d^4 \log^2 T}{T}
}
,
$$
where $C_{\ref{theo:finite_sample}, T}, C_{\ref{theo:finite_sample}, T}^*, C_{\ref{theo:finite_sample}, T}'$ are positive, absolute constants.
\end{theorem}
Theorem \ref{theo:finite_sample} achieves, under the scaling condition $T = \tilde\Omega(d^3)$, an $\tilde\cO(\sqrt{d^4 / T})$ finite-sample error bound on $\left| \tan \angle\left(\ub^{(T)}, \ab_\cI \right) \right|$ for some $\cI$ drawn uniformly at random in $[d]$.
With Theorems \ref{theo:finite_sample2} and \ref{theo:finite_sample} for warm and uniform initializations respectively, a specific choice of stepsize allows us to have the best of the two worlds.
Assuming prior knowledge of the sample size $T$, we initialize $\ub^{(0)}$ uniformly at random from the unit sphere $\cD_1$ and run Algorithm \ref{algo:ICA} in two consecutive phases, each using $T / 2$ observations:

\begin{itemize}
\item
In the first phase, we initialize $\ub^{(0)}$ uniformly at random on unit sphere $\cD_1$, pick a constant stepsize $\eta_1 = \tilde\Theta(d \log T / T)$ and update iteration $\ub^{(t)}$ via \eqref{tensorialSGD} for $T / 2$ iterates.
Theorem \ref{theo:finite_sample} guarantees with high probability that $\ub^{(T/2)}$ satisfies the warm initialization condition \eqref{eq:lemm_warm} under the scaling condition $T = \tilde\Omega(d^4)$;
 
\item
In the second phase, we warm-initialize the algorithm using the output of the first phase $\ub^{(T / 2)}$, pick a constant stepsize $\eta_2 = \tilde\Theta(\log T / T)$ and update the iteration $\ub^{(t)}$ via \eqref{tensorialSGD} for $T / 2$ iterates.
The last iterate achieves an error bound of $\tilde\cO(\sqrt{d / T})$ as indicated by Theorem \ref{theo:finite_sample2}.
\end{itemize}
This two-phase procedure yields an improved finite-sample error bound $\tilde\cO(\sqrt{d / T})$ under the uniform initialization and scaling condition $T = \tilde\Omega(d^4)$, formally stated in the following corollary.

\begin{corollary}[Improved Finite-Sample Error Bound with Uniform Initialization]\label{coro:two_phase}
Let the dimension $d \ge 2$, let Assumption \ref{assu:distribution} hold, and let initialization $\ub^{(0)}$ be uniformly sampled from the unit sphere $\cD_1$.
Set for sample size $T \ge 200$ the stepsizes as
\beq\label{stepsizes}
\eta_1(T)
=
\frac{8 d \log\left(\frac{|\mu_4 - 3|^2}{8 B^8 d}  T\right)}{|\mu_4 - 3| T}
,\qquad
\eta_2(T)
=
\frac{9 \log\left(\frac{|\mu_4 - 3|^2}{9 B^8} T\right)}{|\mu_4 - 3| T}
.\eeq
Then, for any fixed positive $\epsilon \in (0, 1 / 5]$ satisfying the scaling condition
\beq\label{eq:uniform_scaling3}
d
\ge
2\sqrt{2\pi e} \log\epsilon^{-1} + 1
,\qquad
2\max\{C_{\ref{theo:finite_sample2}, T}^*, C_{\ref{theo:finite_sample}, T}^*, C_{\ref{theo:finite_sample}, T}^2\} \log^8(C_{\ref{theo:finite_sample}, T}' \epsilon^{-1} d T)
\cdot
\frac{B^8}{|\mu_4 - 3|^2}
\cdot
\frac{d^4 \log^2 T}{T}
\le
\frac{\epsilon^2}{ \log^2\epsilon^{-1}}
,
\eeq
there exists a uniformly distributed random variable $\cI \in [d]$ and an event $\cH_{\ref{coro:two_phase}, C}$ with $\PP(\cH_{\ref{coro:two_phase}, C}) \ge 1 - 5\epsilon$ such that on the event $\cH_{\ref{coro:two_phase}, C}$, running Algorithm \ref{algo:ICA} for $T/2$ iterates with stepsize $\eta_1(T)$ followed by $T/2$ iterates with stepsize $\eta_2(T)$ outputs an $\ub^{(T)}$ satisfying
\blue{$$
\left| \tan \angle\left(\ub^{(T)}, \ab_{\cI}\right) \right|
	\le
\sqrt{2} C_{\ref{theo:finite_sample2}, T} \log^{5 / 2}(C_{\ref{theo:finite_sample2}, T}' \epsilon^{-1} d)
\cdot
\frac{B^4}{|\mu_4 - 3|}
\cdot
\sqrt{\frac{d \log^2 T}{T}}
.
$$}%
where $C_{\ref{theo:finite_sample2}, T}, C_{\ref{theo:finite_sample2}, T}',
C_{\ref{theo:finite_sample}, T}, C_{\ref{theo:finite_sample}, T}^*, C_{\ref{theo:finite_sample}, T}'$ are positive, absolute constants defined earlier in Theorems \ref{theo:finite_sample2} and \ref{theo:finite_sample}.
\end{corollary}
\blue{
We make several remarks:
\begin{enumerate}[label=(\roman*)]
\item
Our dynamics-based analysis of uniformly-initialized online tensorial ICA is different from that of PCA in \cite{li2018near} (see \S\ref{sec:intro}).  In particular, we provide a coordinate-pair analysis of the algorithm to cope with the landscape of ICA.
A two-phase procedure is, however, required to achieve a rate that is sharp in dimension $d$.
\item
\cite{ge2015escaping} make use of artificial noise injection in their noisy projected SGD algorithm, which includes online tensorial independent component analysis as an application scenario. The analysis provided in the appendix of \cite{ge2015escaping}, however, is worsened by its unrealistic isotropic covariance assumption for the incurred stochastic noise, which is not satisfied by the tensorial ICA algorithm, noise-injected or not.
A straightforward extension to the generic noise satisfied by the tensorial ICA condition was claimed, but the analysis of such a case is not available in \cite{ge2015escaping}.
\item
By choosing $\epsilon$ in rate analysis above as a small positive constant, say $1/12$, we achieve a scaling condition such that if $d^4 / T$ is sufficiently small up to a polylogarithmic factor, we achieve the finite-sample convergence rate with probability no less than $7/12$.
In comparison, \cite{ge2015escaping} achieve an $d^\alpha / T^{1/4}$ error bound for some $\alpha\ge 2$, under a more stringent assumption that all independent components are almost surely bounded, which indicates a scaling condition of $d^8 / T \to 0$ at best.%
\footnote{
\blue{Our scaling condition is the best known when ignoring the polylogarithmic factors.
Optimistically, if analysis in the $\epsilon^{-4}$ or $\epsilon^{-3.5}$ state-of-the-art complexity results in \cite{jin2019nonconvex,fang2019sharp} can be carried over to  Riemannian optimization, setting $\sqrt{\epsilon} = 1/d$ implies a complexity of $d^9$ or $d^8$ at best (for the ICA objective the smallest Hessian eigenvalue changes from $-1/d$ to $\Omega(1)$ so the Hessian-Lipschitz constant is $\Omega(1)$ and the stochastic gradient admits a variance of $O(d)$).
}}
The scaling condition imposed in Theorem \ref{theo:informal}, however, leads to an error bound that depends polynomially on the inverse unsuccessful probability $1/(5\epsilon)$.
\end{enumerate}
}

\pb
In the rest of this section we study the uniform initialization case and prove Theorem \ref{theo:finite_sample}.
The key idea behind our analysis is that the uniform initialization is sufficiently far from the set where the unstable stationary points lie (with high probability), and thus a delicate concentration analysis can show that  saddle-point avoidance is guaranteed throughout the entire online tensorial ICA algorithm.
Inherited from the warm-initialization analysis in \S\ref{sec:warmsec}, we recall the rotated iteration $\vb^{(t)} \equiv \Pb \Ab^\top \ub^{(t)}$ and the rotated observations $\bY^{(t)} = \Pb \Ab^\top \bX^{(t)}$, so our online tensorial ICA update rule can still be translated into \eqref{eq:rescaled_update}.
Here the permutation matrix $\Pb$ has $\Pb(\cI;1) = \Pb(1;\cI) = 1 = \Pb(j;j)$ for $j\in \{1,\cI\}^c$ and 0 elsewhere, in which $\cI \equiv\argmin_{i\in [d]} \tan^2 \angle\left(\ub^{(0)}, \ab_i\right)$.%
\red{ (also appeared previously as \eqref{eq:u0T} in Remark \ref{rema:u0T}).}
We let the coordinate-wise intermediate initialization region for each $k \in [2, d]$ be
\beq\label{Sintermediate}
\cD_{\text{mid}, k}
\equiv
\left\{ \vb\in \cD_1: v_1^2 \ge \max_{i\in [2,d]} v_i^2 \text{ and } v_1^2 \ge 3 v_k^2 \right\}
.
\eeq
In addition, we let the cold initialization region be
\beq\label{eq:cold}
\cD_{\text{cold}}
=
\left\{
\vb \in \cD^{d - 1} : 
v_1^2
\ge
\max_{i\in [2,d]} v_i^2
\right\}.
\eeq
By definition of the rotated iteration $\{\vb^{(t)}\}$ and index $\cI$, we know that $\vb^{(0)} \in \cD_{\text{cold}}$ always holds.
As we will see later in Lemmas \ref{lemm:expo}, \ref{lemm:initW} and \ref{lemm:hitt} of this section, when initialized uniformly at random on the unit sphere $\cD_1$ a gap between $(v_1^{(0)})^2$ and $\max_{2 \le k \le d} (v_k^{(0)})^2$ persists on a high-probability event ($\cH_{\ref{lemm:initW}, L}$) as the data dimension $d\to\infty$.
Moreover on a high-probability event ($\bigcap_{k\in [2,d]} \cH_{k; \ref{lemm:expo}, L}$), the iterate $\vb^{(t)}$ enters the intersection of intermediate regions $\bigcap_{k\in [2,d]} \cD_{\text{mid}, k}$ within $T_{\eta, 0.5}^o$ iterations.
After $\vb^{(t)}$ enters $\bigcap_{k\in [2,d]} \cD_{\text{mid}, k}$, on a third high-probability event ($\bigcap_{k\in [2,d]} \cH_{k; \ref{lemm:hitt}, L}$) it decays exponentially fast and stays within a $\tilde{O}(\sqrt{d^3 \eta})$-neighborhood of the independent component pair $\pm \eb_1$.

\pb\subsection{Initialization in the Intermediate Region}\label{sec:uniformsub1}
Recall the intermediate initialization region $\cD_{\text{mid}, k}$ defined in \eqref{Sintermediate} for each $k \in [2, d]$.
We also define a slightly larger coordinate-wise intermediate auxiliary region for each $k \in [2, d]$ as
$
\cD_{\text{mid-aux}, k} = \left\{ \vb\in \cD_1: v_1^2 \ge \max_{i \ge 2} v_i^2 \text{ and } v_1^2 \ge 2 v_k^2 \right\}
$.
When  $\vb^{(0)} \in \cD_{\text{mid}, k}$, we define the first time the iterate exits $\cD_{\text{mid-aux}, k}$ as
\beq\label{cTw}
\cT_{w, k}
=
\inf\left\{ t\ge 1: 
\vb^{(t)} \in \cD_{\text{mid-aux}, k}^c
\right\}
.
\eeq
Thus, for each $k \in [2, d]$, $\cT_{w, k}$ is a stopping time with respect to filtration $\cF_t = \sigma(\vb^{(0)}, \bY^{(1)}, \dots, \bY^{(t)})$ (we suppose that all $k$ that appear in the rest of our discussion satisfy $k \in [2, d]$, unless stated otherwise).

Our goal is to prove the following high-probability bound for each coordinate $k$.
For the analysis of the intermediate iterates, we recall the notation $U_k^{(t)} = v_k^{(t)} / v_1^{(t)}$ previously defined in \eqref{eq:Uk}.

\begin{lemma}\label{lemm:expo}
Let the settings in Lemma \ref{lemm:convergence_result} hold, and fix the coordinate $k\in [2, d]$ and value $\tau > 0$.
Then for any positive numbers $\eta, \delta$ satisfying the scaling condition \eqref{eq:uniform_scaling}, and given the coordinate-wise intermediate initialization condition $\vb^{(0)} \in \cD_{\text{mid}, k}$, there exists an event $\cH_{k; \ref{lemm:expo}, L}$ satisfying
$
\PP(\cH_{k; \ref{lemm:expo}, L})
\ge
1 - \left(
6\tau + 12 + \frac{5184}{\log^5\delta^{-1}}
\right)\delta
,
$
such that on event $\cH_{k; \ref{lemm:expo}, L}$ the following holds:
\beq\label{eq:mid_prop1}
\sup_{t\le T_{\eta, \tau}^o  \land \cT_{w, k}} \left|
 U_k^{(t)}  - U_k^{(0)} \prod_{s = 0}^{t - 1} \left[ 1 - \eta |\mu_4 - 3| \left( (v_1^{(s)})^2 - (v_k^{(s)})^2 \right) \right]
\right| 
\le 
2 C_{\ref{lemm:expo}, L} \log^{5 / 2}\delta^{-1}
\cdot
B^4
\cdot
d^{1 / 2} \eta (T_{\eta, \tau}^o)^{1 / 2}
,
\eeq
where $C_{\ref{lemm:expo}, L}$ is a positive, absolute constant.
\end{lemma}
From \eqref{eq:mid_prop1} in Lemma \ref{lemm:expo} we know that, for each coordinate $k \in [2, d]$ initialized in the intermediate region $\cD_{\text{mid}, k}$, with high probability the iteration $ U_k^{(t)}$ fluctuates around a deterministic curve that decays geometrically whenever $(v_1^{(t)})^2 - (v_k^{(t)})^2$ is bounded below by a positive number.

\pb\subsection{Initialization in the Cold Region}\label{sec:uniformsub2}
Recall the cold initialization region $\cD_{\text{cold}}$  defined earlier in \eqref{eq:cold}.
The iteration we study in the old initialization analysis is
\beq\label{eq:W_def}
W_k^{(t)} = \frac{ (v_1^{(t)})^2 - (v_k^{(t)})^2 }{(v_k^{(t)})^2 }
.
\eeq
Under the setting of uniform initialization on the unit sphere in Lemma \ref{lemm:convergence_result}, we have the following lemma. Note that the uniform initialization conditions for $\ub^{(0)}$ and $\vb^{(0)}$ are equivalent.
\begin{lemma}\label{lemm:initW}
Let $\vb^{(0)}$ be uniformly sampled from the unit sphere $\cD_1$ and $\epsilon$ be any fixed positive number, with dimension $d$ and $\epsilon$ satisfying
\beq\label{eq:dim_cond}
d
\ge
2\sqrt{2\pi e} \log\epsilon^{-1} + 1
.
\eeq
Then there exists an event $\cH_{\ref{lemm:initW}, L}$ with $\PP(\cH_{\ref{lemm:initW}, L}) \ge 1 - 3\epsilon$ such that on event $\cH_{\ref{lemm:initW}, L}$ the following holds:
\beq\label{eq:lemm_initW1}
\min_{2 \le k \le d} W_k^{(0)}
\ge
\frac{\epsilon}{8 \log\epsilon^{-1} \log d}
.\eeq
\end{lemma}
Our goal is to estimate the time $t$ when $\vb^{(t)}$ enters each coordinate-wise intermediate initialization region starting with the initialization gap given in Lemma \ref{lemm:initW}.
For each coordinate $k\in [2,d]$, we define the first time $\vb^{(t)}$ enters the coordinate-wise intermediate region $\cD_{\text{mid}, k}$ as
\beq\label{eq:Tck}
\cT_{c, k}
=
\inf \left\{
t \ge 1 : 
\vb^{(t)}
\in
\cD_{\text{mid}, k},
\right\}
\eeq
and the first time the iterate exits $\cD_{\text{cold}}$ without entering $\cD_{\text{mid}, k}$, earlier defined in \eqref{Sintermediate} and \eqref{eq:cold}, as
\beq\label{eq:T1}
\cT_1
=
\inf \left\{
t \ge 1 : 
\vb^{(t)}
\in
\cD_{\text{cold}}^c
\right\}.
\eeq
In other words, $\cT_{c, k}$ is the first time $t$ such that $W_k^{(t)} \ge 2$, and $\cT_1$ is the first time $\min_{i\in [2,d]} W_i^{(t)} < 0$.
By the definition of the rotated iteration $\vb^{(t)}$ we have $\vb^{(0)} \in \cD_{\text{cold}}$ always holds.
For each coordinate $k \in [2, d]$, if $\cT_{c, k} = 0$ then $\vb^{(0)} \in \cD_{\text{mid}, k}$ and the previous analysis directly applies for coordinate $k$.
The following lemma characterizes the opposite case, where the exponential growth of iteration $W_k^{(t)}$ helps us to determine when $\vb^{(t)}$ enters the intermediate region $\cD_{\text{mid}, k}$.

\begin{lemma}\label{lemm:hitt}
Let the settings in Lemma \ref{lemm:convergence_result} hold, and fix the coordinate $k \in [2, d]$.
Then, for any fixed positive numbers $\eta, \delta$ satisfying the scaling condition \eqref{eq:uniform_scaling} along with the coordinate-wise cold initialization condition $\vb^{(0)} \in \cD_{\text{cold}} \cap \cD_{\text{mid}, k}^c$, there exists an event $\cH_{k; \ref{lemm:hitt}, L}$ with
$
\PP(\cH_{k; \ref{lemm:hitt}, L})
\ge
1 - \left(
15 + \frac{5184}{\log^5\delta^{-1}}
\right)\delta
,
$
such that on event $\cH_{k; \ref{lemm:hitt}, L}$ the following holds:
\beq\label{Whitt}
\sup_{t \le T_{\eta, 0.5}^o \wedge \cT_{c, k} \wedge \cT_1} \left|  
 W_k^{(t)}  \prod_{s = 0}^{t - 1} \left( 1 + \eta |\mu_4 - 3| (v_1^{(s)})^2 \right)^{-1}  - W_k^{(0)}
  \right| 
\le
C_{\ref{lemm:hitt}, L} \log^{5 / 2}\delta^{-1}
\cdot
\frac{B^4}{|\mu_4 - 3|^{1 / 2}}
\cdot
d \eta^{1 / 2}
,
\eeq
where $C_{\ref{lemm:hitt}, L}$ is a positive, absolute constant.
\end{lemma}
From Lemma \ref{lemm:hitt} we know that for each coordinate $k \in [2, d]$, if the initialization lies outside intermediate region $\cD_{\text{mid}, k}$, then, with high probability, the iterate $ W_k^{(t)} $ is controlled within an exponentially growing \blue{dynamics} for $T_{\eta, 0.5}^o$ iterates before it either enters the intermediate region $\cD_{\text{mid}, k}$ or exits the cold region $\cD_{\text{cold}}$.
As we will see later in the proof of Lemma \ref{lemm:convergence_result}, by putting all coordinates together we can show that $\vb^{(t)}$ rarely leaves the cold region $\cD_{\text{cold}}$, which implies that $\vb^{(t)}$ will enter the joint intermediate region $\cap_{k\in [2,d]} \cD_{\text{mid}, k}$ within $T_{\eta, 0.5}^o$ iterates with high probability.

\red{Limited by space in the main text, we provide the proof of the convergence Lemma \ref{lemm:convergence_result} in \S\ref{sec:uniform_lemm} and its finite-sample error Theorem \ref{theo:finite_sample} in \S\ref{sec:uniform_theo}, respectively.
}

\pb\section{Additional Related Literature}\label{sec:related}
\red{The method of stochastic approximation, or commonly referred to as stochastic gradient descent (SGD), recently gains tremendous popularity when solving large-scale stochastic optimization problems \citep{robbins1951stochastic,ruppert1988efficient,polyak1992acceleration,KUSHNER-YIN,lai2003stochastic} due to its exceptional performance handling massive (and often streaming) data.
On the theoretical side, many stochastic approximation algorithms and their variants have been proved to provide minimax-optimal estimators under mild initializations for statistical estimation problems that can be casted into convex stochastic optimization problems such as linear regression and principal component analysis.
}
The themes of ICA and tensor decomposition have been studied in numerous statistical and signal-processing literatures \citep{bach2002kernel, chen2006efficient, samworth2012independent, bonhomme2009consistent, eriksson2004identifiability, hallin2015r, HYVARINEN-KARHUNEN-OJA, hyvarinen1997fast, hyvarinen1999fast, hyvarinen2000independent, ilmonen2011semiparametrically, kollo2008multivariate, miettinen2015fourth, oja2006scatter, tichavsky2006performance, wang2017scaling, ge2017optimization}.
For a treatment from a spectral learning perspective (mainly for the deterministic scenario), we refer to the recently published monograph of \citet{janzamin2019spectral} and the bibliography therein.
Recent literature studies the ICA setting in the context of specific parametric families for independent component distributions and shows that parametric \citep{lee1999independent}, semi-parametric \citep{hastie2003independent,chen2006efficient, ilmonen2011semiparametrically} and nonparametric \citep{bach2002kernel, samarov2004nonparametric, samworth2012independent} models can be estimated via maximal likelihood estimation or minimization of mutual information between independent components.
Our work focuses on a different type of contrast function based on tensor decomposition and kurtosis maximization, and hence our methodology is quite different from this line of work.

\blue{Our work is most closely related to \cite{ge2015escaping}, which led to a general line of work on stochastic-gradient-based nonconvex optimization  \citep{dauphin2014identifying,ge2015escaping,sun2015nonconvex,sun2015complete,sun2018geometric,anandkumar2016efficient,jin2017escape,jin2019nonconvex,jin2018accelerated,lei2017non,allen2018natasha2,daneshmand2018escaping,fang2019sharp,cutkosky2019momentum,cutkosky2020momentum} as well as the Riemannian manifold \citep{zhang2016first,zhang2016riemannian,tripuraneni2018stochastic}.
A large family of nonconvex landscape has been studied in \cite{mei2018landscape}, where uniform convergence of the empirical loss to the population loss is established.
It is also related to work on recursive variance-reduced gradient methods for smooth optimization \citep{nguyen2017sarah,nguyen2017stochastic,fang2018spider,zhou2020stochastic,wang2019spiderboost,arjevani2020second}.
In particular, we note the work of \citet{ge2015escaping,jin2019nonconvex,daneshmand2018escaping,fang2019sharp}, who study the dynamics of SGD for optimizing generic functions. When applied to specific statistical estimation problems, however, the results obtained by these methods can be coarse due to their neglect of specific geometric features of the landscape.
}

The pioneering work of \citet{ge2015escaping} studied the convergence rate of SGD for minimizing a large class of nonconvex objectives defined on a generic Riemannian manifold.
Under the bounded distributional assumption, \citet{ge2015escaping} prove that SGD equipped with projection as well as a special noise-injection step can escape from all saddle points and land at an approximate local minimizer in polynomial time of relevant parameters.
Convergence rates for generic first-order gradient descent algorithms without adding noise injection are generally unknown and can be unfavorable \citep{lee2019first, du2017gradient, pemantle1990nonconvergence}.
Favorable results can be obtained under special spherical distributions for the noise or sophisticated procedures for avoidance of saddle points \citep{ge2015escaping,jin2017escape,jin2019nonconvex,sun2015nonconvex,allen2018natasha2,fang2019sharp}.

Recent years have witnessed significant progress on computational and statistical aspects of low-rank representation methods.
A number of recent papers \citep{carmon2020first,ge2017no,li2018near,bai2018subgradient,davis2018subgradient,li2018global,zhu2018dual,ma2019implicit,chen2019gradient,gilboa2019efficient,kuo2019geometry,qu2019nonconvex,tan2019online,yang2019misspecified,na2019high,chen2020spectral,lau2020short,li2020nonconvex,zhai2020complete} study the gradient-descent dynamics of of matrix factorization/completion, principal component pursuit, dictionary learning, phase retrieval, blind deconvolution, and many others, in the setting of batch or online (streaming) data.
Notably, \citet{carmon2020first,li2018near} pursue a dynamics-based analysis to study gradient descent and its cubic-regularized variant for eigenvalue problems.
Related work on efficient convergence of Oja's online PCA iteration can be found in \citet{jain2016streaming,allen2017first}.
\citet{li2016online,wang2020asymptotic,wang2017scaling}, who study the online tensorial ICA method from the viewpoint of scaling limits and (stochastic) differential equation approximations.
While these results provide valuable insights into our problem, straightforward translations to nonasymptotic convergence guarantees are not available, mainly due to the differential equation approximation being a weak convergence formulation instead of a strong one. \blue{Our refined projected stochastic gradient analysis for online tensorial ICA provided in this paper is both \blue{dynamics-based and nonasymptotic}, and we are able to prove that under some mild scaling conditions, random initialization provides sufficient deviation from the set of unstable stationary points such that a vanilla online tensorial ICA algorithm can be guaranteesdto achieve an sharp convergence rate.}

\pb\section{Summary}\label{sec:summary}
We have studied the dynamics of an algorithm formulated as online stochastic approximation of (orthogonal) tensorial ICA.  Our algorithm can be viewed as a method for optimizing a nonconvex objective of excess kurtosis in a given direction.
We show that with properly chosen stepsizes and under mild scaling conditions our online tensorial ICA algorithm achieves a $\tilde{O}(\sqrt{d/T})$-convergence rate, which is superior to the best existing analysis of this problem.
Our algorithm requires no noise-injection steps or specially-designed loops for saddle-point avoidance, and our \blue{dynamics-based} approach enjoys multiple advantages over existing analyses of online stochastic approximation for tensorial ICA estimation.

We believe that our analysis can generalize to a broader class of statistical estimation problems that can be cast as nonconvex stochastic optimization problems.
Future directions include further improvements of the convergence rate and scaling conditions or justification of the impossibility (or minimax optimality) of such rates, analyzing the mini-batch stochastic approximation algorithm as well as the non-identical kurtosis case for ICA, and finally generalizing our analysis of the dynamics of stochastic online algorithms to the nonorthogonal tensor decomposition case and over-parameterized cases.

\section*{Acknowledgements}
We thank Wenlong Mou and Yuren Zhou for valuable discussions.
This work was supported in part by the Mathematical Data Science program of the Office of Naval Research under grant number N00014-18-1-2764.

\bibliographystyle{plainnat}
\bibliography{SAILreference}

\newpage\appendix
\pb\section*{Appendix}
In Appendix,
\S\ref{sec:mainproof} proves the main results in the paper.
\S\ref{sec:warm_aux} and \S\ref{sec:uniform_aux} provide all secondary lemmas and their proofs for warm and uniform initialization analysis, separately.
\S\ref{sec:reversed_gronwall}, \S\ref{sec:psi_alpha} and \S\ref{sec:concentration} provide necessary tools including a reversed version of Gronwall's inequality, preliminaries and properties of Orlicz $\psi_\alpha$-norm and a concentration inequality, all of which are theoretical building blocks of this paper.
Finally, \S\ref{sec:experiment} visualizes the tensorial ICA landscape and provides preliminary simulation results that validate our theory.

\paragraph{Notations}
Throughout this paper, we treat $B,\mu_4,\tau$ as positive constants.
We use bold upper case letters to denote matrices, bold lower case letters to denote vectors and italic letters to denote randomness.
For any matrix $\Ab$ or vector $\vb$, $\Ab^\top$ and $\vb^\top$ denote their transposes.
For any vector $\vb$, $v_k$ denotes its $k$th coordinate.
For a sequence of $\{x^{(t)}\}$ and positive $\{y^{(t)}\}$, we write $x^{(t)} = O(y^{(t)})$ if there exists a positive constant $M$ such that $|x^{(t)}| \le M y^{(t)}$,
write $x^{(t)} = \Omega(y^{(t)})$ if there exists a positive constant $M < \infty$ such that $|x^{(t)}| \ge M y^{(t)}$,
and write $x^{(t)} = \Theta(y^{(t)})$ if both $x^{(t)} = O(y^{(t)})$ and $x^{(t)} = \Omega(y^{(t)})$ hold.
We use $\tilde{O}, \tilde\Theta, \tilde\Omega$ to hide factors that are polylogarithmically dependent on dimension $d$, stepsize $\eta$, sample size $T$ and inverse unsuccessful probability $\delta^{-1}$.
We use $\lfloor x \rfloor$ to denote the floor function and $\lceil x \rceil$ to denote the ceiling function.
We let $x \wedge y = \min(x, y)$ and $x \vee y = \max(x, y)$.
For any vector $\vb$, we use $\|\vb\|$ to denote its Euclidean norm.
For any integer $n$, we define set $[n] = \{1, \dots, n\}$.
Finally, we use $\cS^c$ to denote the complement of set (or event) $\cS$.

\pb\section{Deferred Proofs of Main Results}\label{sec:mainproof}
This section includes the deferred proofs of the main results in \S\ref{sec:warmsec} and \S\ref{sec:uniformsec}.
In their order of appearance, \S\ref{sec:warm_lemm} and \S\ref{sec:warm_theo} prove in sequel Lemma \ref{lemm:convergence_result2} and Theorem \ref{theo:finite_sample2}, and \S\ref{sec:uniform_lemm} and \S\ref{sec:uniform_theo} prove the convergence Lemma \ref{lemm:convergence_result} and its finite-sample error Theorem \ref{theo:finite_sample}, respectively.
All secondary lemmas are deferred to \S\ref{sec:warm_aux} and \S\ref{sec:uniform_aux}.

\pb\subsection{Proof of Lemma \ref{lemm:convergence_result2}}\label{sec:warm_lemm}
We use the key Lemma \ref{lemm:fast_expo} to tightly estimate the \blue{dynamics} in all coordinates and thereby obtain a proof of Lemma \ref{lemm:convergence_result2}.

\begin{proof}[Proof of Lemma \ref{lemm:convergence_result2}]
We denote $\cH_{\ref{lemm:convergence_result2}, L} \equiv \bigcap_{k\in [2,d]} \cH_{k; \ref{lemm:fast_expo}, L}$ as the intersection of events $\cH_{k; \ref{lemm:fast_expo}, L}$.
Consider now the following $(d - 1)$-dimensional vector
\beq\label{eq:warm_vector}
\left(
U_k^{(t)} - U_k^{(0)} \prod_{s = 0}^{t - 1} \left[ 1 - \eta |\mu_4 - 3| \left( (v_1^{(s)})^2 - (v_k^{(s)})^2 \right) \right]
 : k\in [2,d]
\right)
.
\eeq
Using Lemma \ref{lemm:fast_expo}, on the event $\cH_{\ref{lemm:convergence_result2}, L} \cap (t \le T_{\eta, \tau}^* \wedge \cT_x)$ we bound the Euclidean norm of \eqref{eq:warm_vector} by
\beq\label{eq:lemm:convergence_result2-norm1}
\sqrt{ \sum_{k = 2}^d \left(
U_k^{(t)} - U_k^{(0)} \prod_{s = 0}^{t - 1} \left[ 1 - \eta |\mu_4 - 3| \left( (v_1^{(s)})^2 - (v_k^{(s)})^2 \right) \right]
\right)^2 }
\le
\sqrt{d}
\cdot
2 C_{\ref{lemm:fast_expo}, L} \log^{5 / 2}\delta^{-1}
\cdot
B^4
\cdot
\eta (T_{\eta, \tau}^*)^{1 / 2}
.
\eeq
Additionally, the left hand of \eqref{eq:lemm:convergence_result2-norm1} is the norm of the subtraction of two vectors and hence lower bounded by
\beq\label{eq:lemm:convergence_result2-norm2}
\begin{aligned}
\lefteqn{
\sqrt{ \sum_{k = 2}^d \left(
U_k^{(t)} - U_k^{(0)} \prod_{s = 0}^{t - 1} \left[ 1 - \eta |\mu_4 - 3| \left( (v_1^{(s)})^2 - (v_k^{(s)})^2 \right) \right]
\right)^2 }
}
\\&\ge
\sqrt{ \sum_{k = 2}^d \left( 
U_k^{(t)} \right)^2}
-
\sqrt{ \sum_{k = 2}^d \left( 
U_k^{(0)} \prod_{s = 0}^{t - 1} \left[ 1 - \eta |\mu_4 - 3| \left( (v_1^{(s)})^2 - (v_k^{(s)})^2 \right) \right]
\right)^2}
,
\end{aligned}\eeq
due to triangle inequality of Euclidean norm.
The definition of iteration $\{U_k^{(t)}\}$ in \eqref{eq:Uk} implies that $\left|\tan \angle\left( \vb^{(t)}, \eb_1\right)\right| = \sqrt{ \sum_{k = 2}^d (U_k^{(t)})^2 }$ and the definition of stopping time $\cT_x$ in \eqref{eq:cTx} implies that $(v_1^{(s)})^2 - (v_k^{(s)})^2 \ge 1/3$ holds for all $k \in [2, d]$ and $0 \le s < t$ on the event $(t \le \cT_x)$.
Combining this with \eqref{eq:lemm:convergence_result2-norm1} and \eqref{eq:lemm:convergence_result2-norm2}, we obtain on the event $\cH_{\ref{lemm:convergence_result2}, L} \cap (t \le T_{\eta, \tau}^* \wedge \cT_x)$
\beq\label{eq:lemm:convergence_result2-norm}
\begin{aligned}
\lefteqn{
\left|\tan \angle\left( \vb^{(t)}, \eb_1\right)\right|
 =
\sqrt{ \sum_{k = 2}^d \left(U_k^{(t)}\right)^2 }
}
\\&\le
\sqrt{ \sum_{k = 2}^d \left( U_k^{(0)} \right)^2}
\left(1 - \frac{\eta}{3} |\mu_4 - 3|\right)^t
+
\sqrt{d}
\cdot 2 C_{\ref{lemm:fast_expo}, L} \log^{5 / 2}\delta^{-1}
\cdot B^4
\cdot \eta (T_{\eta, \tau}^*)^{1 / 2}
\\&=
\left|\tan \angle\left( \vb^{(0)}, \eb_1\right)\right|
\left(1 - \frac{\eta}{3} |\mu_4 - 3|\right)^t
+
2 C_{\ref{lemm:fast_expo}, L} \log^{5 / 2}\delta^{-1}
\cdot B^4
\cdot (d \eta^2 T_{\eta, \tau}^*)^{1 / 2}
.
\end{aligned}\eeq
The definition of $T_{\eta, \tau}^*$ in \eqref{eq:Tetatau*} along with $-\log\left( 1 - \frac{\eta}{3} |\mu_4 - 3| \right) \ge \frac{\eta}{3} |\mu_4 - 3|$ gives
\beq\label{eq:Tetatau*_relation}
T_{\eta, \tau}^*
\le
1 + \frac{3\tau}{|\mu_4 - 3|}
\cdot \eta^{-1} \log\left(\frac{|\mu_4 - 3|}{B^8}  \eta^{-1}\right)
\le
\frac{3(\tau + 1)}{|\mu_4 - 3|}
\cdot \eta^{-1} \log\left(\frac{|\mu_4 - 3|}{B^8}  \eta^{-1}\right)
,
\eeq
where in the second inequality we use $\frac{|\mu_4 - 3|}{B^8} \eta^{-1} \ge e$ and $\frac{1}{|\mu_4-3|} \eta^{-1} \ge 1$ implied by scaling condition \eqref{eq:warm_scaling}.
Using relation \eqref{eq:Tetatau*_relation}, we find that scaling condition \eqref{eq:warm_scaling} with constant $C_{\ref{lemm:convergence_result2}, L}^* \equiv 713 C_{\ref{lemm:fast_expo}, L}^2$ indicates
\beq\label{eq:warm_drop_scaling}
\begin{aligned}
\lefteqn{
2 C_{\ref{lemm:fast_expo}, L} \log^{5 / 2}\delta^{-1}
\cdot
B^4
\cdot
(d \eta^2 T_{\eta, \tau}^*)^{1 / 2}
\le
\sqrt{
4 C_{\ref{lemm:fast_expo}, L}^2 B^8 \log^5\delta^{-1}
\cdot
\frac{3(\tau + 1)}{|\mu_4 - 3|}
\cdot
d \eta \log\left(\frac{|\mu_4 - 3|}{B^8} \eta^{-1} \right)
}}
\\&\le
\sqrt{
\frac{12}{713} (\tau + 1) \cdot
C_{\ref{lemm:convergence_result2}, L}^* 
\log^8(T_{\eta, 1}^* \delta^{-1})
\cdot \frac{B^8}{|\mu_4 - 3|}
\cdot d \eta \log\left(\frac{|\mu_4 - 3|}{B^8}\eta^{-1}\right)
}
\le
\sqrt{
\frac{12}{713}
}
,
\end{aligned}\eeq
where the elementary inequality $\log^5 \delta^{-1} \le \log^8(T_{\eta, 1}^* \delta^{-1})$ is applied due to $T_{\eta, 1}^* \ge 1$ and $\delta \le e^{-1}$. 
Viewing \eqref{eq:lemm_warm} (equivalent to $\vb^{(0)} \in \cD_{\text{warm}}$ in \eqref{eq:Swarm}) and \eqref{eq:lemm:convergence_result2-norm}, we have for each $t$ on the event $\cH_{\ref{lemm:convergence_result2}, L} \cap (\cT_x \le T_{\eta, \tau}^*) \cap (t \le \cT_x)$ that
\beq\label{eq:lemm:convergence_result2-drop1}
\left|\tan \angle\left( \vb^{(t)}, \eb_1\right)\right|
\le
\left|\tan \angle\left( \vb^{(0)}, \eb_1\right)\right|
+
\sqrt{
\frac{12}{713}
}
 <
\frac{1}{\sqrt{3}}
+
\left(
\frac{1}{\sqrt{2}} - \frac{1}{\sqrt{3}}
\right)
 = 
\frac{1}{\sqrt{2}}
,
\eeq
where in the first inequality we again apply $(\eta/3) |\mu_4-3| \le 1$ from \eqref{eq:warm_scaling}.
This further indicates that on event $\cH_{\ref{lemm:convergence_result2}, L} \cap (\cT_x \le T_{\eta, \tau}^*)$, inequality $\left|\tan \angle\left( \vb^{(\cT_x)}, \eb_1\right)\right| < 1/\sqrt{2}$ holds, which contradicts with the fact that $\vb^{(\cT_x)} \in \cD_{\text{warm-aux}}^c$ on the same event.
Therefore, we have $\cH_{\ref{lemm:convergence_result2}, L} \cap (\cT_x \le T_{\eta, \tau}^*) = \varnothing$, and equivalently
\beq\label{eq:lemm:convergence_result2-drop}
\cH_{\ref{lemm:convergence_result2}, L}
 =
\cH_{\ref{lemm:convergence_result2}, L} \cap (\cT_x > T_{\eta, \tau}^*)
.
\eeq
This implies that for all $t \in [0, T_{\eta, \tau}^*]$, \eqref{eq:lemm:convergence_result2-norm} holds on the event $\cH_{\ref{lemm:convergence_result2}, L}$.
Plugging in the inequality involving $T_{\eta, \tau}^*$ in \eqref{eq:Tetatau*_relation}, we have on the event $\cH_{\ref{lemm:convergence_result2}, L}$ that for all $t \in [0, T_{\eta, \tau}^*]$
\beq\label{eq:lemm:convergence_result2-final}
\begin{aligned}
\lefteqn{
\left|\tan \angle\left( \vb^{(t)}, \eb_1\right)\right|
\le
\left|\tan \angle\left( \vb^{(0)}, \eb_1\right)\right|
\left(1 - \frac{\eta}{3} |\mu_4 - 3|\right)^t
+
2 C_{\ref{lemm:fast_expo}, L} B^4 \log^{5 / 2}\delta^{-1}
\cdot
(d \eta^2 T_{\eta, \tau}^*)^{1 / 2}
}
\\&\le
\left|\tan \angle\left( \vb^{(0)}, \eb_1\right)\right|
\left(
1 - \frac{\eta}{3} |\mu_4 - 3|
\right)^t
\\&\quad+
\sqrt{\tau + 1} \left(2\sqrt{3}C_{\ref{lemm:fast_expo}, L}\right) \log^{5 / 2}\delta^{-1}
\cdot
\frac{B^4}{|\mu_4 - 3|^{1/2}}
\cdot
\sqrt{
d \eta \log\left(\frac{|\mu_4 - 3|}{B^8} \eta^{-1} \right)
}
.
\end{aligned}\eeq
Letting the constant $C_{\ref{lemm:convergence_result2}, L} \equiv 2\sqrt{3} C_{\ref{lemm:fast_expo}, L}$, the scaling relation \eqref{eq:uv} and the above derivation \eqref{eq:lemm:convergence_result2-final} prove that \eqref{eq:lemm:convergence_result-tanbdd} holds for all $t \in [0, T_{\eta, \tau}^*]$ on the event $\cH_{\ref{lemm:convergence_result2}, L}$.

The only left is to estimate the probability of $\cH_{\ref{lemm:convergence_result2}, L}$.
Lemma \ref{lemm:fast_expo} gives for each $k \in [2, d]$ the probability of event $\PP(\cH_{k; \ref{lemm:fast_expo}, L}) \ge 1 - \left(6\tau + 12 + \dfrac{5184}{\log^5\delta^{-1}}\right)\delta$, and hence elementary union bound calculation gives
\beq\label{eq:PP_tilde_cH}
\PP(\cH_{\ref{lemm:convergence_result2}, L})
=
\PP\left(\bigcap_{k\in[2,d]} \cH_{k; \ref{lemm:fast_expo}, L}\right)
\ge
1 - \left(
6 \tau + 12 + \frac{5184}{\log^5 \delta^{-1}}
\right) d \delta
,
\eeq
completing the whole proof of Lemma \ref{lemm:convergence_result2}.
\end{proof}

\pb\subsection{Proof of Theorem \ref{theo:finite_sample2}}\label{sec:warm_theo}
Now we turn to the proof of the finite-sample error Theorem \ref{theo:finite_sample2}.
The idea is to apply Lemma \ref{lemm:convergence_result2} with appropriate stepsize $\eta$ (as in \eqref{etaPickT}) as well as an appropriate $\tau$ to obtain the finite-sample error bound.

\begin{proof}[Proof of Theorem \ref{theo:finite_sample2}]

\begin{enumerate}[label=(\arabic*)]
\item
We first provide an upper bound on $T_{\eta(T), \tau}^*$. 
Under scaling condition \eqref{eq:warm_scaling2} with constant $C_{\ref{theo:finite_sample2}, T}^* \equiv \max\{90 C_{\ref{lemm:convergence_result2}, L}^*, 10\}$ and some constant $C_{\ref{theo:finite_sample2}, T}' > 1$ to be determined later, we have $\frac{2 |\mu_4 - 3|^2}{9 B^8} T \ge e$.
Plugging in $\eta = \eta(T)$ from \eqref{etaPickT} to relation \eqref{eq:Tetatau*_relation}, we have
\beq\label{eq:T_eta_ub}
\begin{aligned}
\lefteqn{
T_{\eta(T), \tau}^*
\le
\frac{3(\tau + 1)}{|\mu_4 - 3|}
\cdot
\eta(T)^{-1} \log\left(\frac{|\mu_4 - 3|}{B^8}  \eta(T)^{-1}\right)
}
\\&=
\frac{3(\tau + 1)}{|\mu_4 - 3|}
\cdot
\frac{2 |\mu_4 - 3| T}{9 \log\left(\frac{2 |\mu_4 - 3|^2}{9 B^8} T\right)} \log\left(
  \frac{|\mu_4 - 3|}{B^8}  \cdot \frac{2 |\mu_4 - 3| T}{9 \log\left(\frac{2 |\mu_4 - 3|^2}{9 B^8} T\right)}
\right)
\le
\frac{2 (\tau + 1)}{3} T
.
\end{aligned}\eeq

\item
Next we provide a lower bound on $T_{\eta(T), \tau}^*$. By Taylor expansion, for all $x \in (0,1/3]$ we know that
$$
\left| \log(1 - x) + x \right|
=
\left| \sum_{n=2}^\infty \frac{x^n}{n} \right|
\le 
\frac{x^2}{2}
\sum_{n=0}^\infty x^n
\le
\frac{x^2}{2} \frac{1}{1-1/3}
 =
\frac{3x^2}{4}
,
$$
and hence
\beq\label{eq:Taylor1}
\frac{1}{- \log(1 - x)}
\ge
\frac{1}{x + 3 x^2 / 4}
\ge
\frac{1}{x + x / 4}
\ge
\frac{4}{5x}
.
\eeq
From the definition of $T_{\eta, \tau}^*$ in \eqref{eq:Tetatau*}, for $\eta$ satisfying $\eta |\mu_4 - 3| / 3 \le 1 / 3$ we have
\beq\label{eq:Tetatau*_relation2}
T_{\eta, \tau}^*
\ge
\frac{\tau \log\left(\frac{|\mu_4 - 3|}{B^8} \cdot \eta^{-1}\right)}{- \log\left( 1 - \frac{\eta}{3} |\mu_4 - 3| \right)}
\ge
\frac{12 \tau }{5 |\mu_4 - 3| }
\cdot
\eta^{-1} \log\left(\frac{|\mu_4 - 3|}{B^8} \cdot \eta^{-1}\right)
.
\eeq
Under scaling condition \eqref{eq:warm_scaling2}, along with relation $\frac{B^4}{|\mu_4 - 3|} \ge \frac18$ given by Lemma \ref{lemm:B_mu4_relation} in Appendix \ref{sec:warm_aux} and $T \ge 100$, we have
\beq\label{eq:warm_scaling_verify_mid2}
\frac{\eta(T) |\mu_4 - 3|}{3}
=
\frac{3 \log\left( \frac{2 |\mu_4 - 3|^2}{9 B^8} T \right)}{2 T}
<
\frac{3 \log T}{T}
\le
\frac13
,\qquad
\frac{2 |\mu_4 - 3|^2}{9 B^8} T
\ge
e
.
\eeq
Plugging in $\eta = \eta(T)$ (as in \eqref{etaPickT}) to \eqref{eq:Tetatau*_relation2} and we obtain
\beq\label{eq:T_eta_lb}
\begin{aligned}
\lefteqn{
T_{\eta(T), \tau}^*
\ge
\frac{12 \tau}{5 |\mu_4 - 3| }
\cdot
\eta(T)^{-1} \log\left(\frac{|\mu_4 - 3|}{B^8} \cdot \eta(T)^{-1}\right)
}
\\&=
\frac{12 \tau }{5}
\cdot
\frac{2 T}{9 \log\left(\frac{2 |\mu_4 - 3|^2}{9 B^8} T\right)} \log\left(\frac{2 |\mu_4 - 3|^2 T}{9 B^8 \log\left(\frac{2 |\mu_4 - 3|^2}{9 B^8} T\right)}\right)
\\&\ge
\frac{8 \tau}{15}
\cdot
\frac{T}{\log\left(\frac{2 |\mu_4 - 3|^2}{9 B^8} T\right)}
\cdot
\frac12 \log\left(\frac{2 |\mu_4 - 3|^2}{9 B^8} T\right)
\ge
\frac{4 \tau}{15} T
,
\end{aligned}\eeq
where we use the elementary inequality $\log\left(\frac{x}{\log x}\right) \ge \frac12 \log x$ for all $x \ge e$.

\item
From \eqref{eq:T_eta_ub} and \eqref{eq:T_eta_lb} we know that $T \in [T_{\eta(T), 0.5}^*, T_{\eta(T), 4}^*]$. Here we will verify scaling condition \eqref{eq:warm_scaling} required in Lemma \ref{lemm:convergence_result2} under our setting.
By choosing
\beq\label{epsstate}
\epsilon
\equiv
\left(
36 + \frac{5184}{\log^5 \delta^{-1}}
\right) d \delta
,
\eeq
we have 
\beq\label{eq:warm_epsilon}
T_{\eta(T), 1}^* \delta^{-1}
\le
C_{\ref{theo:finite_sample2}, T}' \epsilon^{-1} d T
,
\eeq
where we define constant $C_{\ref{theo:finite_sample2}, T}' \equiv (4/3) \cdot (36 + 5184) = 6960$ and use results $T_{\eta(T), 1}^* \le 4 T / 3$, $\delta \le e^{-1}$ implied by \eqref{eq:T_eta_ub}, $\epsilon \le 1$.

Therefore for the first scaling condition in \eqref{eq:warm_scaling}, our pick of $\tau = 4$ requires
$$
22.5 C_{\ref{lemm:convergence_result2}, L}^* \log^8(T_{\eta(T), 1}^* \delta^{-1})
\cdot
\frac{d \log\left(\frac{2 |\mu_4 - 3|^2}{9 B^8} T\right)}{T} \log\left(\frac{2 |\mu_4 - 3|^2 T}{9 B^8 \log\left(\frac{2 |\mu_4 - 3|^2}{9 B^8} T\right)} \right)
\le
\frac{|\mu_4 - 3|^2}{B^8}
,
$$
while a sufficient condition for the above to hold is, due to \eqref{eq:warm_epsilon},
\beq\label{eq:warm_scaling_theo-verification1}
C_{\ref{theo:finite_sample2}, T}^* \log^8(C_{\ref{theo:finite_sample2}, T}' \epsilon^{-1} d T)
\cdot
\frac{d \log^2 T}{T}
\le
\frac{|\mu_4 - 3|^2}{B^8}
,
\eeq
which comes from \eqref{eq:warm_scaling2} and constant definition $C_{\ref{theo:finite_sample2}, T}^* \equiv 22.5 C_{\ref{lemm:convergence_result2}, L}^* \cdot 2^2 = 90 C_{\ref{lemm:convergence_result2}, L}^*$, because $T \ge 100$ and the relation \eqref{eq:B_mu4_relation} of $B, |\mu_4 - 3|$ in Appendix Lemma \ref{lemm:B_mu4_relation} imply $T \ge \frac{2 |\mu_4 - 3|^2}{9 B^8}$, and hence
\beq\label{eq:warm_scaling_useful}
1
\le
\log\left(\frac{2 |\mu_4 - 3|^2}{9 B^8} T\right)
\le
2 \log T
.
\eeq

To verify the second condition in \eqref{eq:warm_scaling}, using scaling condition \eqref{eq:warm_scaling2} in Theorem \ref{theo:finite_sample2} and \eqref{eq:warm_scaling_useful}, we have
$$
\frac{B^8}{|\mu_4 - 3|} \eta(T)
=
\frac{9 B^8}{2 |\mu_4 - 3|^2} \cdot \frac{\log\left(\frac{2 |\mu_4 - 3|^2}{9 B^8} T\right)}{T}
\le
\frac{9 B^8}{|\mu_4 - 3|^2} \cdot \frac{\log T}{T}
<
e^{-1}
,
$$
Note that we have already verified $|\mu_4 - 3| \eta(T) < 1$ in \eqref{eq:warm_scaling_verify_mid2}.
So far we have shown that, under scaling condition \eqref{eq:warm_scaling2} in Theorem \ref{theo:finite_sample2}, for stepsize $\eta = \eta(T)$ given in \eqref{etaPickT} the scaling condition \eqref{eq:warm_scaling} in Lemma \ref{lemm:convergence_result2} is guaranteed to hold, which justifies our following act on proving Theorem \ref{theo:finite_sample2} with Lemma \ref{lemm:convergence_result2}.

\item
Using warm initialization condition \eqref{eq:lemm_warm} in Lemma \ref{lemm:convergence_result2} and the definition of $T_{\eta, \tau}^*$ given in \eqref{eq:Tetatau*}, for all $t \in [T_{\eta, 0.5}^*, T_{\eta, 4}^*]$ we have
$$\begin{aligned}
\lefteqn{
\left| \tan \angle\left(
\ub^{(0)}, \ab_i
\right) \right|
\left(  1 - \frac{\eta}{3} |\mu_4 - 3| \right)^t
\le
\frac{1}{\sqrt{3}}
\cdot
\frac{B^4}{|\mu_4 - 3|^{1 / 2}} \eta^{1 / 2}
}
\\&\le
\frac{1}{\sqrt{3}} \log^{5 / 2}\delta^{-1}
\cdot
\frac{B^4}{|\mu_4 - 3|^{1 / 2}}
\cdot
\sqrt{
d \eta \log\left(\frac{|\mu_4 - 3|}{B^8}\eta^{-1}\right)
}
,
\end{aligned}$$
where the first inequality comes from $t \ge T_{\eta, 0.5}^*$ and definition of $T_{\eta, \tau}^*$ in \eqref{eq:Tetatau*}, and the second inequality is due to $\delta^{-1} \ge e$ and $\frac{|\mu_4 - 3|}{B^8}\eta^{-1} \ge e$ given by scaling condition \eqref{eq:warm_scaling}.
Therefore, on the event $\cH_{\ref{lemm:convergence_result2}, L}$ we have for all $t \in [T_{\eta, 0.5}^*, T_{\eta, 4}^*]$
\beq\label{eq:warm_lemm-theo_transit}
\left| \tan \angle\left(\ub^{(t)}, \ab_i\right) \right|
\le
3 C_{\ref{lemm:convergence_result2}, L} \log^{5 / 2}\delta^{-1}
\cdot
\frac{B^4}{|\mu_4 - 3|^{1 / 2}}
\cdot
\sqrt{
d \eta \log\left(\frac{|\mu_4 - 3|}{B^8}\eta^{-1}\right)
}
.
\eeq

To finalize our proof, we plug in $\eta = \eta(T)$ to \eqref{eq:warm_lemm-theo_transit} and conclude from $T \in [T_{\eta(T), 0.5}^*, T_{\eta(T), 4}^*]$ that there exists an event $\cH_{\ref{theo:finite_sample2}, T}$ equivalent to $\cH_{\ref{lemm:convergence_result2}, L}$ with, due to \eqref{epsstate},
$
\PP(\cH_{\ref{theo:finite_sample2}, T})
\ge
1 - \epsilon
,
$
such that on event $\cH_{\ref{theo:finite_sample2}, T}$ the following holds
\beq\label{tanbdd} \begin{aligned}
\lefteqn{
\left| \tan \angle\left(\ub^{(T)}, \ab_i\right) \right|
}
\\&\le
3 C_{\ref{lemm:convergence_result2}, L} \log^{5 / 2}\delta^{-1}
\cdot
\frac{B^4}{|\mu_4 - 3|^{1 / 2}}
\cdot
\sqrt{d \eta(T) \log\left(\frac{|\mu_4 - 3|}{B^8}\eta(T)^{-1}\right)}
\\&=
\frac{9\sqrt{2}}{2} C_{\ref{lemm:convergence_result2}, L} \log^{5 / 2}\delta^{-1}
\cdot
\frac{B^4}{|\mu_4 - 3|}
\cdot
\sqrt{\frac{d\log\left(\frac{2 |\mu_4 - 3|^2}{9 B^8} T\right)}{T} \log\left(\frac{2 |\mu_4 - 3|^2 T}{9 B^8\log\left(\frac{2 |\mu_4 - 3|^2}{9 B^8} T\right)} \right)}
\\&\le
C_{\ref{theo:finite_sample2}, T} \log^{5 / 2}(C_{\ref{theo:finite_sample2}, T}' d \epsilon^{-1})
\cdot
\frac{B^4}{|\mu_4 - 3|}
\cdot
\sqrt{\frac{d \log^2 T}{T}}
,
\end{aligned}\eeq
where in the last step we apply \eqref{eq:warm_scaling_useful}, $
\log^{5 / 2}\delta^{-1} \le \log^{5 / 2}(C_{\ref{theo:finite_sample2}, T}' d \epsilon^{-1})
$ from \eqref{epsstate}, with constants $C_{\ref{theo:finite_sample2}, T} \equiv 9\sqrt{2} C_{\ref{lemm:convergence_result2}, L}, C_{\ref{theo:finite_sample2}, T}' \equiv 6960$.
This completes the whole proof of the theorem.

\end{enumerate}

\end{proof}

\pb\subsection{Proof of Lemma \ref{lemm:convergence_result}}\label{sec:uniform_lemm}
In uniform initialization analysis, intuitively $\vb^{(t)}$ needs to enter intermediate region $\cD_{\text{mid}, k}$ first before we worry about its exit of the intermediate-auxilliary region $\cD_{\text{mid-aux}, k}$.
For each coordinate $k \in [2, d]$, we upper bound the stopping time $\cT_{c, k}$ earlier defined in \eqref{eq:cTx} using cold initialization Lemmas \ref{lemm:initW} and \ref{lemm:hitt}.
For each coordinate $k \in [2, d]$, we view the iterative process $\{U_k^{(t)} \equiv v_k^{(t)} / v_1^{(t)}\}$ as a Markov chain with $\vb$ initialized in region $\cD_{\text{mid-aux}, k}$ shifted by $\cT_{c, k}$.
We notice that process $\{U_k^{(t)} 1_{(t < \cT_1)}\}$ is bounded by 1 due to definition of $U_k^{(t)}, \cT_1$, and we have $U_k^{(t)} 1_{(t < \cT_1)} \equiv U_k^{(t)}$ for all $(t < \cT_{w, k})$.
Due to strong Markov property the intermediate initialization Lemma \ref{lemm:expo} applies to the shifted Markov chain.
\red{Careful, why strong Markov here when we dont have unboundedness?}

\begin{proof}[Proof of Lemma \ref{lemm:convergence_result}]
We let constant $C_{\ref{lemm:convergence_result}, L}^* \equiv \max\{256 C_{\ref{lemm:hitt}, L}^2, 476 C_{\ref{lemm:expo}, L}^2\}$ in scaling condition \eqref{eq:uniform_scaling} in Lemma \ref{lemm:convergence_result}.

\begin{enumerate}[label=(\arabic*)]
\item
We start by making coordinate-wise analysis for each $k \in [2, d]$.
For all $t \le T_{\eta, 0.5}^o \wedge \cT_{c, k} \wedge \cT_1$, on the event $\cH_{\ref{lemm:initW}, L} \cap \cH_{k; \ref{lemm:hitt}, L}$, since $(v_1^{(t - 1)})^2 \ge 1 / d$, by applying Lemmas \ref{lemm:initW} and \ref{lemm:hitt} we have
\beq\label{eq:W_preprocess}
\begin{aligned}
&
W_k^{(t)}
\ge
\left(
W_k^{(0)}
-
C_{\ref{lemm:hitt}, L} \log^{5 / 2}\delta^{-1}
\cdot \frac{B^4}{|\mu_4 - 3|^{1 / 2}}
\cdot d \eta^{1 / 2}
\right) \left(1 + \frac{\eta}{d} |\mu_4 - 3|\right)^t
\\&\ge
\left(
\frac{\epsilon}{8 \log\epsilon^{-1} \log d}
-
C_{\ref{lemm:hitt}, L} \log^{5 / 2}\delta^{-1}
\cdot
\frac{B^4}{|\mu_4 - 3|^{1 / 2}}
\cdot
d \eta^{1 / 2}
\right)
\cdot
\left(1 + \frac{\eta}{d} |\mu_4 - 3|\right)^t
\ge
0
,
\end{aligned}\eeq
where in the last step we used the following inequality implied by scaling condition \eqref{eq:uniform_scaling}
\beq\label{eq:uniform_scaling_sub4}
\frac{\epsilon}{8 \log\epsilon^{-1} \log d}
\ge
2 C_{\ref{lemm:hitt}, L} \log^{5 / 2}\delta^{-1}
\cdot
\frac{B^4}{|\mu_4 - 3|^{1 / 2}}
\cdot
d \eta^{1 / 2}
.
\eeq
Using the elementary inequality $-\log(1 - x) \le \log(1 + 2x)$ for all $0 \le x \le 1 / 2$, we have
\beq\label{eq:Tetatauo_power}
\begin{aligned}
\lefteqn{
\left(1 + \frac{\eta}{d} |\mu_4 - 3|\right)^{T_{\eta, 0.5}^o}
}
\\&\ge
\exp\left(
\frac12 \log\left(\frac{|\mu_4 - 3|}{B^8} \eta^{-1}\right)
\cdot
\frac{\log\left(1 + \frac{|\mu_4 - 3|}{d} \eta\right)}{-\log\left(1 - \frac{|\mu_4 - 3|}{2d} \eta\right)}
\right)
\ge
\frac{|\mu_4 - 3|^{1 / 2}}{B^4} \eta^{-1 / 2}
,
\end{aligned}\eeq
since $\frac{|\mu_4 - 3|}{2d} \eta \le 1 / 2$ under scaling condition \eqref{eq:uniform_scaling}.
Hence on the event
$$
\cH_{\ref{lemm:initW}, L}
\cap \cH_{k; \ref{lemm:hitt}, L}
\cap (\cT_1 > T_{\eta, 0.5}^o)
\cap (\cT_{c, k} > T_{\eta, 0.5}^o)
,
$$
using \eqref{eq:W_preprocess}, \eqref{eq:uniform_scaling_sub4}, \eqref{eq:Tetatauo_power} and $\delta \in (0, e^{-1}]$, we have
\beq\label{eq:W_process}
W_k^{(T_{\eta, 0.5}^o)}
\ge
C_{\ref{lemm:hitt}, L} \log^{5 / 2}\delta^{-1}
\cdot\frac{B^4}{|\mu_4 - 3|^{1 / 2}}
\cdot d \eta^{1 / 2}
\cdot \frac{|\mu_4 - 3|^{1 / 2}}{B^4} \eta^{-1 / 2}
=
C_{\ref{lemm:hitt}, L} \log^{5 / 2}\delta^{-1}
\cdot d
\ge
2
,
\eeq
which indicates $\vb^{(T_{\eta, 0.5}^o)} \in \cD_{\text{mid-aux}, k}$, i.e. event $\cH_{\ref{lemm:initW}, L} \cap \cH_{k; \ref{lemm:hitt}, L} \cap (\cT_1 > T_{\eta, 0.5}^o) \cap (\cT_{c, k} > T_{\eta, 0.5}^o) \subseteq (\cT_{c, k} \le T_{\eta, 0.5}^o)$, and hence
\beq\label{eq:cold_stopping}
\cH_{\ref{lemm:initW}, L} \cap \cH_{k; \ref{lemm:hitt}, L} \cap (\cT_1 > T_{\eta, 0.5}^o)
\subseteq
(\cT_{c, k} \le T_{\eta, 0.5}^o)
.
\eeq

\item
To study $\vb^{(t)}$ in the intermediate region, we first assume $\cT_{c, k} = 0$, that is, the initialization $\vb^{(0)} \in \cD_{\text{mid-aux}, k}$.
For all $t \le \cT_{w, k} \wedge T_{\eta, \tau}^o$, on the event $\cH_{k; \ref{lemm:expo}, L}$, since $
(v_1^{(t - 1)})^2 - (v_k^{(t - 1)})^2 \ge 1 / (2d)
$ and the following is guaranteed by scaling condition \eqref{eq:uniform_scaling}
\beq\label{eq:uniform_scaling_sub5}
\frac{|\mu_4 - 3|}{2d} \eta
\le	1
,\qquad
2 C_{\ref{lemm:expo}, L} B^4 \log^{5 / 2}\delta^{-1} \cdot d^{1 / 2} \eta (T_{\eta, \tau}^o)^{1 / 2}
\le	\sqrt{\frac{8}{476}}
,
\eeq
applying Lemma \ref{lemm:expo} we have
\beq\label{eq:U_preprocess}
\begin{aligned}
&
|U_k^{(t)}|
\le
|U_k^{(0)}|\left(1 - \frac{|\mu_4 - 3| \eta}{2d}\right)^t
+
2 C_{\ref{lemm:expo}, L} B^4 \log^{5 / 2}\delta^{-1} \cdot d^{1 / 2} \eta (T_{\eta, \tau}^o)^{1 / 2}
\\&<
\frac{1}{\sqrt{3}} \cdot 1 + \sqrt{\frac{8}{476}}
\le
\frac{1}{\sqrt{2}}
.
\end{aligned}\eeq
Now we consider uniform initialization case.
Using strong Markov property as discussed earlier, we have for all $t \in [\cT_{c, k}, \cT_{w, k} \wedge T_{\eta, \tau}^o]$
\beq\label{eq:strong_markov_U}
\left|
U_k^{(t)}
\right|
\le
\frac{1}{\sqrt{2}}
.
\eeq
On the event
$$
\cH_{k; \ref{lemm:expo}, L} 
\cap \cH_{\ref{lemm:initW}, L} 
\cap \cH_{k; \ref{lemm:hitt}, L} 
\cap (\cT_1 > T_{\eta, 0.5}^o) 
\cap (\cT_{w, k} \le T_{\eta, \tau}^o)
,
$$
we have already proven $\cT_{c, k} \le T_{\eta, \tau}^o (\tau > 0.5)$ in \eqref{eq:cold_stopping}.
Applying \eqref{eq:strong_markov_U} with $t = \cT_{w, k} = \cT_{w, k} \wedge T_{\eta, \tau}^o$, we obtain $|U_k^{(\cT_{w, k})}| \le 1 / \sqrt{2}$, which leads to contradiction with the definition of $\cT_{w, k}$ in \eqref{cTw} and indicates $\cH_{k; \ref{lemm:expo}, L} \cap \cH_{\ref{lemm:initW}, L} \cap \cH_{k; \ref{lemm:hitt}, L} \cap (\cT_1 > T_{\eta, 0.5}^o) \cap (\cT_{w, k} \le T_{\eta, \tau}^o) = \varnothing$, i.e.
\beq\label{eq:mid_stopping}
\cH_{k; \ref{lemm:expo}, L} \cap \cH_{\ref{lemm:initW}, L} \cap \cH_{k; \ref{lemm:hitt}, L} \cap (\cT_1 > T_{\eta, 0.5}^o)
\subseteq
(\cT_{w, k} > T_{\eta, \tau}^o)
.
\eeq

\item
With \eqref{eq:cold_stopping} and \eqref{eq:mid_stopping} proven, we put all coordinates $k \in [2, d]$ together and define event
$$
\cH_{\ref{lemm:convergence_result}, L}
\equiv
\left(\cap_{k\in [2,d]} \cH_{k; \ref{lemm:expo}, L}\right) 
\cap \cH_{\ref{lemm:initW}, L}
\cap \left(\cap_{k\in [2,d]} \cH_{k; \ref{lemm:hitt}, L}\right)
.
$$
On the event
$
\cH_{\ref{lemm:convergence_result}, L} \cap (\cT_1 \le T_{\eta, 0.5}^o)
,
$
for each coordinate $k \in [2, d]$ satisfying $\cT_1 \le T_{\eta, 0.5}^o \wedge \cT_{c, k}$, we apply \eqref{eq:W_preprocess} with $t = \cT_1 = T_{\eta, 0.5}^o \wedge \cT_{c, k} \wedge \cT_1$ and obtain $W_k^{(\cT_1)} \ge 0$.
Then due to the definition of $\cT_1$ in \eqref{eq:T1}, there must exist $k \in [2, d]$ such that $W_k^{(\cT_1)} < 0$ and $\cT_{c, k} < \cT_1 \le T_{\eta, 0.5}^o$.
Recall that $W_k^{(t)} < 0$ is equivalent to $|U_k^{(t)}| > 1$.
By definitions of stopping times $\cT_1, \cT_{w, k}$, we know the existence of $k \in [2, d]$ such that $\cT_{c, k} < \cT_{w, k} \le \cT_1 \le T_{\eta, 0.5}^o$ on the event $\cH_{\ref{lemm:convergence_result}, L} \cap (\cT_1 \le T_{\eta, 0.5}^o)$.
By applying \eqref{eq:strong_markov_U} with $t = \cT_{w, k} = \cT_{w, k} \wedge T_{\eta, \tau}^o$, we find $|U_k^{(\cT_{w, k})}| \le 1 / \sqrt{2}$, which contradicts with the definition of $\cT_{w, k}$ and indicates that $\cH_{\ref{lemm:convergence_result}, L} \cap (\cT_1 \le T_{\eta, 0.5}^o) = \varnothing$, i.e.
\beq\label{eq:max_stopping}
\cH_{\ref{lemm:convergence_result}, L}
\subseteq
(\cT_1 > T_{\eta, 0.5}^o)
.
\eeq
Combining \eqref{eq:cold_stopping} and \eqref{eq:mid_stopping} for each $k \in [2, d]$ along with \eqref{eq:max_stopping}, we have
\beq\label{eq:all_stopping}
\cH_{\ref{lemm:convergence_result}, L}
\subseteq
\left(\sup_{2 \le k \le d} \cT_{c, k}\le T_{\eta, 0.5}^o\right)
\cap
\left(\inf_{2 \le k \le d} \cT_{w, k} > T_{\eta, \tau}^o\right)
\cap
(\cT_1 > T_{\eta, 0.5}^o)
.
\eeq

\item
\red{Remark \ref{rema:uniform_story} interprets \eqref{eq:all_stopping} and depicts the story of convergence with uniform initialization.
}With \eqref{eq:all_stopping} ready at hand, on the event $\cH_{\ref{lemm:convergence_result}, L}$, we apply Lemma \ref{lemm:expo} for each coordinate $k \in [2, d]$ and all $t \in [T_{\eta, 0.5}^o, T_{\eta, \tau}^o]$ to obtain
\beq\label{eq:lemm:convergence_result-norm1}
\begin{aligned}
&
\sqrt{ \sum_{k = 2}^d \left(
U_k^{(t)}  - U_k^{(T_{\eta, 0.5}^o)} \prod_{s = T_{\eta, 0.5}^o}^{t - 1} \left[ 1 - \eta |\mu_4 - 3| \left( (v_1^{(s)})^2 - (v_k^{(s)})^2 \right) \right]
\right)^2 }
\\&\le
2 C_{\ref{lemm:expo}, L} \log^{5 / 2}\delta^{-1}
\cdot B^4
\cdot d \eta (T_{\eta, \tau}^o)^{1 / 2}
.
\end{aligned}\eeq
On the event $\cH_{\ref{lemm:convergence_result}, L}$, we have $\cT_{c, k} \le T_{\eta, 0.5}^o$, $(v_1^{(s)})^2 - (v_k^{(s)})^2 \ge 1 / (2d)$, $\left|U_k^{(\cT_{c, k})}\right| \le 1$ for all $k \in [2, d]$, $s \in [T_{\eta, 0.5}^o, T_{\eta, \tau}^o)$. Since the left hand of \eqref{eq:lemm:convergence_result-norm1} is the norm of two vectors subtraction, we use the triangle inequality of Euclidean norms to lower bound it as
\beq\label{eq:lemm:convergence_result-norm2}
\begin{aligned}
\lefteqn{
\sqrt{ \sum_{k = 2}^d \left(
U_k^{(t)}  - U_k^{(T_{\eta, 0.5}^o)} \prod_{s = T_{\eta, 0.5}^o}^{t - 1} \left[ 1 - \eta |\mu_4 - 3| \left( (v_1^{(s)})^2 - (v_k^{(s)})^2 \right) \right]
\right)^2 }
}
\\&\ge
\sqrt{ \sum_{k = 2}^d (U_k^{(t)})^2}
-
\sqrt{ \sum_{k = 2}^d \left(
U_k^{(T_{\eta, 0.5}^o)} \prod_{s = T_{\eta, 0.5}^o}^{t - 1} \left[ 1 - \eta |\mu_4 - 3| \left( (v_1^{(s)})^2 - (v_k^{(s)})^2 \right) \right]
\right)^2 }
\\&\ge
\left| \tan \angle \left(\vb^{(t)}, \eb_1\right) \right|
-
\sqrt{d}
\cdot \left(1 - \frac{\eta}{2d} |\mu_4 - 3|\right)^{t - T_{\eta, 0.5}^o}
.
\end{aligned}\eeq
Recall the definition of $T_{\eta, \tau}^o$ in \eqref{Tetatau}.
Scaling condition \eqref{eq:uniform_scaling} guarantees the following
\beq\label{eq:uniform_scaling_sub3}
\frac{|\mu_4 - 3|}{2 d} \eta
\le 1
,\qquad
\frac{B^8}{|\mu_4 - 3|} \eta
\le e^{-1}
.
\eeq
Since $- \log\left( 1 - \frac{\eta}{2d} |\mu_4 - 3| \right) \ge \frac{\eta}{2d} |\mu_4 - 3|$ and \eqref{eq:uniform_scaling_sub3} holds, for each positive $\tau$ we have relation
\beq\label{eq:Tetatauo_relation}
T_{\eta, \tau}^o
\le
1 + \frac{2 \tau d}{|\mu_4 - 3|} \cdot \eta^{-1} \log\left(\frac{|\mu_4 - 3|}{B^8} \eta^{-1}\right)
\le
\frac{2(\tau + 1) d}{|\mu_4 - 3|} \cdot \eta^{-1} \log\left(\frac{|\mu_4 - 3|}{B^8} \eta^{-1}\right)
.
\eeq
Combining \eqref{eq:lemm:convergence_result-norm1}, \eqref{eq:lemm:convergence_result-norm2} together and using relation \eqref{eq:Tetatauo_relation}, we have
\beq\label{eq:lemm:convergence_result-final}
\begin{aligned}
\left| \tan \angle \left(\vb^{(t)}, \eb_1\right) \right|
&\le
\sqrt{d}
\cdot
\left(1 - \frac{\eta}{2d} |\mu_4 - 3|\right)^{t - T_{\eta, 0.5}^o}
\\&\quad
+
\sqrt{\tau + 1} C_{\ref{lemm:convergence_result}, L} \log^{5 / 2}\delta^{-1}
\cdot
\frac{B^4}{|\mu_4 - 3|^{1 / 2}}
\cdot
\sqrt{d^3 \eta \log\left(\frac{|\mu_4 - 3|}{B^8}\eta^{-1}\right)}
,
\end{aligned}\eeq
where constant $C_{\ref{lemm:convergence_result}, L} \equiv 2\sqrt{2} C_{\ref{lemm:expo}, L}$.
To complete proof of Lemma \ref{lemm:convergence_result}, we provide a lower bound on probability of event $\cH_{\ref{lemm:convergence_result}, L}$ by taking union bound
$$
\PP(\cH_{\ref{lemm:convergence_result}, L})
\ge
1 - \sum_{k = 2}^d \PP(\cH_{k; \ref{lemm:expo}, L}^c) - \PP(\cH_{\ref{lemm:initW}, L}^c) - \sum_{k = 2}^d \PP(\cH_{k; \ref{lemm:hitt}, L}^c)
\ge
1 - \left(6\tau + 27 + \frac{10368}{\log^5\delta^{-1}}\right) d \delta - 3 \epsilon
.
$$
Applying the scaling relation \eqref{eq:uv} of $\{\ub^{(t)}\}_{t \ge 0}$ and $\{\vb^{(t)}\}_{t \ge 0}$ to \eqref{eq:lemm:convergence_result-final} on the event $\cH$ completes the proof of \eqref{eq:uniform_cr}, and hence Lemma \ref{lemm:convergence_result}.

\end{enumerate}

\end{proof}

\pb\subsection{Proof of Theorem \ref{theo:finite_sample}}\label{sec:uniform_theo}
Now we are ready to derive the finite-sample error bound and prove Theorem \ref{theo:finite_sample}.

\begin{proof}[Proof of Theorem \ref{theo:finite_sample}]

\begin{enumerate}[label=(\arabic*)]
\item
Analogous to the proof of Theorem \ref{theo:finite_sample2}, we sharply upper- and lower-bound the rescaled time $T_{\eta(T), \tau}^o$.
Using relation $\frac{B^4}{|\mu_4 - 3|} \ge \frac18$ in Lemma \ref{lemm:B_mu4_relation} and elementary inequality $\frac{\log T}{T} \le \frac{1}{20}$, we have
\beq\label{eq:uniform_etaT_cond}
\frac{|\mu_4 - 3|}{2d} \eta(T)
=
\frac{2}{T} \log\left(\frac{|\mu_4 - 3|^2}{4 B^8 d} T\right)
\le
\frac13
.
\eeq
Since $\frac{1}{- \log(1 - x)} \le \frac{1}{x}$ for $x \in (0, 1)$, similar to \eqref{eq:Tetatauo_relation} we have
\beq\label{eq:T_eta_ub2}
\begin{aligned}
\lefteqn{
T_{\eta(T), \tau}^o
\le
1 + \frac{2\tau d}{|\mu_4 - 3|} \eta(T)^{-1} \log\left(\frac{|\mu_4 - 3|}{B^8} \eta(T)^{-1}\right)
}
\\&=
1 + \frac{\tau}{2 \log\left(\frac{|\mu_4 - 3|^2}{4 B^8 d} T\right)} T \log\left(
\frac{|\mu_4 - 3|^2 T}{4 B^8 d \log\left(\frac{|\mu_4 - 3|^2}{4 B^8 d} T\right)}
\right)
\\&=
1 + \frac{\tau}{2} T
\cdot
\frac{\log\left(\frac{|\mu_4 - 3|^2}{4 B^8 d} T\right) - \log\left(\log\left(\frac{|\mu_4 - 3|^2}{4 B^8 d} T\right)\right)}{\log\left(\frac{|\mu_4 - 3|^2}{4 B^8 d} T\right)}
\le
\frac{\tau + 1}{2} T - 1
,
\end{aligned}\eeq
where the last inequality holds due to $\frac{|\mu_4 - 3|^2}{4B^8 d} T \ge e$ under scaling condition \eqref{eq:uniform_scaling} with constants $C_{\ref{theo:finite_sample}, T}^* \equiv 96 C_{\ref{lemm:convergence_result}, L}^*, C_{\ref{theo:finite_sample}, T}' = 10425$.

On the lower bound side, previously in \eqref{eq:Taylor1} we show that $\frac{1}{- \log(1 - x)} \ge \frac{4}{5 x}$ for all $x \in (0, 1 / 3]$, which can be applied to $T_{\eta(T), \tau}^o$ by replacing $x$ with $\frac{|\mu_4 - 3|}{2d} \eta(T)$ due to \eqref{eq:uniform_etaT_cond}.
By noticing that $T \ge 100$ and scaling condition \eqref{eq:uniform_scaling2} guarantee the following inequalities
\beq\label{eq:uniform_scaling_new2}
\frac{4 \log T}{T}
\le
\frac13
,\qquad
\frac{8 B^8}{|\mu_4 - 3|^2} \cdot \frac{d \log T}{\sqrt{T}}
\le
1
,
\eeq
we have 
\beq\label{eq:T_eta_lb2}
\begin{aligned}
&
T_{\eta(T), \tau}^o
\ge
\frac{\tau \log\left(\frac{|\mu_4 - 3|}{B^8} \eta(T)^{-1}\right)}{- \log\left(1 - \frac{|\mu_4 - 3|}{2d} \eta(T)\right)}
\ge
\frac{8 \tau d \log\left(\frac{|\mu_4 - 3|}{B^8} \eta(T)^{-1}\right)}{5 |\mu_4 - 3| \eta(T)}
\\&\ge
\frac{2 \tau T}{5 \log\left(\frac{|\mu_4 - 3|^2}{4 B^8 d} T\right)}
\cdot
\log\left(
\frac{|\mu_4 - 3|^2 T}{4 B^8 d \log\left(\frac{|\mu_4 - 3|^2}{4 B^8 d} T\right)}
\right)
\ge
\frac{\tau}{5} T
,
\end{aligned}\eeq
where in the last step we use the elementary inequality $\log\left(\frac{x}{\log x}\right) \ge \frac12 \log x$ for all $x \ge e$, since $\frac{|\mu_4 - 3|^2}{4 B^8 d} T \ge e$ is satisfied under scaling condition \eqref{eq:uniform_scaling}.

\item
From \eqref{eq:T_eta_ub2} and \eqref{eq:T_eta_lb2} we find that $T \in [T_{\eta(T), 1}^o + 1, T_{\eta(T), 5}^o]$.
By letting
$
\epsilon
=
\left(
57 + \frac{10368}{\log^5\delta^{-1}}
\right) d \delta
$
along with \eqref{eq:T_eta_ub2} we have
$$
T_{\eta(T), 1}^o \delta^{-1}
\le
C_{\ref{theo:finite_sample}, T}' \epsilon^{-1} d T
,
$$
for positive, absolute constant $C_{\ref{theo:finite_sample}, T}' \equiv 10425$, due to $T_{\eta(T), 1}^o \le T$ given by \eqref{eq:T_eta_ub2} and $\delta \le e^{-1}$ guaranteed by $\epsilon \le 1 / 4$.

The third scaling condition in \eqref{eq:uniform_scaling} with our pick $\tau = 5$ and $\eta = \eta(T)$ is satisfied by
\beq\label{eq:uniform_scaling_verify3}
\begin{aligned}
&
24 C_{\ref{lemm:convergence_result}, L}^* \log^8(C_{\ref{theo:finite_sample}, T}' \epsilon^{-1} d T)
\cdot 
\frac{B^8}{|\mu_4 - 3|^2}
\cdot
\frac{d^3 \log^2 d \log\left(\frac{|\mu_4 - 3|^2 T}{4 B^8 d}\right)}{T}
\log\left(\frac{|\mu_4 - 3|^2 T}{4 B^8 d \log\left(\frac{|\mu_4 - 3|^2 T}{4 B^8 d}\right)}\right)
\\&\le
\frac{\epsilon^2}{\log^2\epsilon^{-1}}
.
\end{aligned}\eeq
From Lemma \ref{lemm:B_mu4_relation} we have $\frac{B^4}{|\mu_4 - 3|} \ge \frac18$.
Along with scaling condition \eqref{eq:uniform_scaling2} and $C_{\ref{theo:finite_sample}, T}^* \equiv 96 C_{\ref{lemm:convergence_result}, L}^*$, we have $T / d \ge 16$ and hence
\beq\label{eq:uniform_scaling_verify_mid}
1
\le
\log\left(\frac{|\mu_4 - 3|^2 T}{4 B^8 d}\right)
\le
2 \log (T / d)
,
\eeq
implying \eqref{eq:uniform_scaling_verify3} holds under scaling condition \eqref{eq:uniform_scaling2}.

To verify the second scaling condition in \eqref{eq:uniform_scaling}, we notice that the following holds under \eqref{eq:uniform_scaling2}
$$
\frac{B^8}{|\mu_4 - 3|} \eta(T)
=
\frac{4 B^8 d \log\left(\frac{|\mu_4 - 3|^2}{4 B^8 d} T\right)}{|\mu_4 - 3|^2 T}
\le
\frac{8 B^8 d \log T}{|\mu_4 - 3|^2 T}
<
e^{-1}
,
$$
and
$$
\eta(T)
=
\frac{4 d \log\left(\frac{|\mu_4 - 3|^2}{4 B^8 d} T\right)}{|\mu_4 - 3| T}
\le
\frac{8 \log (T / d)}{|\mu_4 - 3| (T / d)}
<
\frac{1}{|\mu_4 - 3|}
,
$$
due to \eqref{eq:uniform_scaling_verify_mid} and the elementary inequality $\frac{8 \log x}{x} < 1$ for all $x \ge 100$.

Therefore, all scaling conditions required in Lemma \ref{lemm:convergence_result} are satisfied by scaling condition \eqref{eq:uniform_scaling2} in Theorem \ref{theo:finite_sample}.

\item
Using $B^8 \eta / |\mu_4 - 3| \le e^{-1}$ and $\delta \le e^{-1}$ given by \eqref{eq:uniform_scaling} and $T_{\eta, 1}^o + 1 - T_{\eta, 0.5}^o \ge T_{\eta, 0.5}^o$ following its definition \eqref{Tetatau}, for all $t \in [T_{\eta, 1}^o + 1, T_{\eta, 5}^o]$, on the event $\cH_{\ref{lemm:convergence_result}, L}$ we have 
$$\begin{aligned}
&
\sqrt{d} \cdot \left(1 - \frac{\eta}{2d}|\mu_4 - 3|\right)^{t - T_{\eta, 0.5}^o}
\le
\sqrt{d} \cdot \left(1 - \frac{\eta}{2d}|\mu_4 - 3|\right)^{T_{\eta, 0.5}^o}
\le
\frac{B^4}{|\mu_4 - 3|^{1 / 2}} \cdot \sqrt{d \eta}
\\&\le
(3 - \sqrt{6}) C_{\ref{lemm:convergence_result}, L} \log^{5 / 2}\delta^{-1}
\cdot \frac{B^4}{|\mu_4 - 3|^{1 / 2}}
\cdot \sqrt{d^3 \eta \log\left(\frac{|\mu_4 - 3|}{B^8}\eta^{-1}\right)}
.
\end{aligned}$$
Therefore, from Lemma \ref{lemm:convergence_result}, on the event $\cH_{\ref{lemm:convergence_result}, L}$ we have for all $t \in [T_{\eta, 1}^o + 1, T_{\eta, 5}^o]$ that
\beq\label{eq:uniform_lemm-theo_transit}
\left| \tan \angle\left(
\ub^{(t)}, \ab_\cI
\right) \right|
\le
3 C_{\ref{lemm:convergence_result}, L} \log^{5 / 2}\delta^{-1}
\cdot
\frac{B^4}{|\mu_4 - 3|^{1 / 2}}
\cdot
\sqrt{
d^3 \eta \log\left(\frac{|\mu_4 - 3|}{B^8}\eta^{-1}\right)
}
.
\eeq

Plugging in our choice of $\tau = 5$ and $\eta(T) = \frac{4 d \log\left(\frac{|\mu_4 - 3|^2}{4 B^8 d} T\right)}{|\mu_4 - 3| T}$ to \eqref{eq:uniform_lemm-theo_transit} and using \eqref{eq:uniform_scaling_verify_mid}, we know that there exists an event $\cH_{\ref{theo:finite_sample}, T} \equiv \cH_{\ref{lemm:convergence_result}, L}$ with
$
\PP(\cH_{\ref{theo:finite_sample}, T})
\ge
1 - 4\epsilon
$
such that on event $\cH_{\ref{theo:finite_sample}, T}$ we have
$$
\left| \tan \angle\left(
\ub^{(T)}, \ab_\cI
\right) \right|
\le
C_{\ref{theo:finite_sample}, T} \log^{5 / 2}(C_{\ref{theo:finite_sample}, T}' \epsilon^{-1} d)
\cdot
\frac{B^4}{|\mu_4 - 3|}
\cdot
\sqrt{
\frac{d^4 \log^2 T}{T}
}
,
$$
where constants $C_{\ref{theo:finite_sample}, T} \equiv 12 C_{\ref{lemm:convergence_result}, L}, C_{\ref{theo:finite_sample}, T}' \equiv 10425$.

\end{enumerate}

\end{proof}

\pb\section{Secondary Lemmas in Warm Initialization Analysis}\label{sec:warm_aux}

For notational simplicity, we denote $\vb \equiv \vb^{(t - 1)}$ and $\bY \equiv \bY^{(t)}$.
We first provide a lemma on Orlicz $\psi_2$-norm of $\vb^\top \bY$ and the relation between $B$ and $\mu_4$.

\begin{lemma}\label{lemm:B_mu4_relation}
Let Assumption \ref{assu:distribution} hold.
For each rotated observation $\bY$ and any unit vector $\vb$, we have Orlicz $\psi_2$-norm $\|\vb^\top \bY\|_{\psi_2} \le B$ and the following relation of $B$ and $\mu_4$
\beq\label{eq:B_mu4_relation}
\frac{B^4}{|\mu_4 - 3|}
\ge
\frac18
.
\eeq
\end{lemma}
With the bound on Orlicz $\psi_2$-norm given in Lemma \ref{lemm:B_mu4_relation} and $T_{\eta, 1}^*$ defined in \eqref{eq:Tetatau*}, we introduce truncation barrier parameter
\beq\label{Mdef}
\rB_*\equiv B \log^{1 / 2}(T_{\eta, 1}^* \delta^{-1})
,
\eeq
where $\delta \in (0, e^{-1}]$ is some fixed positive.
For each coordinate $k\in [2,d]$ define the first time the norm of a data observation exceeds the truncation barrier $\rB_*$ as
\beq\label{cTM}
\cT_{\rB_*, k}
=
\inf\left\{ t\ge 1: 
\left| \vb^{(t-1)}\,^\top \bY^{(t)} \right| > \rB_*
\text{ or } \left|Y_1^{(t)}\right|  > \rB_*
\text{ or } \left|Y_k^{(t)}\right|  > \rB_*
\right\}
,
\eeq

\pb\subsection{Proof of Lemma \ref{lemm:fast_expo}}\label{sec:proof,prop:fast_expo}

For each $t \ge 1$, we define random variable
\beq\label{Uinfquan}
Q_{U, k}^{(t)}
\equiv
U_k^{(t)} - U_k^{(t - 1)}
-
\sign(\mu_4 - 3) \eta
\cdot
(\vb^\top \bY)^3 v_1^{-2} \left(v_1 Y_k - v_k Y_1\right)
.
\eeq

\begin{lemma}\label{lemm:Uinfquan}
Let $\rB$ be any positive value. For each coordinate $k \in [2, d]$ and any $t \ge 1$, on the event 
\beq\label{eq:recursive_event}
\cH_{k; \ref{lemm:Uinfquan}, L}^{(t)}
\equiv
\left(
|\vb^\top \bY| \le \rB
,~
|Y_1| \le \rB
,~
|Y_k| \le \rB
,~
v_1^2 \ge 1 / d
,~
v_1^2 \ge v_k^2
\right)
,
\eeq
for stepsize $\eta \le 1 / (2 \rB^4 d^{1 / 2})$ we have $|Q_{U, k}^{(t)}| \le 4 \rB^8 \eta^2 v_1^{-2}$.
\end{lemma}
 
\begin{lemma}\label{lemm:infiU}
Let Assumption \ref{assu:distribution} hold.
For each coordinate $k \in [2, d]$ and any $t \ge 1$, we have
\beq\label{UinfE}
\EE\left[ \left.
\sign(\mu_4 - 3) \eta
\cdot
(\vb^\top \bY)^3 v_1^{-2} \left(v_1 Y_k - v_k Y_1\right)
\right| \cF_{t - 1} \right] 
 =
- \eta |\mu_4 - 3| \cdot (v_1^2 - v_k^2) U_k^{(t - 1)}
.
\eeq
\end{lemma}
For each $k \in [2, d]$ and $t \ge 1$, at the $t$-th iteration we define
\beq\label{ekt}
e_k^{(t)}
\equiv 
\sign(\mu_4 - 3) \eta
\cdot
(\vb^\top \bY)^3 v_1^{-2} \left(v_1 Y_k - v_k Y_1\right)
+
\eta |\mu_4 - 3|\cdot ( v_1^2 - v_k^2 ) U_k^{(t-1)}
\eeq
which, indexed by $t$, forms a sequence of martingale differences with respect to $\cF_{t - 1}$.

\begin{lemma}\label{lemm:representation}
For each coordinate $k \in [2, d]$ and any $t \ge 1$, $U_k^{(t)}$ has linear representation
\beq\label{Urep}
U_k^{(t)} 
 = 
U_k^{(0)}
-
\eta |\mu_4 - 3| \sum_{s = 0}^{t - 1} \left( (v_1^{(s)})^2 - (v_k^{(s)})^2 \right) U_k^{(s)}
+
\sum_{s = 1}^t Q_{U, k}^{(s)}
+
\sum_{s = 1}^t e_k^{(s)}
.
\eeq
\end{lemma}

\begin{lemma}\label{lemm:1order_concentration}
Let Assumption \ref{assu:distribution} hold and initialization $\vb^{(0)} \in \cD_{\text{warm}}$.
Let $\delta \in (0, e^{-1}]$ and $\tau$ be any fixed positive.
For each coordinate $k\in [2,d]$, there exists an event $\cH_{k; \ref{lemm:1order_concentration}, L}$ satisfying
$$
\PP(\cH_{k; \ref{lemm:1order_concentration}, L})
\ge
1 - \left(
6 + \frac{5184}{\log^5\delta^{-1}}
\right)\delta
,
$$
such that on the event $\cH_{k; \ref{lemm:1order_concentration}, L}$ the following concentration result holds
$$
\max_{1 \le t \le T_{\eta, \tau}^* \wedge \cT_x} \left|
\sum_{s = 1}^t e_k^{(s)}
\right|
\le
C_{\ref{lemm:1order_concentration}, L} \log^{5 / 2}\delta^{-1}
\cdot
B^4
\cdot
\eta (T_{\eta, \tau}^*)^{1 / 2}
,
$$
where $C_{\ref{lemm:1order_concentration}, L}$ is a positive, absolute constant.
\end{lemma}

\begin{lemma}\label{lemm:tail_probability}
Let $\delta \in (0, e^{-1}]$ and $\tau$ be any fixed positive.
For each coordinate $k \in [2, d]$, we have
\beq\label{tailcTM}
\PP\left(
\cT_{\rB_*, k}
\le
T_{\eta, \tau}^*
\right)
\le
6 (\tau + 1) \delta
.
\eeq
\end{lemma}

\pb

With the above secondary lemmas at hand, we are now ready to prove Lemma \ref{lemm:fast_expo}.

\begin{proof}[Proof of Lemma \ref{lemm:fast_expo}]
We recall the definition of stopping time $\cT_x$ in \eqref{eq:cTx}.
Because stepsize $\eta \le 1 / (2 \rB_*^4 d^{1 / 2})$ holds under scaling condition \eqref{eq:warm_scaling}, on the event $(t \le T_{\eta, \tau}^* \wedge \cT_{x}) \cap (\cT_{\rB_*, k} > T_{\eta, \tau}^*) \subseteq \cH_{k; \ref{lemm:Uinfquan}, L}^{(t)}$, we have $v_1^{-2} \le \frac32$, and applying Lemma \ref{lemm:Uinfquan} gives
$$
|Q_{U, k}^{(t)}|
\le
4 \rB_*^8 \eta^2 v_1^{-2}
\le
6 B^8 \eta^2 \log^4(T_{\eta, 1}^* \delta^{-1})
.
$$
We define event $\cH_{k; \ref{lemm:fast_expo}, L} \equiv (\cT_{\rB_*, k} > T_{\eta, \tau}^*) \cap \cH_{k; \ref{lemm:1order_concentration}, L}$.
Applying Lemmas \ref{lemm:representation} and \ref{lemm:1order_concentration}, on the event $\cH_{k; \ref{lemm:fast_expo}, L}$ for all $t \le T_{\eta, \tau}^* \wedge \cT_{x}$ we have
\beq\label{prf_warm_prop_1}
\begin{aligned}
\lefteqn{
\left|
U_k^{(t)} - U_k^{(0)} + \eta |\mu_4 - 3| \sum_{s=0}^{t-1} \left( (v_1^{(s)})^2 - (v_k^{(s)})^2 \right) U_k^{(s)}
\right|
}
\\&\le
T_{\eta, \tau}^*
\cdot 6 B^8 \eta^2 \log^4(T_{\eta, 1}^* \delta^{-1})
+
C_{\ref{lemm:1order_concentration}, L} \log^{5 / 2}\delta^{-1}
\cdot B^4
\cdot \eta (T_{\eta, \tau}^*)^{1 / 2}
.
\end{aligned}\eeq
Scaling condition \eqref{eq:warm_scaling} and definition of $T_{\eta, \tau}^*$ in \eqref{eq:Tetatau*} imply that
\beq\label{eq:warm_scaling_sub1}
B^4 \eta (T_{\eta, \tau}^*)^{1 / 2} \log^4(T_{\eta, 1}^* \delta^{-1})
\le
\log^{5 / 2}\delta^{-1}
,
\eeq
Combining \eqref{prf_warm_prop_1} and \eqref{eq:warm_scaling_sub1}, we have
$$
\left|
U_k^{(t)} - U_k^{(0)} + \eta |\mu_4 - 3| \sum_{s=0}^{t-1} \left( (v_1^{(s)})^2 - (v_k^{(s)})^2 \right) U_k^{(s)}
\right|
\le
C_{\ref{lemm:fast_expo}, L} \log^{5 / 2}\delta^{-1}
\cdot B^4
\cdot \eta (T_{\eta, \tau}^*)^{1 / 2}
,
$$
where constant $C_{\ref{lemm:fast_expo}, L} \equiv C_{\ref{lemm:1order_concentration}, L} + 6$.
Scaling condition \eqref{eq:warm_scaling} implies $\eta |\mu_4 - 3| ((v_1^{(s)})^2 - (v_k^{(s)})^2) \in [0, 1)$ for all $s < T_{\eta, \tau}^* \wedge \cT_x$.
Using reversed Gronwall Lemma \ref{lemm:grw_type}, on the event $\cH_{k; \ref{lemm:fast_expo}, L}$ we have the following holds for all $t \le T_{\eta, \tau}^* \wedge \cT_x$
$$
\left|
U_k^{(t)}  - U_k^{(0)} \prod_{s = 0}^{t - 1} \left[
  1 - \eta |\mu_4 - 3|  \left( (v_1^{(s)})^2 - (v_k^{(s)})^2 \right)
\right]
\right|
\le
2 C_{\ref{lemm:fast_expo}, L} \log^{5 / 2}\delta^{-1}
\cdot B^4
\cdot \eta (T_{\eta, \tau}^*)^{1 / 2}
.
$$
To complete proof of Lemma \ref{lemm:fast_expo}, we apply Lemmas \ref{lemm:1order_concentration} and \ref{lemm:tail_probability} and take union bound to obtain
$$
\PP(\cH_{k; \ref{lemm:fast_expo}, L})
\ge
1 - \PP(\cT_{\rB_*, k} \le T_{\eta, \tau}^*) - \PP(\cH_{k; \ref{lemm:1order_concentration}, L}^c)
\ge
1 - \left(6\tau + 12 + \frac{5184}{\log^5\delta^{-1}}\right)\delta
.
$$
\end{proof}

\pb\subsection{Proof of Secondary Lemmas}

\begin{proof}[Proof of Lemma \ref{lemm:B_mu4_relation}]

\begin{enumerate}[label=(\arabic*)]
\item
Because random vector $\bY$ is a permutation of $\bZ$, from Assumption \ref{assu:distribution} we know that each $Y_i$ is sub-Gaussian with parameter $\sqrt{3/8} B$.
By independence of $Y_i$, for any unit vector $\vb$ and all $\lambda \in \RR$ we have
$$
\EE \exp\left(
\lambda \vb^\top \bY
\right)
=
\prod_{i = 1}^d \EE\exp\left(
\lambda v_i Y_i
\right)
\le
\prod_{i = 1}^d \exp\left(
\frac{\lambda^2 v_i^2}{2} \cdot \left(\sqrt{\frac{3}{8}} B\right)^2
\right)
=
\exp\left(
\frac{\lambda^2}{2} \cdot \left(\sqrt{\frac{3}{8}} B\right)^2
\right)
,
$$
which implies that $\vb^\top \bY$ is also sub-Gaussian with parameter $\sqrt{3/8} B$.
Theorem 2.6 in \citet{wainwright2019high} shows that any sub-Gaussian random variable $X$ with parameter $\sigma$ satisfies
$
\EE\exp\left(\frac{\lambda X^2}{2 \sigma^2}\right)
\le
\frac{1}{\sqrt{1 - \lambda}}
$
for all $\lambda \in [0, 1)$.
By choosing $\lambda = 3/4$, $\sigma = \sqrt{3 / 8} B$ and $X = \vb^\top \bY$, we have
\beq\label{eq:psi_2_norm}
\EE\exp\left(\frac{(\vb^\top \bY)^2}{B^2}\right)\le 2
\quad\text{i.e.}\quad
\|\vb^\top \bY\|_{\psi_2}\le B
.
\eeq
We refer the readers to  \S\ref{sec:psi_alpha} in Appendix for more details on Orlicz $\psi_2$-norm.

\item
Applying Markov's inequality to \eqref{eq:psi_2_norm} with $\vb = \eb_i$ gives
$$
\PP(Y_i^2 \ge \lambda)
=
\PP\left(
\exp\left(\frac{Y_i^2}{B^2}\right)
\ge
\exp\left(\frac{\lambda}{B^2}\right)
\right)
\le
\exp\left(- \frac{\lambda}{B^2}\right)
\EE\exp\left(\frac{Y_i^2}{B^2}\right)
\le
2 \exp\left(- \frac{\lambda}{B^2}\right)
.
$$
Hence we have
$$
\mu_4 = \EE Y_i^4
=
\int_0^\infty \PP(Y_i^2 \ge \lambda) \cdot 2\lambda \ud \lambda
\le
4 \int_0^\infty \lambda \exp\left(- \frac{\lambda}{B^2}\right) \ud \lambda
=
4 B^4
,
$$
and hence
$
\dfrac{B^4}{|\mu_4-3|}
\ge
\dfrac{\mu_4}{4|\mu_4 - 3|}
\ge
\dfrac18
$
holds due to $\mu_4 = \EE Y_i^4 \ge (\EE Y_i^2)^2 = 1$.
\end{enumerate}
\end{proof}

\pb
\begin{proof}[Proof of Lemma \ref{lemm:Uinfquan}]
Let $\eta_S \equiv \eta\cdot \sign(\mu_4-3)$ for notational simplicity.
We recall definitions of iteration $U_k^{(t)}$ in \eqref{eq:Uk} and random variable $Q_{U, k}^{(t)}$ in \eqref{Uinfquan}.
Using update formula \eqref{eq:rescaled_update}, we have
$$\begin{aligned}
&
U_k^{(t)} - U_k^{(t - 1)}
= 
\frac{v_k + \eta_S \cdot (\vb^\top \bY)^3 Y_k}{v_1 + \eta_S \cdot (\vb^\top \bY)^3 Y_1}
-
\frac{v_k}{v_1}
= 
\frac{
\left(v_k + \eta_S \cdot (\vb^\top \bY)^3 Y_k\right) v_1
-
\left(v_1 + \eta_S \cdot (\vb^\top \bY)^3 Y_1 \right) v_k
}{
\left(v_1 + \eta_S \cdot (\vb^\top \bY)^3 Y_1 \right) v_1
}
\\&= 
\eta_S
\cdot \left(1 + \eta_S \cdot (\vb^\top \bY)^3 \frac{Y_1}{v_1} \right)^{-1}
\cdot (\vb^\top \bY)^3 v_1^{-2} \left( v_1 Y_k - v_k Y_1\right)
.
\end{aligned}$$
Along with \eqref{Uinfquan}, we obtain
$$
Q_{U, k}^{(t)}
=
\eta_S
\cdot \left[ \left(1 + \eta_S \cdot  (\vb^\top \bY)^3 \frac{Y_1}{v_1} \right)^{-1} - 1 \right]
\cdot (\vb^\top \bY)^3 v_1^{-2} \left( v_1 Y_k - v_k Y_1\right)
.
$$
For any $|x| \le \frac12$, summation of geometric series gives
$$
\left|(1+x)^{-1} - 1\right|
= |x| \left| \sum_{k=0}^\infty (-x)^k \right|
\le 2|x|
.
$$
On the event $\cH_{k; \ref{lemm:Uinfquan}, L}^{(t)}$ defined earlier in \eqref{eq:recursive_event}, since $\left| \eta_S \cdot (\vb^\top \bY)^3 Y_1 / v_1 \right| \le B^4 d^{1 / 2} \eta \le 1 / 2$, we have
$$\begin{aligned}
&
|Q_{U, k}^{(t)}|
\le
\eta\cdot \left|\left(1 + \eta_S \cdot  (\vb^\top \bY)^3 \frac{Y_1}{v_1} \right)^{-1} - 1\right|
\cdot |\vb^\top \bY|^3 v_1^{-2} | v_1 Y_k - v_k Y_1|
\\&\le
\eta\cdot 2 \left|\eta_S \cdot (\vb^\top \bY)^3 \frac{Y_1}{v_1}\right|
\cdot |\vb^\top \bY|^3 v_1^{-2} | v_1 Y_k - v_k Y_1|
\\&=
2\eta^2 v_1^{-2} |Y_1| |\vb^\top \bY|^6 \left|Y_k - \frac{v_k}{v_1} Y_1\right|
\le
4 \rB^8 \eta^2 v_1^{-2}
.
\end{aligned}$$
\end{proof}

\pb
\begin{proof}[Proof of Lemma \ref{lemm:infiU}]
Under Assumption \ref{assu:distribution}, for all $k \in [d]$ we have
\beq\label{eq:vY3Y}
\EE\left[ \left.
(\vb^\top \bY)^3 Y_k
\right| \cF_{t - 1} \right]
=
\mu_4 v_k^3 + 3 v_k \sum_{i = 1, i \ne k}^d v_i^2
=
(\mu_4 - 3) v_k^3 + 3 v_k
.
\eeq
Recall that $U_k^{(t - 1)} = v_k / v_1$, then
$$\begin{aligned}
&\quad
\EE\left[ \left.
\sign(\mu_4 - 3) \eta
\cdot
(\vb^\top \bY)^3 v_1^{-2} (v_1 Y_k - v_k Y_1)
\right| \cF_{t - 1} \right]
\\&=
\sign(\mu_4 - 3) \eta
\cdot
v_1^{-2} \left(
(\mu_4 - 3) v_1 v_k^3 + 3 v_1 v_k
-
(\mu_4 - 3) v_k v_1^3 - 3 v_k v_1
\right)
\\&=
- \eta |\mu_4 - 3| \cdot (v_1^2 - v_k^2) U_k^{(t - 1)}
.
\end{aligned}$$
\end{proof}

\pb

\begin{proof}[Proof of Lemma \ref{lemm:representation}]
From definitions \eqref{Uinfquan} and \eqref{ekt}, by applying \eqref{UinfE} in Lemma \ref{lemm:infiU},
\beq\label{Ugeom}
U_k^{(s)} - U_k^{(s - 1)}
=
 - \eta |\mu_4 - 3|\cdot \left( (v_1^{(s - 1)})^2 - (v_k^{(s - 1)})^2 \right) U_k^{(s - 1)}
+ Q_{U, k}^{(s)} + e_k^{(s)}
.
\eeq
Iteratively applying \eqref{Ugeom} for $s = 1, \dots, t$ gives \eqref{Urep}.
\end{proof}

\pb

\begin{proof}[Proof of Lemma \ref{lemm:1order_concentration}]
Under Assumption \ref{assu:distribution}, we apply \eqref{eq:triangle_inequality} along with Lemmas \ref{lemm:B_mu4_relation} and \ref{lemm:Orlicz_norm_property} to obtain
$$\begin{aligned}
&
\left\|\eta \cdot v_1^{-2} (\vb^\top \bY)^3 (v_1 Y_k - v_k Y_1)\right\|_{\psi_{1/2}}
\le
\eta v_1^{-2}
\cdot \|(\vb^\top \bY)^2\|_{\psi_1}
\cdot \|(\vb^\top \bY)(v_1 Y_k - v_k Y_1)\|_{\psi_1}
\\&\le
\eta v_1^{-2}
\cdot \|\vb^\top \bY\|_{\psi_2}^3
\cdot \|v_1 Y_k - v_k Y_1\|_{\psi_2}
\\&\le
\eta |v_1|^{-1}
\cdot \|\vb^\top \bY\|_{\psi_2}^3
\cdot (\|Y_k\|_{\psi_2} + |v_k / v_1| \|Y_1\|_{\psi_2})
\le
B^4 \eta\cdot |v_1|^{-1} (1 + |v_k / v_1|)
.
\end{aligned}$$
Recall the definition of martingale difference sequence $\{e_k^{(t)}\}_{t \ge 1}$ in \eqref{ekt}.
By applying Lemma \ref{lemm:psi1/2}, we have
\beq\label{eq:e_psi1/2_shared_bound}
\begin{aligned}
&
\left\|e_k^{(t)}\right\|_{\psi_{1/2}}
=
\left\|
\eta \cdot v_1^{-2} (\vb^\top \bY)^3 (v_1 Y_k - v_k Y_1)
-
\EE \left[
\eta \cdot v_1^{-2} (\vb^\top \bY)^3 (v_1 Y_k - v_k Y_1)
\right]
\right\|_{\psi_{1/2}}
\\&\le
C_{\ref{lemm:psi1/2}, L}' 
\left\|\eta \cdot v_1^{-2} (\vb^\top \bY)^3 (v_1 Y_k - v_k Y_1)\right\|_{\psi_{1/2}}
\le
C_{\ref{lemm:psi1/2}, L}' B^4 \eta
\cdot
|v_1|^{-1} (1 + |v_k / v_1|)
.
\end{aligned}\eeq
Because initialization $\vb^{(0)} \in \cD_{\text{warm}}$, on the event $(t \le \cT_x)$ for $\cT_x$ earlier defined in \eqref{eq:cTx}, we have $|v_1|^{-1} \le \sqrt{3} / \sqrt{2}$, $|v_k / v_1| \le 1 / \sqrt{2}$, and then
$
\left\|e_k^{(t)}\right\|_{\psi_{1/2}}
\le
3 C_{\ref{lemm:psi1/2}, L}' B^4 \eta
.
$
Since $1_{(t \le \cT_x)} \in \cF_{t - 1}$,we know that $\{e_k^{(t)} 1_{(t \le \cT_x)}\}$ forms a martingale difference sequence with respect to $\cF_{t - 1}$.
Additionally, because $\|e_k^{(t)} 1_{(t \le \cT_x)}\|_{\psi_{1/2}} \le \|e_k^{(t)}\|_{\psi_{1/2}}$, we have
$$
\left\|e_k^{(t)} 1_{(t \le \cT_x)}\right\|_{\psi_{1/2}}
\le
3 C_{\ref{lemm:psi1/2}, L}' B^4 \eta
. 
$$
With the bound given above, we apply Theorem \ref{theo:concentration} with $\alpha = 1 / 2$ and obtain for all $\delta \in (0, e^{-1}]$,
$$\begin{aligned}
\lefteqn{
\PP\left(
\max_{1 \le t \le T_{\eta, \tau}^* \wedge \cT_x} \left|\sum_{s = 1}^t e_k^{(s)}\right|
\ge
C_{\ref{lemm:1order_concentration}, L} B^4 \eta (T_{\eta, \tau}^*)^{1 / 2} \log^{5 / 2}\delta^{-1}
\right)
}
\\&=
\PP\left(
\max_{1 \le t \le T_{\eta, \tau}^*} \left|\sum_{s = 1}^t e_k^{(s)} 1_{(s \le \cT_x)}\right|
\ge
C_{\ref{lemm:1order_concentration}, L} B^4 \eta (T_{\eta, \tau}^*)^{1 / 2} \log^{5 / 2}\delta^{-1}
\right)
\\&\le
2 \left[
3 + 6^4 \frac{64 \cdot T_{\eta, \tau}^* \cdot 9 C_{\ref{lemm:psi1/2}, L}'^2 B^8 \eta^2}{C_{\ref{lemm:1order_concentration}, L}^2 B^8 \eta^2 T_{\eta, \tau}^* \log^5\delta^{-1}}
\right]
\exp\left\{ - \left(
\frac{C_{\ref{lemm:1order_concentration}, L}^2 B^8 \eta^2 T_{\eta, \tau}^* \log^5\delta^{-1}}{32 \cdot T_{\eta, \tau}^* \cdot 9 C_{\ref{lemm:psi1/2}, L}'^2 B^8 \eta^2}
\right)^{\frac15} \right\}
\\&=
\left(
6 + \frac{5184}{\log^5\delta^{-1}}
\right)\delta
,
\end{aligned}$$
where constant $C_{\ref{lemm:1order_concentration}, L} = 12\sqrt{2} C_{\ref{lemm:psi1/2}, L}'$.
\end{proof}

\pb
\begin{proof}[Proof of Lemma \ref{lemm:tail_probability}]
Recall definition of $\rB_*$ in \eqref{Mdef}, $\cT_{\rB_*, k}$ in \eqref{cTM} and $T_{\eta, \tau}^*$ in \eqref{eq:Tetatau*}. Using Markov inequality and Lemma \ref{lemm:B_mu4_relation}, we have
$$
\PP(|\vb^\top \bY| > \rB_*)
=
\PP\left(
\frac{|\vb^\top \bY|^2}{B^2}
>
\frac{\rB_*^2}{B^2}
\right)
\le
\exp\left(
-\frac{\rB_*^2}{B^2}
\right)
\cdot
\EE \exp\left(
\frac{|\vb^\top \bY|^2}{B^2}
\right)
\le
\frac{2\delta}{T_{\eta, 1}^*}
.
$$
Similarly we also have
$$
\PP(|Y_1| > \rB_*)
\le
\frac{2\delta}{T_{\eta, 1}^*}
,\qquad
\PP(|Y_k| > \rB_*)
\le
\frac{2\delta}{T_{\eta, 1}^*}
.
$$
Taking union bound,
$$\begin{aligned}
&
\PP\left( \cT_{\rB_*, k} \le T_{\eta, \tau}^* \right)
\le
\sum_{t = 1}^{T_{\eta, \tau}^*} \left(
\PP(|\vb^{(t - 1) \top} \bY^{(t)}| > \rB_*)
+\PP(|Y_1^{(t)}| > \rB_*)
+\PP(|Y_k^{(t)}| > \rB_*)
\right)
\\&\le
3T_{\eta, \tau}^*
\cdot\frac{2\delta}{T_{\eta, 1}^*}
\le
6 (\tau + 1) \delta
,
\end{aligned}$$
where we use the elementary inequality $\lceil \tau x \rceil \le (\tau + 1) \lceil x \rceil$ for all $\tau, x \ge 0$ to obtain $T_{\eta, \tau}^* / T_{\eta, 1}^* \le \tau + 1$.
\end{proof}

\pb\section{Secondary Lemmas in Uniform Initialization Analysis}\label{sec:uniform_aux}

For notational simplicity, we denote $\vb \equiv \vb^{(t - 1)}$, $\bY \equiv \bY^{(t)}$.
Recall that we bounded the Orlicz $\psi_2$-norm of each $Y_i$ and $\vb^\top \bY$ by $B$ using Lemma \ref{lemm:B_mu4_relation} in Appendix \ref{sec:warm_aux} and introduced truncation barrier $\rB_*$ in \eqref{Mdef} under warm initialization condition, based on rescaled time $T_{\eta, \tau}^*$.
Under uniform initialization condition, we consider a different rescaled time $T_{\eta, \tau}^o$ defined in \eqref{Tetatau}.
Accordingly, we introduce a slightly larger truncation barrier based on $T_{\eta, \tau}^o$
\beq\label{Mdef2}
\rB_o
\equiv
B \log^{1 / 2}(T_{\eta, 1}^o \delta^{-1})
,
\eeq
where $\delta \in (0, e^{-1}]$ is some fixed positive.
For each coordinate $k\in [2,d]$ define the first time the norm of a data observation exceeds the truncation barrier $\rB_o$ as
\beq\label{cTM2}
\cT_{\rB_o, k}
=
\inf\left\{ t\ge 1: 
|\vb^{(t-1)\top} \bY^{(t)}| > \rB_o
\text{ or }
|Y_1^{(t)}|  > \rB_o
\text{ or }
|Y_k^{(t)}|  > \rB_o
\right\}
.
\eeq

\pb\subsection{Proof of Lemma \ref{lemm:expo}}

Proof of Lemma \ref{lemm:expo} shares Lemma \ref{lemm:Uinfquan}, \ref{lemm:infiU} and \ref{lemm:representation} with proof of Lemma \ref{lemm:fast_expo}.
With the new truncation barrier $\rB_o$ given in \eqref{Mdef2} and the new rescaled time $T_{\eta, \tau}^o$ earlier defined in \eqref{Tetatau}, we plug in $\rB = \rB_o$ in Lemma \ref{lemm:Uinfquan} and introduce a new concentration lemma and a new tail probability lemma.

\begin{lemma}\label{lemm:1order_concentration2}
For any fixed coordinate $k \in [2, d]$, let Assumption \ref{assu:distribution} hold and initialization $\vb^{(0)} \in \cD_{\text{mid}, k}$. Let $\delta, \tau$ be any fixed positives. Then there exists an event $\cH_{k; \ref{lemm:1order_concentration2}, L}$ satisfying
$$
\PP(\cH_{k; \ref{lemm:1order_concentration2}, L})
\ge
1 - \left(
6 + \frac{5184}{\log^5\delta^{-1}}
\right)\delta
,
$$
such that on event $\cH_{k; \ref{lemm:1order_concentration2}, L}$ the following concentration result holds
$$
\max_{1 \le t \le T_{\eta, \tau}^o \wedge \cT_{w, k}} \left|
\sum_{s = 1}^t e_k^{(s)}
\right|
\le
C_{\ref{lemm:1order_concentration2}, L} \log^{5 / 2}\delta^{-1}
\cdot
B^4
\cdot
d^{1 / 2} \eta (T_{\eta, \tau}^o)^{1 / 2}
,
$$
where $C_{\ref{lemm:1order_concentration2}, L}$ is a positive, absolute constant.
\end{lemma}

\begin{lemma}\label{lemm:tail_probability2}
Let $\tau$ be any fixed positive. For each coordinate $k \in [2, d]$, we have
\beq\label{tailcTM2}
\PP\left(
\cT_{\rB_o, k}
\le
T_{\eta, \tau}^o
\right)
\le
6 (\tau + 1) \delta
.
\eeq
\end{lemma}

\pb

With the above secondary lemmas at hand, we are now ready for the proof of Lemma \ref{lemm:expo}.

\begin{proof}[Proof of Lemma \ref{lemm:expo}]
Recall the definition of stopping time $\cT_{w, k}$ in \eqref{cTw} and $\cT_{\rB_o, k}$ in \eqref{cTM2}.
Since scaling condition \eqref{eq:uniform_scaling} implies stepsize $\eta \le 1 / (2 \rB_o^4 d^{1 / 2})$, we can apply Lemma \ref{lemm:Uinfquan} with $\rB = \rB_o$.
On the event $(t \le T_{\eta, \tau}^o \wedge \cT_{w, k}) \cap (\cT_{\rB_o, k} > T_{\eta, \tau}^o)$, we have $v_1^{-2} \le d$ and hence
$$
\left|
Q_{U, k}^{(t)}
\right|
=
4 \rB_o^8 \eta^2 v_1^{-2}
\le
4 B^8 d \eta^2 \log^4(T_{\eta, 1}^o \delta^{-1})
.
$$
We define event $\cH_{k; \ref{lemm:expo}, L} := (\cT_{\rB_o, k} > T_{\eta, \tau}^o) \cap \cH_{k; \ref{lemm:1order_concentration2}, L}$, then on the event $\cH_{k; \ref{lemm:expo}, L}$, by applying Lemma \ref{lemm:representation} and \ref{lemm:1order_concentration2}, for all $t \le T_{\eta, \tau}^o \wedge \cT_{w, k}$ we have
\beq\label{prf_mid_prop_1}
\begin{aligned}
\lefteqn{
\left|
U_k^{(t)} - U_k^{(0)} + \eta |\mu_4 - 3| \sum_{s=0}^{t-1} \left( (v_1^{(s)})^2 - (v_k^{(s)})^2 \right) U_k^{(s)}
\right|
}
\\&\le
T_{\eta, \tau}^o \cdot 4 B^8 d \eta^2 \log^4(T_{\eta, 1}^o \delta^{-1})
+
C_{\ref{lemm:1order_concentration}, L} \log^{5 / 2}\delta^{-1}
\cdot
B^4
\cdot
d^{1 / 2} \eta (T_{\eta, \tau}^o)^{1 / 2}
.
\end{aligned}\eeq
Scaling condition \eqref{eq:uniform_scaling} in Lemma \ref{lemm:convergence_result} implies that 
\beq\label{eq:uniform_scaling_mid1}
B^4 d^{1 / 2} \eta (T_{\eta, \tau}^o)^{1 / 2} \log^4(T_{\eta, 1}^o \delta^{-1})
\le
\log^{5 / 2}\delta^{-1}
.
\eeq
Together with \eqref{prf_mid_prop_1}, on the event $\cH_{k; \ref{lemm:expo}, L}$ we have
$$
\left|
U_k^{(t)} - U_k^{(0)} + \eta |\mu_4 - 3| \sum_{s=0}^{t-1} \left( (v_1^{(s)})^2 - (v_k^{(s)})^2 \right) U_k^{(s)}
\right|
\le
C_{\ref{lemm:expo}, L} \log^{5 / 2}\delta^{-1}
\cdot B^4
\cdot d^{1 / 2} \eta (T_{\eta, \tau}^o)^{1 / 2}
,
$$
where positive constant $C_{\ref{lemm:expo}, L} = C_{\ref{lemm:1order_concentration}, L} + 4$.
Scaling condition \eqref{eq:uniform_scaling} implies $\eta |\mu_4 - 3| ((v_1^{(s)})^2 - (v_k^{(s)})^2) \in [0, 1)$ for all $s < T_{\eta, \tau}^o \wedge \cT_{w, k}$.
From the reversed Gronwall Lemma \ref{lemm:grw_type}, on the event $\cH_{k; \ref{lemm:expo}, L}$ the following holds for all $t \le T_{\eta, \tau}^o \wedge \cT_{w, k}$
$$
\left|
U_k^{(t)}  - U_k^{(0)} \prod_{s = 0}^{t - 1} \left[  1 - \eta |\mu_4 - 3|  \left( (v_1^{(s)})^2 - (v_k^{(s)})^2 \right)  \right]
\right|
\le
2 C_{\ref{lemm:expo}, L} \log^{5 / 2}\delta^{-1}
\cdot B^4
\cdot d^{1 / 2} \eta (T_{\eta, \tau}^o)^{1 / 2}
.
$$
To obtain a lower bound on the probability of event $\cH_{k; \ref{lemm:expo}, L}$, we combine Lemmas \ref{lemm:1order_concentration2}, \ref{lemm:tail_probability2} and take union bound,
$$
\PP(\cH_{k; \ref{lemm:expo}, L})
\ge
1 - \left(6\tau + 12 + \frac{5184}{\log^5\delta^{-1}}\right)\delta
$$
\end{proof}

\pb\subsection{Proof of Lemma \ref{lemm:initW}}\label{sec:proof,lemm:initW}

Lemma \ref{lemm:initW} provides a quantitative characterization of the uniform initialization on $\cD_1$.
As a probabilistic fact, the uniform distribution $\vb^{(0)}$ in $\cD_1$ is equal in distribution to $\|\boldsymbol{\chi}\|^{-1} \boldsymbol{\chi}$ with $\boldsymbol{\chi} \sim N(0, \Ib)$.
In addition, we have
\beq\label{init_spacing}
\min_{2 \le k \le d} W_k^{(0)}
\ge
\min_{2 \le k \le d} \log\left(
\frac{(v_1^{(0)})^2}{(v_k^{(0)})^2}
\right)
=
\log(v_1^{(0)})^2
-
\max_{2 \le k \le d} \log(v_k^{(0)})^2
,
\eeq
due to definition of $W_k^{(t)}$ in \eqref{eq:W_def} and the elementary inequality $x - 1 \ge \log x$ for all $x > 0$.
Intuitively, this means that $\min_{2 \le k \le d} W_k^{(0)}$ is lower bounded by the spacing between the largest and second largest order statistics of $d$ i.i.d. logarithmic chi-squared distributions.

We provide an elementary probabilistic Lemma \ref{lemm:pure_prob} to show the spacing on the right hand of \eqref{init_spacing} $W_k^{(0)} = \Omega(\log^{-1} d)$  with high probability.%
\footnote{In retrospect, a similar lemma characterizing the lower bound of spacings was achieved using a different method in \citet{bai2018subgradient}.}

\begin{lemma}\label{lemm:pure_prob}
Let $\chi_i^2$ ($1 \le i \le n$) be squares of i.i.d.~standard normal variables, and denote their order statistics as $\chi_{(1)}^2 \le \dots \le \chi_{(n)}^2$.
Then for any $\epsilon \in (0,1/3)$, when $n \ge 2\sqrt{2\pi e} \log\epsilon^{-1} + 1$, we have with probability at least $1 - 3\epsilon$
\beq\label{eq:pure_prob}
\log\chi^2_{(n)}
-
\log \chi^2_{(n - 1)}
\ge
\frac{\epsilon}{8 \log\epsilon^{-1} \log n}
.\eeq
\end{lemma}

We can apply Lemma \ref{lemm:pure_prob} and prove Lemma \ref{lemm:initW}.

\begin{proof}[Proof of Lemma \ref{lemm:initW}]
Following assumptions in Theorem \ref{theo:finite_sample}, we know that $\vb^{(0)}$ has the same distribution as $\|\boldsymbol{\chi}\|^{-1} \boldsymbol{\chi}$ where $\boldsymbol{\chi} \sim N(0, \Ib)$.
Under scaling condition \eqref{eq:uniform_scaling} in Lemma \ref{lemm:convergence_result}, Lemma \ref{lemm:initW} is straightforward if we apply Lemma \ref{lemm:pure_prob} with $n = d$ and use \eqref{init_spacing}.
\end{proof}

\pb\subsection{Proof of Lemma \ref{lemm:hitt}}

Let $k \in [2, d]$ be any fixed coordinate. For each $t \ge 1$, we define random variable
\beq\label{Winfquan}
Q_{W, k}^{(t)}
=
W_k^{(t)} - W_k^{(t - 1)}
+
\sign(\mu_4 - 3) \eta
\cdot
2 (\vb^\top \bY)^3 v_1 v_k^{-3} (v_1 Y_k - v_k Y_1)
.
\eeq

\begin{lemma}\label{lemm:Winfquan}
For each coordinate $k \in [2, d]$ and any $t \ge 1$, on the event
\beq\label{eq:recursive_event2}
\cH_{k; \ref{lemm:Winfquan}, L}^{(t)}
\equiv
\left(
|\vb^\top \bY| \le \rB_o
,~
|Y_1| \le \rB_o
,~
|Y_k| \le \rB_o
,~
v_1^2 < 3 v_k^2
,~
v_1^2 \ge \max_{i\in [2,d]} v_i^2
\right)
,
\eeq
under condition
\beq\label{eq:cold_weak_scaling}
12 \rB^8 d \eta^2 \log^4(T_{\eta, 1}^o \delta^{-1})
\le
1
,
\eeq
we have
$$
|Q_{W, k}^{(t)}|
\le
C_{\ref{lemm:Winfquan}, L} \rB_o^8 d \eta^2
,
$$
where $C_{\ref{lemm:Winfquan}, L}$ is a positive, absolute constant.
\end{lemma}

\begin{lemma}\label{lemm:infiW}
Let Assumption \ref{assu:distribution} hold.
For each coordinate $k \in [2, d]$ and any $t \ge 1$, we have
\beq\label{WinfE}
\EE\left[ \left.
\sign(\mu_4 - 3) \eta
\cdot
2 (\vb^\top \bY)^3 v_1 v_k^{-3} (v_1 Y_k - v_k Y_1)
\right| \cF_{t - 1} \right]
=
- 2 \eta |\mu_4 - 3| v_1^2 W_k
.
\eeq
\end{lemma}

For each $k \in [2, d]$ and $t \ge 1$, at the $t$-th iterate we let 
\beq\label{Kfn}
f_k^{(t)}
=
-
\sign(\mu_4 - 3) \eta
\cdot
2 (\vb^\top \bY)^3 v_1 v_k^{-3} (v_1 Y_k - v_k Y_1)
-
2 \eta |\mu_4 - 3| v_1^2 W_k
.
\eeq
which, indexed by $t$, forms a sequence of martingale differences with respect to $\cF_{t - 1}$.
Combining \eqref{Winfquan} and \eqref{WinfE} together we have
\beq\label{Wgeom}
W_k^{(t)}
=
\left(
1 + 2 \eta |\mu_4 - 3| (v_1^{(t - 1)})^2
\right) W_k^{(t - 1)}
+
Q_{W, k}^{(t)}
+
f_k^{(t)}
.
\eeq
By letting
\beq\label{eq:Pot}
P_t^o
=
\prod_{s = 0}^{t-1} \left(
1 + 2 \eta |\mu_4 - 3| (v_1^{(s)})^2
\right)^{-1}
\in
\cF_{t - 1}
,
\eeq
we conclude the following lemma.

\begin{lemma}\label{lemm:Ugeom}
For each coordinate $k \in [2, d]$ and any $t \ge 1$, the iteration $W_k^{(t)}$ can be represented linearly as
\beq\label{Wrep}
P_t^o W_k^{(t)}
=
W_k^{(0)} + \sum_{s = 1}^t P_s^o Q_{W, k}^{(s)} + \sum_{s = 1}^t P_s^o f_k^{(s)}
.
\eeq
\end{lemma}

\begin{lemma}\label{lemm:sum_f_concentration}
Let Assumption \ref{assu:distribution} hold and initialization $\vb^{(0)} \in \cD_{\text{cold}} \cap \cD_{\text{mid}, k}^c$. Let $\delta$ be a fixed positive. For each coordinate $k \in [2, d]$, there exists an event $\cH_{k; \ref{lemm:sum_f_concentration}, L}$ satisfying
$$
\PP(\cH_{k; \ref{lemm:sum_f_concentration}, L})
\ge
1 - \left(
6 + \frac{5184}{\log^5\delta^{-1}}
\right) \delta
$$
such that on event $\cH_{k; \ref{lemm:sum_f_concentration}, L}$ the following concentration result holds
$$
\max_{1 \le t \le T_{\eta, 0.5}^o \wedge \cT_{c, k} \wedge \cT_1} \left|
\sum_{s = 1}^t P_s^o f_k^{(s)}
\right|
\le
C_{\ref{lemm:sum_f_concentration}, L} \log^{5 / 2}\delta^{-1}
\cdot
\frac{B^4}{|\mu_4 - 3|^{1 / 2}}
\cdot
d \eta^{1 / 2}
$$
where $C_{\ref{lemm:sum_f_concentration}, L}$ is a positive, absolute constant.
\end{lemma}

\pb

Along with the tail probability Lemma \ref{lemm:tail_probability2}, we are ready to present the proof of Lemma \ref{lemm:hitt}.

\begin{proof}[Proof of Lemma \ref{lemm:hitt}]
Recall the definition of truncation barrier $\rB_o$ in \eqref{Mdef2}, stopping times $\cT_{\rB_o, k}$ in \eqref{cTM2}, $\cT_{c, k}$ in \eqref{eq:Tck} and $\cT_1$ in \eqref{eq:T1}. 
We notice that \eqref{eq:cold_weak_scaling} holds under scaling condition \eqref{eq:uniform_scaling}, and event $(t \le T_{\eta, 0.5}^o \wedge \cT_{c, k} \wedge \cT_1) \cap (\cT_{\rB_o, k} > T_{\eta, 0.5}^o) \subseteq \cH_{k; \ref{lemm:Winfquan}, L}^{(t)}$. 
We define event $\cH_{k; \ref{lemm:hitt}, L} \equiv (\cT_{\rB_o, k} > T_{\eta, 0.5}^o) \cap \cH_{k; \ref{lemm:sum_f_concentration}, L}$, then on event $\cH_{k; \ref{lemm:hitt}, L}$, by applying Lemma \ref{lemm:Winfquan} for all $t \le T_{\eta, 0.5}^o \wedge \cT_{c, k} \wedge \cT_1$ we obtain
$$
v_1^2
\ge
\frac{1}{d}
,\qquad
|Q_{W, k}^{(t)}|
\le
C_{\ref{lemm:Winfquan}, L} B^8 d \eta^2 \log^4(T_{\eta, 1}^o \delta^{-1})
.
$$
Scaling condition \eqref{eq:uniform_scaling} also guarantees
\beq\label{eq:cold_scaling_sub2}
\frac{2 |\mu_4 - 3|}{d} \eta
<
1
,\qquad
\frac{B^4}{|\mu_4 - 3|^{1 / 2}}
\cdot
d \eta ^{1 / 2} \log^4(T_{\eta, 1}^o \delta^{-1})
\le
\log^{5 / 2}\delta^{-1}
,
\eeq
and hence by summation of geometric series, on event $\cH_{k; \ref{lemm:hitt}, L}$ we have for all $t \le T_{\eta, 0.5}^o \wedge \cT_{c, k} \wedge \cT_1$
$$\begin{aligned}
&
\left|\sum_{s = 1}^t P_s^o Q_{W, k}^{(s)}\right|
\\&\le
\sum_{s = 1}^t \left(1 + \frac{2 |\mu_4 - 3|}{d} \eta\right)^{-s} \cdot |Q_{W, k}^{(s)}|
\le
\frac{\left(1 + \frac{2 |\mu_4 - 3|}{d} \eta\right)^{-1}}{1 - \left(1 + \frac{2 |\mu_4 - 3|}{d} \eta\right)^{-1}}
\cdot
C_{\ref{lemm:Winfquan}, L} B^8 d \eta^2 \log^4(T_{\eta, 1}^o \delta^{-1})
\\&=
\frac{C_{\ref{lemm:Winfquan}, L}}{2}
\cdot \frac{B^8}{|\mu_4 - 3|}
\cdot d^2 \eta \log^4(T_{\eta, 1}^o \delta^{-1})
\le
\frac{C_{\ref{lemm:Winfquan}, L}}{2} \log^{5 / 2}\delta^{-1}
\cdot \frac{B^4}{|\mu_4 - 3|^{1 / 2}}
\cdot d \eta^{1 / 2}
.
\end{aligned}$$
Combining with Lemmas \ref{lemm:Ugeom} and \ref{lemm:sum_f_concentration}, on event $\cH_{k; \ref{lemm:hitt}, L}$ we know that for all $t \le T_{\eta, 0.5}^o \wedge \cT_{c, k} \wedge \cT_1$ the following holds for constant $C_{\ref{lemm:hitt}, L} \equiv C_{\ref{lemm:sum_f_concentration}, L} + C_{\ref{lemm:Winfquan}, L} / 2$
$$
\left|  
W_k^{(t)}  \prod_{s = 0}^{t - 1} \left( 1 + \eta |\mu_4 - 3| (v_1^{(s)})^2 \right)^{-1}  - W_k^{(0)}
\right|
\le
C_{\ref{lemm:hitt}, L} \log^{5 / 2}\delta^{-1}
\cdot \frac{B^4}{|\mu_4 - 3|^{1 / 2}}
\cdot d \eta^{1 / 2}
.
$$
We verify the remaining claims in Lemma \ref{lemm:hitt} by applying Lemma \ref{lemm:tail_probability2} with $\tau = 0.5$, Lemma \ref{lemm:sum_f_concentration} and taking union bound
$$
\PP(\cH_{k; \ref{lemm:hitt}, L})
\ge
1 - \PP(\cT_{\rB_o, k} > T_{\eta, 0.5}^o) - \PP(\cH_{k; \ref{lemm:sum_f_concentration}, L}^c)
\ge
1 - \left(15 + \frac{5184}{\log^5\delta^{-1}}\right) \delta
.
$$
.
\end{proof}

\pb\subsection{Proof of Secondary Lemmas}

\begin{proof}[Proof of Lemma \ref{lemm:1order_concentration2}]
$\{e_k^{(t)}\}_{t \ge 1}$ earlier defined in \eqref{ekt} forms a martingale difference sequence with respect to $\cF_{t - 1}$.
Recall \eqref{eq:e_psi1/2_shared_bound} in proof of Lemma \ref{lemm:1order_concentration}, we have derived the following Orlicz $\psi_{1/2}$-norm bound under Assumption \ref{assu:distribution}
$$
\left\|e_k^{(t)}\right\|_{\psi_{1/2}}
\le
C_{\ref{lemm:psi1/2}, L}' \eta B^4
\cdot
|v_1|^{-1} (1 + |v_k / v_1|)
.
$$
Because initialization $\vb^{(0)} \in \cD_{\text{mid}, k}$, on the event $(t \le \cT_{w, k})$, where $\cT_{w, k}$ is earlier defined in \eqref{cTw}, we have $|v_1|^{-1} \le d^{1 / 2}$, $|v_k / v_1| \le 1 / \sqrt{2}$, and then
$$
\left\|e_k^{(t)}\right\|_{\psi_{1/2}}
\le
2 C_{\ref{lemm:psi1/2}, L}' B^4 d^{1 / 2} \eta
.
$$
Because $1_{(t \le \cT_{w, k})} \in \cF_{t - 1}$, we know that $\{e_k^{(t)} 1_{(t \le \cT_{w, k})}\}$ forms a martingale difference sequence with respect to $\cF_{t - 1}$, and $\|e_k^{(t)} 1_{(t \le \cT_{w, k})}\|_{\psi_{1/2}} \le \|e_k^{(t)}\|_{\psi_{1/2}} \le 2 C_{\ref{lemm:psi1/2}, L}' B^4 d^{1 / 2} \eta$.
Hence we can apply Theorem \ref{theo:concentration} with $\alpha = 1 / 2$ and obtain for all $\delta \in (0, e^{-1}]$,
$$\begin{aligned}
&\lefteqn{
\PP\left(
\max_{1 \le t \le T_{\eta, \tau}^o \wedge \cT_{w, k}} \left|\sum_{s = 1}^t e_k^{(s)}\right|
\ge
C_{\ref{lemm:1order_concentration2}, L} B^4 d^{1 / 2} \eta (T_{\eta, \tau}^o)^{1 / 2} \log^{5 / 2}\delta^{-1}
\right)
}\\&=
\PP\left(
\max_{1 \le t \le T_{\eta, \tau}^o} \left|\sum_{s = 1}^t e_k^{(s)} 1_{(s \le \cT_{w, k})}\right|
\ge
C_{\ref{lemm:1order_concentration2}, L} B^4 d^{1 / 2} \eta (T_{\eta, \tau}^o)^{1 / 2} \log^{5 / 2}\delta^{-1}
\right)
\\&\le
2 \left[
3 + 6^4 \frac{64 \cdot T_{\eta, \tau}^o \cdot 4 C_{\ref{lemm:psi1/2}, L}'^2 B^8 d \eta^2}{C_{\ref{lemm:1order_concentration2}, L}^2 B^8 d \eta^2 T_{\eta, \tau}^o \log^5\delta^{-1}}
\right]
\exp\left\{
- \left(
\frac{C_{\ref{lemm:1order_concentration2}, L}^2 B^8 d \eta^2 T_{\eta, \tau}^o \log^5\delta^{-1}}{32 \cdot T_{\eta, \tau}^o \cdot 4 C_{\ref{lemm:psi1/2}, L}'^2 B^8 d \eta^2}
\right)^{\frac15}
\right\}
\\&=
\left(6 + \frac{5184}{\log^5\delta^{-1}}\right)\delta
,
\end{aligned}$$
where constant $C_{\ref{lemm:1order_concentration2}, L} \equiv 8\sqrt{2} C_{\ref{lemm:psi1/2}, L}'$.
\end{proof}

\pb
\begin{proof}[Proof of Lemma \ref{lemm:tail_probability2}]
Applying Markov inequality and Lemma \ref{lemm:B_mu4_relation} gives
$$
\PP(|\vb^\top \bY| > \rB_o)
=
\PP\left(\frac{|\vb^\top \bY|^2}{B^2} > \frac{\rB_o^2}{B^2}\right)
\le
\exp\left(-\frac{\rB_o^2}{B^2}\right)
\cdot
\EE\exp\left(\frac{|\vb^\top \bY|^2}{B^2}\right)
\le
\frac{2\delta}{T_{\eta, 1}^o}
.
$$
With the same procedure we also obtain
$$
\PP(|Y_1^{(t)}| > \rB_o)
\le
\frac{2\delta}{T_{\eta, 1}^o}
,\qquad
\PP(|Y_k^{(t)}| > \rB_o)
\le
\frac{2\delta}{T_{\eta, 1}^o}
.
$$
Taking union bound,
$$\begin{aligned}
&
\PP\left( \cT_{\rB_o, k} \le T_{\eta, \tau}^o \right)
\le
\sum_{t = 1}^{T_{\eta, \tau}^o} \left(
\PP(|\vb^{(t - 1) \top} \bY^{(t)}| > \rB_o) + \PP(|Y_1^{(t)}| > \rB_o) + \PP(|Y_k^{(t)}| > \rB_o)
\right)
\\&\le
3 T_{\eta, \tau}^o \cdot \frac{2\delta}{T_{\eta, 1}^o}
\le
6(\tau + 1) \delta
,
\end{aligned}$$
due to $T_{\eta, \tau}^o / T_{\eta, 1}^o \le \tau + 1$, which comes from its definition in \eqref{Tetatau} and the elementary inequality $
\lceil \tau x \rceil \le (\tau + 1) \lceil x \rceil
$ for all $\tau, x \ge 0$.
\end{proof}

\pb
\begin{proof}[Proof of Lemma \ref{lemm:pure_prob}]
Let $F(x)$ be the cumulative distribution function of $\log(\chi^2)$, where $\chi^2$ denote the chi-squared distribution, then
\beq\label{eq:F_def}
F(x) = 2\Phi\left( e^{x/2} \right) - 1
,\qquad
F'(x) = \phi\left( e^{x/2} \right) e^{x/2}
.
\eeq
where $\Phi(x)$ and $\phi(x)$ are cumulative distribution function and probability density function of standard Gaussian random variables.
Let $(U_i, 1\le i\le n$) be i.i.d.~samples from $\mathrm{Uniform}(0, 1)$, whose largest and second largest order statistics $U_{(n)}, U_{(n - 1)}$ are denoted as $U, U_-$ for notational simplicity.
We also define $M \equiv F^{-1}(U)$, $M_- \equiv F^{-1}(U_-)$.
Since $F(x)$ is concave for $x \ge 0$, under condition
\beq\label{prob_cond0}
M_- \ge 0
,\quad\text{i.e.}\quad
U_- \ge 2\Phi(1) - 1
,
\eeq
by concavity we have
\beq\label{prob_main}
U - U_-
= F(M) - F(M_-) \le F'(M_-) (M - M_-)
.
\eeq
We denote the inverse functions of $F, \Phi$ by $F^{-1}$, $\Phi^{-1}$, then from \eqref{eq:F_def} we have
\beq\label{eq:F-1_Phi-1}
F^{-1}(y) = 2 \log \Phi^{-1} \left(\frac{y + 1}{2}\right)
.
\eeq

\begin{enumerate}[label=(\arabic*)]
\item
Taking the derivative for an inverse function gives
$$
[\Phi^{-1}]'(y)
= \frac{1}{\phi(\Phi^{-1}(y))}
= \sqrt{2\pi}\exp\left(\frac{\Phi^{-1}(y)^2}{2}\right)
.
$$
Applying chain rule of derivative to \eqref{eq:F-1_Phi-1}, we have
\beq\label{prob_chainrule}
[F^{-1}]'(y)
= \frac{[\Phi^{-1}]'\left(\frac{y + 1}{2}\right)}{\Phi^{-1}\left(\frac{y + 1}{2}\right)}
= \frac{\sqrt{2\pi}\exp\left(\Phi^{-1}\left(\frac{y + 1}{2}\right)^2 / 2\right)}{\Phi^{-1}\left(\frac{y + 1}{2}\right)}
.
\eeq

\item
The following bound on tail probability of standard gaussian random variable is folklore:
\beq\label{normal_tail}
\frac{x}{x^2 + 1} \cdot \frac{1}{\sqrt{2\pi}} \exp\left(- \frac{x^2}{2}\right)
\le
1 - \Phi(x)
\le
\frac{1}{x} \cdot \frac{1}{\sqrt{2\pi}} \exp\left(- \frac{x^2}{2}\right)
.
\eeq
We define
$$
z_1 \equiv \frac{x^2 + 1}{x} \exp\left(\frac{x^2}{2}\right)
,\qquad
z_2 \equiv x \exp\left(\frac{x^2}{2}\right)
,
$$
then for all $x \ge 1$, if $z_1 \ge \exp(x^2 / 2)$, i.e.~$x \le \sqrt{2 \log z_1}$, we have
$$
\exp\left(\frac{x^2}{2}\right)
\ge \frac{z_1}{2x}
\ge \frac{z_1}{2 \sqrt{2 \log z_1}}
,
$$
which implies that
\beq\label{prob_x_lbound}
x
\ge \sqrt{2 \log z_1 - \log \log z_1 - 3 \log 2}
.
\eeq
If $z_2 \ge \exp(x^2 / 2)$, we have
\beq\label{prob_x_ubound}
x
\le \sqrt{2 \log z_2}
.
\eeq
For all $y \in [1 - 1 / \sqrt{2 \pi e}, 1)$, we can find the solution $x \ge 1$ of
$
\frac{y + 1}{2}
= 1 - \frac{1}{\sqrt{2\pi} z_1}
.
$
By applying \eqref{normal_tail} and \eqref{prob_x_lbound} with this solution $x$, we obtain
\beq\label{prob_Phi-1_lbound}
\Phi^{-1}\left(\frac{y + 1}{2}\right)
\ge x
\ge \sqrt{2 \log \frac{\sqrt{2}}{\sqrt{\pi}(1 - y)} - \log\log \frac{\sqrt{2}}{\sqrt{\pi} (1 - y)} - 3 \log 2}
.
\eeq
We can find solution $x \ge 1$ to
$$
\frac{y + 1}{2}
= 1 - \frac{1}{\sqrt{2\pi} z_2}
.
$$
By applying \eqref{normal_tail} and \eqref{prob_x_ubound} with this solution $x$, we obtain
\beq\label{prob_Phi-1_ubound}
\Phi^{-1}\left(\frac{y + 1}{2}\right)
\le x
\le \sqrt{2 \log \frac{\sqrt{2}}{\sqrt{\pi}(1 - y)}}
.
\eeq
Applying the lower bound \eqref{prob_Phi-1_lbound} and upper bound \eqref{prob_Phi-1_ubound} of $\Phi^{-1}\left( \frac{y + 1}{2} \right)$ to \eqref{prob_chainrule}, we have for all $y \in [1 - 1 / \sqrt{2 \pi e}, 1)$
$$
[F^{-1}]'(y)
\ge \sqrt{2\pi}
\cdot \frac{1}{\sqrt{2 \log \frac{\sqrt{2}}{\sqrt{\pi}(1 - y)}}}
\cdot \frac{\frac{\sqrt{2}}{\sqrt{\pi}(1 - y)}}{\sqrt{8\log \frac{\sqrt{2}}{\sqrt{\pi} (1 - y)}}}
\ge \frac{1}{2 (1 - y) \log(1 - y)^{-1}}
.
$$
Replacing $y$ with $U_-$ and $F^{-1}(y)$ by $M_-$ we have under condition
\beq\label{prob_cond1}
U_- \ge 1 - \frac{1}{\sqrt{2\pi e}}
\eeq
we have
\beq\label{prob_F'M_-}
F'(M_-)
= \frac{1}{[F^{-1}]'(U_-)}
\le 2 (1 - U_-) \log (1 - U_-)^{-1}
\eeq

\item
It is standard from order statistics that $U - U_- \sim \mathrm{Beta}(1, n)$ and $1 - U_- \sim \mathrm{Beta}(2, n - 1)$.
Therefore, for any $\epsilon \in (0, 1 / 3]$ and $n \ge 2$,
$$
\PP\left(U - U_-\le \frac{\epsilon}{n}\right)
\le \int_0^{\frac{\epsilon}{n}} n (1 - x)^{n - 1} \ud x
= 1 - \left(1 - \frac{\epsilon}{n}\right)^n
\le 1 - \left(\frac13\right)^\epsilon
\le (\log 3) \epsilon
$$
where we used elementary inequalities $(1 - 1 / x)^x \ge 1 / 3$ when $x \ge 6$, and $1 - 3^{-x} \le (\log 3) x$ for all $x \in (0, 1 / 3]$.
In addition, for any $\epsilon \in (0, 1 / 3]$ and $n \ge 2 \sqrt{2 \pi e} \log\epsilon^{-1} + 1$ we have
$$\begin{aligned}
&
\PP\left(
1 - U_-\ge \frac{2\log\epsilon^{-1}}{n - 1}
\right)
=
\int_{\frac{2\log\epsilon^{-1}}{n - 1}}^1 n(n - 1) x (1 - x)^{n - 2} \ud x
\\&=
2\log\epsilon^{-1} \frac{n}{n - 1} \left(1 - \frac{2\log\epsilon^{-1}}{n - 1}\right)^{n - 1} 
+ 
\left(1 - \frac{2\log\epsilon^{-1}}{n - 1}\right)^n
\\&\le
\left(2\log\epsilon^{-1} + 1\right) \cdot \left(1 - \frac{2\log\epsilon^{-1}}{n - 1}\right)^{n - 1}
\le
(2\log\epsilon^{-1} + 1) \cdot \epsilon^2
\le
\frac{2 \log 3 + 1}{3} \epsilon
,
\end{aligned}$$
where we used elementary inequalities $(1 - 1/x)^x \le 1 / e$ for all $x \ge \sqrt{2\pi e}$, and $(2\log\epsilon^{-1} + 1) \epsilon \le (2 \log 3 + 1) / 3$ for all $\epsilon \in (0, 1 / 3]$.
Taking union bound, we have
\beq\label{prob_diff_bound}
\PP\left(
U - U_-\ge\frac{\epsilon}{n},
1 - U_-\le\frac{2\log\epsilon^{-1}}{n - 1}
\right)
\ge
1 - 3\epsilon
.
\eeq

\item
Notice that when $\epsilon \in (0, 1 / 3]$ and $n \ge 2\sqrt{2\pi e} \log\epsilon^{-1} + 1$, conditions \eqref{prob_cond0} and \eqref{prob_cond1} hold automatically when \eqref{prob_diff_bound} holds.
Combining \eqref{prob_main}, \eqref{prob_F'M_-} and \eqref{prob_diff_bound}, for all $\epsilon\in (0, 1 / 3]$, $n \ge 2\sqrt{2\pi e} \log\epsilon^{-1} + 1$ we have with probability at least $1 - 3\epsilon$ that
$$
M - M_-
\ge \frac{U - U_-}{F'(M_-)}
\ge \frac{U - U_-}{2(1 - U_-) \log(1 - U_-)^{-1}}
\ge \frac{(n - 1) \epsilon}{4 n \log\epsilon^{-1} \log\frac{n - 1}{2\log\epsilon^{-1}}}
\ge \frac{\epsilon}{8 \log\epsilon^{-1} \log n}
.
$$
Seeing that $(M,M_-)$ is equidistributed as $(\log \chi_{(n)}^2, \log \chi_{(n - 1)}^2)$, for all $\epsilon \in \left(0, 1 / 3\right]$ and $n\ge 2\sqrt{2\pi e} \log\epsilon^{-1} + 1$ we have \eqref{eq:pure_prob} holds with probability $\ge 1 - 3\epsilon$.
\end{enumerate}
\end{proof}

\pb

\begin{proof}[Proof of Lemma \ref{lemm:Winfquan}]
Recall that we define $\eta_S = \eta \cdot \sign(\mu_4 - 3)$.
For any fixed coordinate $k \in [2, d]$, we have
$$\begin{aligned}
&
W_k^{(t)} - W_k^{(t - 1)}
=
\frac{\left(
v_1 + \eta_S (\vb^\top \bY)^3 Y_1
\right)^2 - \left(
v_k + \eta_S (\vb^\top \bY)^3 Y_k
\right)^2}{\left(
v_k + \eta_S (\vb^\top \bY)^3 Y_k
\right)^2}
-
\frac{v_1^2 - v_k^2}{v_k^2}
\\&=
\frac{
2 \eta_S (\vb^\top \bY)^3 v_1 v_k (v_k Y_1 - v_1 Y_k)
+
\eta_S^2 (\vb^\top \bY)^6 (v_k^2 Y_1^2 - v_1^2 Y_k^2)
}{
v_k^2 \left(v_k + \eta_S (\vb^\top \bY)^3 Y_k\right)^2
}
\\&=
\eta_S \cdot \left(
1 + \eta_S (\vb^\top \bY)^3 \frac{Y_k}{v_k}
\right)^{-2} \cdot v_k^{-4} \left(
2 (\vb^\top \bY)^3 v_1 v_k (v_k Y_1 - v_1 Y_k)
+
\eta_S (\vb^\top \bY)^6 (v_k^2 Y_1^2 - v_1^2 Y_k^2)
\right)
.
\end{aligned}$$
Combining with \eqref{Winfquan}, this implies
$$\begin{aligned}
&
Q_{W, k}^{(t)}
=
2 \eta_S \cdot \left[ \left(1 + \eta_S (\vb^\top \bY)^3 \frac{Y_k}{v_k}\right)^{-2} - 1\right]
\cdot (\vb^\top \bY)^3 v_1 v_k^{-2} \left(Y_1 - \frac{v_1}{v_k} Y_k\right)
\\&\hspace{1in}
+
\eta_S^2 \cdot \left(1 + \eta_S (\vb^\top \bY)^3 \frac{Y_k}{v_k}\right)^{-2}
\cdot (\vb^\top \bY)^6 v_k^{-2} \left(Y_1^2 - \frac{v_1^2}{v_k^2} Y_k^2\right)
.
\end{aligned}$$
Taylor series give for all $|x| \le \frac12$ that 
$$
\left|(1 + x)^{-2} - 1\right|
= |x| \left|\sum_{i = 0}^\infty (-1)^i (i + 2) x^i\right|
\le 6 |x|
,\qquad
(1 + x)^{-2} \le 4
.
$$
On event $\cH_{k; \ref{lemm:Winfquan}, L}^{(t)}$, since $|v_1 / v_k| < \sqrt{3}$ and $|v_k|^2 \ge 1 / (3d)$, \eqref{eq:cold_weak_scaling} implies $\left| \eta_S (\vb^\top \bY)^3 \frac{Y_k}{v_k} \right| \le \frac12$, and hence we have
$$
|Q_{W, k}^{(t)}|
\le
2 \eta \cdot 6 \sqrt{3} \rB_o^4 d^{1 / 2} \eta \cdot (3 + 3 \sqrt{3}) \rB_o^4 d^{1 / 2}
+
\eta^2 \cdot 4 \cdot 12 \rB_o^8 d
\le
C_{\ref{lemm:Winfquan}, L} \rB_o^8 \eta^2 d
$$
where constant $C_{\ref{lemm:Winfquan}, L} = 156 + 36\sqrt{3}$.
\end{proof}

\pb

\begin{proof}[Proof of Lemma \ref{lemm:infiW}]
Recall that $W_k^{(t - 1)} = (v_1^2 - v_k^2) / v_k^2$.
Under Assumption \ref{assu:distribution}, using \eqref{eq:vY3Y} we have
$$\begin{aligned}
&\quad
\EE\left[ \left.
\sign(\mu_4 - 3) \eta
\cdot
2 v_1 v_k^{-3} (\vb^\top \bY)^3 (v_1 Y_k - v_k Y_1)
\right| \cF_{t - 1} \right]
\\&=
2 \sign(\mu_4 - 3) \eta
\cdot
v_1 v_k^{-3} \left(
(\mu_4 - 3) v_1 v_k^3
+
3 v_1 v_k
-
(\mu_4 - 3) v_k v_1^3
-
3 v_k v_1
\right)
\\&=
- 2\eta |\mu_4 - 3| \cdot v_1^2 W_k
\end{aligned}$$
\end{proof}

\pb

\begin{proof}[Proof of Lemma \ref{lemm:Ugeom}]
From \eqref{Wgeom} and \eqref{eq:Pot}, we have
\beq\label{eq:Wgeom2}
P^o_s W_k^{(s)} - P^o_{s - 1} W_k^{(s - 1)}
=
P^o_s \left(Q_{W, k}^{(s)} + f_k^{(s)}\right)
\eeq
We iteratively apply \eqref{eq:Wgeom2} for $s = 1, \dots, t$ and obtain \eqref{Wrep}.
\end{proof}

\pb

\begin{proof}[Proof of Lemma \ref{lemm:sum_f_concentration}]
$\{f_k^{(t)}\}_{t \ge 1}$, earlier defined in \eqref{Kfn}, forms a martingale difference sequence with respect to $\cF_{t - 1}$.
Similar to techniques in proof of Lemma \ref{lemm:1order_concentration}, we apply Lemma \ref{lemm:Orlicz_norm_property} three times, then use \eqref{eq:triangle_inequality} and \eqref{eq:psi_2_norm}, in order to obtain the following bound on Orlicz $\psi_{1/2}$-norm
$$\begin{aligned}
&
\left\|\eta \cdot 2 (\vb^\top \bY)^3 v_1 v_k^{-3} (v_k Y_1 - v_1 Y_k)\right\|_{\psi_{1/2}}
\le
2\eta |v_1| |v_k|^{-2}
\cdot \left\| (\vb^\top \bY)^2 \right\|_{\psi_1}
\cdot \left\| (\vb^\top \bY)\left(Y_1 - \frac{v_1}{v_k} Y_k\right) \right\|_{\psi_1}
\\&\le
2 \eta |v_1| |v_k|^{-2}
\cdot \left\| \vb^\top \bY \right\|_{\psi_2}^3
\cdot \left\| Y_1 - \frac{v_1}{v_k} Y_k \right\|_{\psi_2}
\le
6 \eta d^{1 / 2}
\cdot
\left\| \vb^\top \bY \right\|_{\psi_2}^3
\cdot (\|Y_1\|_{\psi_2} + \sqrt{3} \|Y_k\|)
\\&\le
6(1 + \sqrt{3}) B^4 d^{1 / 2} \eta
\end{aligned}$$
where we use the fact that $|v_1 / v_k| \le \sqrt{3}$ and $v_k^2 \ge 1 / (3d)$ on the event $(t \le \cT_{c, k} \wedge \cT_1)$.
Then by applying Lemma \ref{lemm:psi1/2} we further derive
$$\begin{aligned}
&
\|f_k^{(t)}\|_{\psi_{1/2}}
=
\left\|
\eta \cdot 2 (\vb^\top \bY)^3 v_1 v_k^{-3} (v_k Y_1 - v_1 Y_k)
-
\EE \left[\eta \cdot 2 (\vb^\top \bY)^3 v_1 v_k^{-3} (v_k Y_1 - v_1 Y_k)\right]
\right\|_{\psi_{1/2}}
\\&\le
C_{\ref{lemm:psi1/2}, L}' \left\|\eta \cdot 2 (\vb^\top \bY)^3 v_1 v_k^{-3} (v_k Y_1 - v_1 Y_k)\right\|_{\psi_{1/2}}
\le
6(1 + \sqrt{3}) C_{\ref{lemm:psi1/2}, L}' B^4 d^{1 / 2} \eta
\end{aligned}$$
Because $1_{(t \le \cT_{c, k} \wedge \cT_1)} \in \cF_{t - 1}$, we know that $\left\{f_k^{(t)} 1_{(t \le \cT_{c, k} \wedge \cT_1)}\right\}$ forms a martingale difference sequence with respect to $\cF_{t - 1}$, and $\|f_k^{(t)} 1_{(t \le \cT_{c, k} \wedge \cT_1)}\|_{\psi_{1/2}} \le \|f_k^{(t)}\|_{\psi_{1/2}}$.
Since $v_1^2 \ge 1 / d$ on the event $(t \le \cT_{c, k} \wedge \cT_1)$ and $2 \eta |\mu_4 - 3| / d < 1$ under scaling condition \eqref{eq:uniform_scaling}, by summation of geometric series, for all $t \ge 1$ we have
$$\begin{aligned}
&
\sum_{s = 1}^t (P_s^o)^2 \left\|f_k^{(s)} 1_{(t \le \cT_{c, k} \wedge \cT_1)}\right\|_{\psi_{1/2}}^2
\le
\sum_{s = 1}^t \left(1 + \frac{2 \eta |\mu_4 - 3|}{d}\right)^{-2s}
\cdot 36 (1 + \sqrt{3})^2 C_{\ref{lemm:psi1/2}, L}'^2 B^8 d \eta^2
\\&\le
\frac{\left(1 + \frac{2 \eta |\mu_4 - 3|}{d}\right)^{-2}}{1 - \left(1 + \frac{2 \eta |\mu_4 - 3|}{d}\right)^{-2}}
\cdot 36 (1 + \sqrt{3})^2 C_{\ref{lemm:psi1/2}, L}'^2 B^8 d \eta^2
=
18 (1 + \sqrt{3})^2 C_{\ref{lemm:psi1/2}, L}'^2
\cdot \frac{B^8}{|\mu_4 - 3|}
\cdot d^2 \eta
\end{aligned}$$
With the bound given above, we apply Theorem \ref{theo:concentration} with $\alpha = 1 / 2$ as
$$\begin{aligned}
&\quad
\PP \left(
\max_{1 \le t \le T_{\eta, 0.5}^o \wedge \cT_{c, k} \wedge \cT_1} \left|\sum_{s = 1}^t P_s^o f_k^{(s)}\right|
\ge 
C_{\ref{lemm:sum_f_concentration}, L} \log^{5 / 2}\delta^{-1} \cdot \frac{B^4}{|\mu_4 - 3|^{1 / 2}} \cdot d \eta^{1 / 2}
\right)
\\&=
\PP \left(
\max_{1 \le t \le T_{\eta, 0.5}^o} \left|\sum_{s = 1}^t P_s^o f_k^{(s)} 1_{(s \le \cT_{c, k} \wedge \cT_1)}\right|
\ge 
C_{\ref{lemm:sum_f_concentration}, L} \log^{5 / 2}\delta^{-1} \cdot \frac{B^4}{|\mu_4 - 3|^{1 / 2}} \cdot d \eta^{1 / 2}
\right)
\\&\le
2 \left[
3 + 6^4 \frac{64 \cdot 18 (1 + \sqrt{3})^2 C_{\ref{lemm:psi1/2}, L}'^2 B^8|\mu_4 - 3|^{-1} d^2 \eta}{C_{\ref{lemm:sum_f_concentration}, L}^2 \log^5\delta^{-1} B^8 |\mu_4 - 3|^{-1} d^2 \eta}
\right]
\\&\hspace{2in}\cdot
\exp\left\{ - \left(
\frac{C_{\ref{lemm:sum_f_concentration}, L}^2 \log^5\delta^{-1} B^8 |\mu_4 - 3|^{-1} d^2 \eta}{32 \cdot 18 (1 + \sqrt{3})^2 C_{\ref{lemm:psi1/2}, L}'^2 B^8|\mu_4 - 3|^{-1} d^2 \eta}
\right)^{\frac15} \right\}
\\&=
\left(6 + \frac{5184}{\log^5\delta^{-1}}\right) \delta
\end{aligned}$$
where constant $C_{\ref{lemm:sum_f_concentration}, L} = 24(1 + \sqrt{3}) C_{\ref{lemm:psi1/2}, L}'$.
\end{proof}

\pb\section{A Reversed Gronwall's Inequality}\label{sec:reversed_gronwall}

We present a discrete generalization of Gronwall's inequality, which is sharper than a straightforward application of Gronwall's inequality.
Such sharp estimation plays a key role in our analysis.
Although elementary, the lemma seems not recorded in relevant literature:

\begin{lemma}\label{lemm:grw_type}
If for all $t=1, \dots, T$, $\beta(t) \in [0,1)$ and if $u(t)$ satisfies that for some positive constant $\alpha$
\beq\label{Gronwallstar}
\left|
 u(t) - u(0) + \sum_{0\le s < t} \beta(s) u(s)
\right|
\le
\alpha
\eeq
then for all $t=1, \dots, T$
\beq\label{Gronwall_final}
\left|
 u(t) - u(0) \prod_{0\le s < t} ( 1 -\beta(s) )
\right|
\leq
2\alpha
-
\alpha \prod_{0\le s< t}  \left( 1 -\beta(s) \right)
\leq
2\alpha
.
\eeq
\end{lemma}
Note unlike analogous Gronwall-type results, here we pose no assumption on the sign of $u(t)$.

\begin{proof}[Proof of Lemma \ref{lemm:grw_type}]
Define for all $t=1, \dots, T$
\beq\label{def_auxv}
v(t)
 =
\prod_{0\le s<t} \left( 1 -\beta(s)  \right)^{-1}
\left(
u(0)
-
\sum_{0\le s<t} \beta(s) u(s)
\right)
.
\eeq
We have for $s=0,\dots,t-1$
$$\begin{aligned}
    &
v(s+1) - v(s)
 =
- \prod_{0\le r<s+1} \left( 1 -\beta(r)  \right)^{-1} \cdot  \beta(s) u(s)
\\&\quad+
\left(
 u(0) - \sum_{0\le r<s} \beta(r) u(r)
\right)
\left(
\prod_{0\le r<s+1} \left( 1 -\beta(r)  \right)^{-1} - \prod_{0\le r<s} \left( 1 -\beta(r)  \right)^{-1}
\right)
\\&=
-\prod_{0\le r<s+1} \left( 1 -\beta(r)  \right)^{-1} \cdot  \beta(s) u(s)
+
\left(
 u(0) - \sum_{0\le r<s} \beta(r) u(r)
\right)
\cdot \left(
  \prod_{0\le r<s+1} \left( 1 -\beta(r)  \right)^{-1}  \cdot \beta(s)
\right)
\\&=
-\left(
 u(s) - u(0) + \sum_{0\le r<s} \beta(r) u(r)
\right)
\cdot  \beta(s) \prod_{0\le r<s+1} \left( 1 -\beta(r)  \right)^{-1} 
.
\end{aligned}$$
Since $\beta(s)\in [0,1)$, and the product is nonnegative, the use of the lower side of \eqref{Gronwallstar} upper-estimates the difference of $v(s)$
\beq\label{Gronlower}
v(s+1) - v(s)
\le
\alpha \beta(s)  \prod_{0\le r<s+1} \left( 1 -\beta(r)  \right)^{-1} 
\eeq
Since $v(0) = u(0)$, telescoping the above inequality for $s=0,\dots,t-1$ gives
$$
v(t)
=
v(0) + \sum_{0\le s<t} v(s+1) - v(s)
\le
u(0) + \alpha \sum_{0\le s<t} \beta(s)  \prod_{0\le r<s+1} \left( 1 -\beta(r)  \right)^{-1} 
$$
and hence from the definition of $v(t)$ in \eqref{def_auxv}
$$\begin{aligned}
&
u(0) - \sum_{0\le s<t} \beta(s) u(s)
=
\prod_{0\le s<t} \left( 1 -\beta(s)  \right)  v(t)
\\&\le
\prod_{0\le s<t} \left( 1 -\beta(s)  \right)  \left(
 u(0) + \alpha \sum_{0\le s<t} \beta(s)  \prod_{0\le r<s+1} \left( 1 -\beta(r)  \right)^{-1} 
\right)
\\&=
u(0) \prod_{0\le s<t} \left( 1 -\beta(s)  \right)
+
\alpha \sum_{0\le s<t} \beta(s)  \prod_{s+1\le r<t} \left( 1 -\beta(r)  \right)
.
\end{aligned}$$
Taking the above result into the upper side of \eqref{Gronwallstar} gives
$$
u(t)
\le
\alpha
+
u(0) - \sum_{0\le s < t} \beta(s) u(s)
\le
\alpha
+
u(0) \prod_{0\le s<t} \left( 1 -\beta(s)  \right)
+
\alpha \sum_{0\le s<t} \beta(s)  \prod_{s+1\le r<t} \left( 1 -\beta(r)  \right)
,
$$
which further reduces to
$$\begin{aligned}
u(t) - u(0) \prod_{0\le s<t} \left( 1 -\beta(s)  \right)
&\le
\alpha
+
\alpha \sum_{0\le s<t} \beta(s)  \prod_{s+1\le r<t} \left( 1 -\beta(r)  \right)
.
\end{aligned}$$
That is, for all $t=0,1,\dots,T$,
$$\begin{aligned}
&
u(t) - u(0) \prod_{0\le s<t} \left( 1 -\beta(s)  \right)
\le
\alpha
+
\alpha \sum_{0\le s<t} \beta(s)  \prod_{s+1\le r<t} \left( 1 -\beta(r)  \right)
\\&=
\alpha
+
\alpha \sum_{0\le s<t} \left(1 - (1-\beta(s)) \right)  \prod_{s+1\le r<t} \left( 1 -\beta(r)  \right)
\\&=
\alpha
+
\alpha \sum_{0\le s<t} \left[
\prod_{s+1\le r<t} \left( 1 -\beta(r)  \right)
-
\prod_{s\le r<t} \left( 1 -\beta(r)  \right)
\right]
\\&=
2\alpha
-
\alpha \prod_{0\le r<t} \left( 1 -\beta(r)  \right)
\le
2\alpha
,
\end{aligned}$$
which proves the upper side of \eqref{Gronwall_final}.
For the lower side, applying the same inequality to $-u$ in the place of $u$ gives the desired result.
\end{proof}

\pb\section{Orlicz $\psi$-Norm}\label{sec:psi_alpha}

In this section, we recap Orlicz $\psi$-norm and relevant properties.
Many of the contributions can be found in standard texts \cite{wainwright2019high,VAN-WELLNER}, and we derive its (old and new) conclusions for our use in online tensorial ICA analysis.
We begin with its definition

\begin{definition}[Orlicz $\psi$-norm]\label{defi:orlicz}
For a continuous, monotonically increasing and convex function $\psi(x)$ defined for all $x \ge 0$ satisfying $\psi(0) = 0$ and $\lim_{x \to \infty} \psi(x) = \infty$, we define the Orlicz $\psi$-norm for a random variable $X$ as
$$
\|X\|_\psi
\equiv
\inf\left\{ K > 0: \EE \psi(|X / K|) \le 1 \right\}
.
$$
\end{definition}


When $\psi(x)$ is monotonically increasing and convex for $x \ge 0$, the Orlicz $\psi$-norm satisfies triangle inequality, i.e.~for any random variables $X$ and $Y$ we have 
$\|X + Y\|_\psi \le \|X\|_\psi + \|Y\|_\psi$.
The central case of interest in this paper is the Orlicz $\psi_\alpha$-norm defined for a random variable $X$ as
$$
\|X\|_{\psi_\alpha}
\equiv
\inf\left\{ K > 0: \EE \exp\left( \frac{|X|^\alpha}{K^\alpha} \right) \le 2 \right\}
.
$$
In above we adopt the function $\psi_\alpha(x) \equiv \exp(x^\alpha) - 1$.
When $\alpha \ge 1$, $\psi_\alpha(x)$ is monotonically increasing and convex for $x \ge 0$, and consequently $\psi_\alpha$-norm satisfies triangle inequality \citep{VAN-WELLNER}
\beq\label{eq:triangle_inequality}
\|X + Y\|_{\psi_\alpha} \le \|X\|_{\psi_\alpha} + \|Y\|_{\psi_\alpha}
.
\eeq
We state a multiplicative property of the Orlicz $\psi_\alpha$-norm which extends a standard result (e.g.~Proposition D.3 in \cite{vu2013minimax}):

\begin{lemma}[Multiplicative property]\label{lemm:Orlicz_norm_property}
Let $X$ and $Y$ be random variables then for some $\alpha \ge 1$
$$
\|XY\|_{\psi_{\alpha / 2}} \le \|X\|_{\psi_\alpha} \|Y\|_{\psi_\alpha}
.
$$
\end{lemma}

\begin{proof}[Proof of Lemma \ref{lemm:Orlicz_norm_property}]
The inequality is trivial when 
(i)
$\|X\|_{\psi_\alpha} = 0$ or $\|Y\|_{\psi_\alpha} = 0$ since one of the variables is 0 a.s., or
(ii)
either $X$ or $Y$ has an infinite $\psi_\alpha$-norm.
Otherwise let $A \equiv X / \|X\|_{\psi_\alpha}$ and $B \equiv Y / \|Y\|_{\psi_\alpha}$ then $\|A\|_{\psi_\alpha} = \|B\|_{\psi_\alpha} = 1$.
Using $|AB| \le \frac14 (|A| + |B|)^2$ and triangle inequality in \eqref{eq:triangle_inequality} we have
$$
\|AB\|_{\psi_{\alpha / 2}}
\le
\left\|\frac14 (|A| + |B|)^2\right\|_{\psi_{\alpha / 2}}
=
\frac14 \left\||A| + |B|\right\|_{\psi_\alpha}^2
\le
\frac14 \left(\|A\|_{\psi_\alpha} + \|B\|_{\psi_\alpha}\right)^2
=
1
.
$$
Multiplying both sides of the inequality by the positive $\|X\|_{\psi_\alpha} \|Y\|_{\psi_\alpha}$ gives the desired result.
\end{proof}

In the case of $\alpha\in (0,1)$, $\psi_\alpha(x)$ is no longer convex;
in fact it is only convex after a modification of the function in a neighborhood of $x=0$.
Strictly speaking, $\psi_\alpha$-``norm'' in this case is not a norm (defined in a Banach space) in the sense that the triangle inequality \eqref{eq:triangle_inequality} does not hold for this set of $\alpha$'s.
For the purpose of our applications, a generalized triangle inequality holds for the $\alpha = 1/2$ case in the following Lemma \ref{lemm:psi1/2}:

\begin{lemma}\label{lemm:psi1/2}
For any random variables $X, Y$ we have the following inequalities for Orlicz $\psi_{1/2}$-norm
\beq\label{gentri}
\|X + Y\|_{\psi_{1/2}}
\le
C_{\ref{lemm:psi1/2}, L} (\|X\|_{\psi_{1/2}} + \|Y\|_{\psi_{1/2}})
\quad\text{and}\quad
\|\EE X\|_{\psi_{1/2}}
\le
C_{\ref{lemm:psi1/2}, L} \|X\|_{\psi_{1/2}}
\eeq
with $C_{\ref{lemm:psi1/2}, L} \equiv 1.3937$.
In addition
\beq\label{gentritwo}
\|X - \EE X\|_{\psi_{1/2}}
\le
C_{\ref{lemm:psi1/2}, L}' \|X\|_{\psi_{1/2}}
,
\eeq
where $C_{\ref{lemm:psi1/2}, L}' \equiv 3.3359$.
\end{lemma}

\begin{proof}[Proof of Lemma \ref{lemm:psi1/2}]
We first rule out some simple cases.
Analogously, we only need to prove \eqref{gentri} and \eqref{gentritwo} under the setting that both $\|X\|_{\psi_\alpha}$ and $\|Y\|_{\psi_\alpha}$ belongs to the two-point set $\{0,+\infty\}$.
\red{We do have ``uniqueness'' $\|X\|_{\psi_\alpha} = 0$ then $X = 0$ a.s. Short proof?}%
\red{By definition, $\|X\|_{\psi_\alpha} = 0$ implies $\EE \exp(|X|^\alpha / K^\alpha) \le 2$ for all $K > 0$.
Applying Jensen's inequality, we obtain $\exp(\EE |X|^\alpha / K^\alpha) \le \EE \exp(|X|^\alpha / K^\alpha) \le 2$, i.e. $\EE |X|^\alpha \le K^\alpha \log 2$ for all $K > 0$, which implies $\EE |X|^\alpha = 0$ and $X = 0$ a.s.
}%
Recall that when $\alpha \in (0,1)$, $\psi_\alpha(x)$ does not satisfy convexity when $x$ is around 0.
Let the modified Orcliz-$\alpha$ function be
$$
\tilde\psi_\alpha(x)
:=
\left\{\begin{array}{ll}
\exp(x^\alpha) - 1								&	x \ge x_\alpha
\\
\frac{x}{x_\alpha}\left( \exp(x_\alpha^\alpha) - 1 \right)	&	x \in [0, x_\alpha)
\end{array}\right.
$$
for some appropriate $x_\alpha > 0$, so as to make the function convex.
Here $x_\alpha$ is chosen such that the tangent line of function $\psi_\alpha$ at $x_\alpha$ passes through origin, i.e.
$$
\alpha x_\alpha^{\alpha - 1} \exp(x_\alpha^\alpha)
= \frac{\exp(x_\alpha^\alpha) - 1}{x_\alpha}
.
$$
Simplifying it gives us a transcendental equation $(1 - \alpha x_\alpha^\alpha) \exp(x_\alpha^\alpha) = 1$ which (admits no an analytic solution but) can be solved numerically.
When $\alpha = 1 / 2$, we have $x_{1 / 2} = 2.5396$.
Some numerical calculation yields
\beq\label{eq:psi_tildepsi}
0 \le \psi_{1/2} (x) - \tilde{\psi}_{1/2} (x) \le 0.2666
.
\eeq
From \eqref{eq:psi_tildepsi} we immediately have
$$
\EE\tilde\psi_{1/2}(|X|) \le 1
	\quad\Longrightarrow\quad
\EE\psi_{1/2}(|X|) \le 1.2666
	\quad\text{i.e.~}
\EE\exp(|X|^{1/2}) \le 2.2666
$$

\begin{enumerate}[label=(\arabic*)]
\item
Let $K_1, K_2$ denote the $\psi_{1/2}$ norms of $X$ and $Y$, then $\EE \psi_{1/2} (|X / K_1|) \le 1$ and $\EE \psi_{1/2} (|Y / K_2|) \le 1$.
Based on \eqref{eq:psi_tildepsi} we have $\EE \tilde\psi_{1/2} (|X / K_1|)\le 1$ and $\EE \tilde\psi_{1/2} (|Y / K_2|) \le 1$.
Applying triangle inequality \eqref{eq:triangle_inequality} to Orlicz $\tilde\psi_{1/2}$-norm, we have
$$
\EE \tilde\psi_{1/2}\left(\left| \frac{X + Y}{K_1 + K_2}\right|\right)
\le 1
\quad\Longrightarrow\quad
\EE \psi_{1/2} \left(\left| \frac{X + Y}{K_1 + K_2}\right|\right)
\le 1.2666
.
$$
where we applied \eqref{eq:psi_tildepsi}.
Now, applying Jensen's inequality to concave function $f(z) = z^{\log_{2.2666}2}$ gives that, for constant $C_{\ref{lemm:psi1/2}, L} \equiv (\log_2 2.2666)^2 = 1.3937$,
$$\begin{aligned}
&
\EE \psi_{1/2} (|(X + Y) / (C_{\ref{lemm:psi1/2}, L} (K_1 + K_2))|)
=\EE \exp(|(X + Y) / (K_1 + K_2)|^{1 / 2})^{\log_{2.2666} 2} - 1
\\&\le
\left(\EE \exp(|(X + Y) / (K_1 + K_2)|^{1 / 2})\right)^{\log_{2.2666} 2} - 1
\le 1
,
\end{aligned}$$
which implies the first conclusion of \eqref{gentri}:
$$
\|X + Y\|_{\psi_{1/2}}
\le C_{\ref{lemm:psi1/2}, L} (\|X\|_{\psi_{1/2}} + \|Y\|_{\psi_{1/2}})
.
$$

\item
Let $K$ denote the $\psi_{1/2}$ norm of $X$, then $\EE \psi_{1/2}(|X / K|) \le 1$.
Based on \eqref{eq:psi_tildepsi} we have $\EE \tilde\psi_{1/2}(|X / K|) \le 1$.
Because $\tilde\psi_{1/2}$ is a convex function, we can apply Jensen's inequality as
$$
\tilde\psi_{1/2} (|\EE X / K|)
\le
\tilde\psi_{1/2} (\EE |X / K|)
\le
\EE \tilde\psi_{1/2} (|X / K|)
\le
1
$$
Combining with \eqref{eq:psi_tildepsi}, it holds that $\psi_{1/2}(|\EE X / K|) \le 1.2666$ which is equivalent to \\
$\psi_{1/2}(|\EE X / (C_{\ref{lemm:psi1/2}, L} K)|) \le 1$, where $C_{\ref{lemm:psi1/2}, L} = (\log_2 2.2666)^2 = 1.3937$.
Noticing that \\
$\EE \psi_{1/2}(\EE X / (C K)) = \psi_{1/2}(\EE X / (C K))$, it holds that $\|\EE X\|_{\psi_{1/2}}\le C_{\ref{lemm:psi1/2}, L} \|X\|_{\psi_{1/2}}$ which concludes \eqref{gentri}.

\item
To conclude \eqref{gentritwo}, we have
$$
\|X - \EE X\|_{\psi_{1/2}}
\le
C_{\ref{lemm:psi1/2}, L} (\|X\|_{\psi_{1/2}} + \|\EE X\|_{\psi_{1/2}})
\le
C_{\ref{lemm:psi1/2}, L} (1 + C_{\ref{lemm:psi1/2}, L}) \|X\|_{\psi_{1/2}}
=
C_{\ref{lemm:psi1/2}, L}' \|X\|_{\psi_{1/2}}
$$
where positive constant $C_{\ref{lemm:psi1/2}, L}' \equiv C_{\ref{lemm:psi1/2}, L} (1 + C_{\ref{lemm:psi1/2}, L}) = 3.3359$.
\end{enumerate}
\end{proof}

\pb\section{A Concentration Inequality for Martingales with Weak Exponential-type Tails}\label{sec:concentration}

We prove a novel concentration inequality for 1-dimensional supermartingale difference sequence with finite $\psi_\alpha$-norms that plays an important role in our analysis:

\begin{theorem}\label{theo:concentration}
Let $\alpha \in (0, \infty)$ be given.
Assume that $(\u_i: i\ge 1)$ is a sequence of supermartingale differences with respect to $\cF_i$, i.e.~$\EE[\u_i \mid \cF_{i-1}] \le 0$, and it satisfies $\| \u_i \|_{\psi_\alpha} < \infty$ for each $i=1,\dots,N$.
Then for an arbitrary $N\ge 1$ and $z > 0$,
\beq\label{Sbound}
\PP\left( 
\max_{1\le n\le N} \sum_{i=1}^n \u_i   \ge z
\right)
\le
\left[
3 + \left( \frac{3}{\alpha} \right)^{ \frac{2}{\alpha}}
\frac{64 \sum_{i=1}^N \|\u_i\|_{\psi_\alpha}^2}{z^2}
\right]
\exp\left\{ - 
\left(
\frac{z^2}{32 \sum_{i=1}^N \|\u_i\|_{\psi_\alpha}^2}
\right)^{\frac{\alpha}{\alpha+2}} \right\}
.
\eeq
\end{theorem}

\begin{proof}[Proof of Theorem \ref{theo:concentration}]
To prove Theorem \ref{theo:concentration}, we will use a maxima version of the classical Azuma-Hoeffding's inequality proposed by \citet{laib1999exponential} for bounded martingale differences, and then apply an argument of \citet{lesigne2001large} and \citet{fan2012large} to truncate the tail and analyze the bounded and unbounded pieces separately.

\begin{enumerate}[label=(\arabic*)]
\item\label{rescaling}
First of all, for the sake of simplicity and with no loss of generality, throughout the following proof of Theorem \ref{theo:concentration} we shall pose the following extra condition
\beq\label{WLOGassu}
\sum_{i=1}^N \|\u_i\|_{\psi_\alpha}^2 = 1
.
\eeq
In other words, under the additional \eqref{WLOGassu} condition proving \eqref{Sbound} reduces to showing
\beq\label{Sbound_prime}
\PP\left( 
\max_{1\le n\le N} \sum_{i=1}^n \u_i   \ge z
\right)
\le
\left[
3 + \left( \frac{3}{\alpha} \right)^{ \frac{2}{\alpha}}
\frac{64 }{z^2}
\right]
\exp\left\{ - 
\left(
\frac{z^2}{32}
\right)^{\frac{\alpha}{\alpha+2}} \right\}
.
\eeq
This can be made more clear from the following rescaling argument: one can put in the left of \eqref{Sbound_prime} $\u_i / \left( \sum_{i=1}^N \|\u_i\|_{\psi_\alpha}^2 \right)^{1/2}$ in the place of $\u_i$, and $z / \left( \sum_{i=1}^N \|\u_i\|_{\psi_\alpha}^2 \right)^{1/2}$ in the place of $z$, the left hand of \eqref{Sbound} is just
$$
\PP\left( 
\max_{1\le n\le N} \sum_{i=1}^n \frac{ \u_i}{
  \left( \sum_{i=1}^N \|\u_i\|_{\psi_\alpha}^2 \right)^{1/2}
  }  \ge  \frac{z}{ 
   \left( \sum_{i=1}^N \|\u_i\|_{\psi_\alpha}^2 \right)^{1/2}
   }
\right)
$$
which, by \eqref{Sbound_prime}, is upper-bounded by
$$
\le
\left[
3 + \left( \frac{3}{\alpha} \right)^{ \frac{2}{\alpha}}
\frac{64 \sum_{i=1}^N \|\u_i\|_{\psi_\alpha}^2}{z^2}
\right]
\exp\left\{ - 
\left(
\frac{z^2}{32 \sum_{i=1}^N \|\u_i\|_{\psi_\alpha}^2  }
\right)^{\frac{\alpha}{\alpha+2}} \right\}
,
$$
proving \eqref{Sbound}.

\item
We apply a truncation argument used in \citet{lesigne2001large} and later in \citet{fan2012large}.
Let $\cM > 0$ be arbitrary, and we define
\beq\label{Xp}
\u_i'
=
\u_i 1_{\{ |\u_i| \le \cM \|\u_i\|_{\psi_\alpha} \}}
-
\EE\left(\u_i 1_{\{ |\u_i| \le \cM \|\u_i\|_{\psi_\alpha} \}} \mid \cF_{i-1}\right)
,
\eeq
\beq\label{Xpp}
\u_i''
=
\u_i 1_{\{ |\u_i| > \cM \|\u_i\|_{\psi_\alpha} \}}
-
\EE\left(\u_i 1_{\{ |\u_i| > \cM \|\u_i\|_{\psi_\alpha} \}} \mid \cF_{i-1}\right)
,
\eeq
$$
T_n'
=
\sum_{i=1}^n \u_i'
,\qquad
T_n''
=
\sum_{i=1}^n \u_i''
,\qquad
T_n'''
 =
\sum_{i=1}^n  \EE\left(  \u_i \mid \cF_{i-1} \right)
.
$$
Since $\u_i$ is $\cF_i$-measurable, $u _i'$ and $u _i''$ are two martingale difference sequences with respect to $\cF_i$, and let $T_n$ be defined as
\beq\label{Tndef}
T_n
=
\sum_{i=1}^n \u_i
\quad\text{and hence}\quad
T_n = T_n' + T_n'' + T_n'''
.
\eeq
Since $\u_i$ are supermartingale differences we have that $T_n'''$ is $\cF_{n - 1}$-measurable with $T_n''' \le T_0''' = 0,\ a.s.$, and hence for any $z > 0$,
\beq\label{Sk_sep}
\begin{aligned}
\PP\left(
\max_{1\le n\le N} T_n \ge 2z
\right)
&\le
\PP\left(
\max_{1\le n\le N} T_n' + T_n''' \ge z
\right)
+
\PP\left(
\max_{1\le n\le N} T_n'' \ge z
\right)
\\ &\le
\PP\left(
\max_{1\le n\le N} T_n'  \ge z
\right)
+
\PP\left(
\max_{1\le n\le N} T_n'' \ge z
\right)
\end{aligned}\eeq
In the following, we analyze the tail bounds for $T_n'$ and $T_n''$ separately \citep{lesigne2001large,fan2012large}.

\item
To obtain the first bound, we recap Laib's inequality as follows:
\begin{lemma}\label{lemm:laib}\citep{laib1999exponential}
Let $(\w_i: 1\le i\le N)$ be a real-valued martingale difference sequence with respect to some filtration $\cF_i$, i.e.~$\EE[\w_i \mid \cF_{i-1}]  = 0, a.s.$, and the essential norm $\|\w_i\|_\infty$ is finite.
Then for an arbitrary $N\ge 1$ and $z > 0$,
\beq\label{laib}
\PP\left(
\max_{n\le N}   \sum_{i=1}^n \w_i   \ge  z
\right)
\le
\exp\left\{ - \frac{z^2}{2 \sum_{i=1}^N \|\w_i\|_\infty^2 } \right\}
.
\eeq
\end{lemma}

\eqref{laib} generalizes the folklore Azuma-Hoeffding's inequality, where the latter can be concluded from
$$
\max_{n\le N}   \sum_{i=1}^n \w_i   
\ge 
\sum_{i=1}^N \w_i 
.
$$
The proof of Lemma \ref{lemm:laib} is given in \citet{laib1999exponential}.

Recall our extra condition \eqref{WLOGassu}, then from the definition of $\u_i'$ in \eqref{Xp} that $|u _i' | \le 2 \cM \|\u_i\|_{\psi_\alpha}$, the desired bound follows immediately from Laib's inequality in Lemma \ref{lemm:laib} by setting $\w_i = \u_i'$:
\beq\label{Sp_ineq}
\PP\left( \max_{1\le n\le N} T_n' \ge z \right)
=
\PP\left( \max_{1\le n\le N}  \sum_{i=1}^n \u_i' \ge z \right)
\le
\exp\left\{ - \frac{z^2}{8 \cM^2} \right\}
\eeq

To obtain the tail bound of $T_n''$ we only need to show
\beq\label{Xpp_sq}
\EE (\u_i'')^2
\le
(6 \cM^2 + 8 \cB^2) \|\u_i\|_{\psi_\alpha}^2 \exp\left\{
-\cM^{\alpha}
\right\}
,
\eeq
where
\beq\label{beta}
\cB \equiv \left( \frac{3}{\alpha} \right)^{ \frac{1}{\alpha}}
,
\eeq
from which, Doob's martingale inequality implies immediately that
\beq\label{Spp_ineq}
\PP\left(
\max_{1\le n\le N} T_n''  \ge z
\right)
  \le
\frac{1}{z^2} \sum_{i=1}^N \EE (\u_i'')^2
\le
\frac{6 \cM^2 + 8 \cB^2}{z^2} \exp\left\{
-\cM^{\alpha}
\right\}
.
\eeq

To prove \eqref{Xpp_sq}, first recall from the definition of $\u_i''$ in \eqref{Xpp} that
$$
\u_i''
=
\u_i 1_{\{ |\u_i| > \cM \|\u_i\|_{\psi_\alpha} \}}
-
\EE\left(\u_i 1_{\{ |\u_i| > \cM \|\u_i\|_{\psi_\alpha} \}} \mid \cF_{i-1}\right)
.
$$
Recall from the property of conditional expectation that for any random variable $W$ and a $\sigma$-algebra $\cG \subseteq \cF$
$$
\EE \left[ W - \EE( W \mid \cG) \right]^2
=
\EE W^2 
-
\EE\left[ \EE\left( W\mid \cG \right) \right]^2
\le
\EE W^2 
= 
\int_0^\infty 2y \PP(|W| > y) dy
$$
where the last equality is due to a simple application of Fubini's Theorem for nonnegative random variable $|W|$.
Plugging in $W = \u_i 1_{\{ |\u_i| > \cM \|\u_i\|_{\psi_\alpha} \}}$ and $\cG = \cF_{i-1}$ we have
\beq\label{Xppmoment}
\begin{aligned}
&
\EE (\u_i'')^2
=
\EE \left[
\u_i 1_{\{ |\u_i| > \cM \|\u_i\|_{\psi_\alpha} \}}
-
\EE\left(\u_i 1_{\{ |\u_i| > \cM \|\u_i\|_{\psi_\alpha} \}} \mid \cF_{i-1}\right)
\right]^2
\\&\le
\int_0^\infty 2y \PP( |\u_i| 1_{|\u_i| > \cM \|\u_i\|_{\psi_\alpha}} > y ) dy
\\&=
\int_0^{\cM \|\u_i\|_{\psi_\alpha}} 2y dy\cdot \PP( |\u_i|  > \cM \|\u_i\|_{\psi_\alpha} )
+
\int_{\cM \|\u_i\|_{\psi_\alpha}}^\infty 2y \PP( |\u_i| > y ) dy
\\&=
\cM^2 \|\u_i\|_{\psi_\alpha}^2 \PP( |\u_i| > \cM \|\u_i\|_{\psi_\alpha} )
+
\int_\cM^\infty 2 t \|\u_i\|_{\psi_\alpha} \PP( |\u_i| > t \|\u_i\|_{\psi_\alpha} )  ~ \|\u_i\|_{\psi_\alpha} dt
\\&\le
2 \cM^2 \|\u_i\|_{\psi_\alpha}^2 \exp\left\{ -\cM^{\alpha} \right\} + 4 \|\u_i\|_{\psi_\alpha}^2 \int_\cM^\infty t \exp\{ - t^{\alpha} \} ~ dt
,
\end{aligned}\eeq
where the last inequality is due to Markov's inequality that for all $z > 0$
\beq\label{Xp_ineq}
\PP(|\u_i | / \|\u_i\|_{\psi_\alpha} \ge z)
\le 
\exp\{ - z^{\alpha} \} \EE \exp\{ |\u_i | ^{\alpha} / \|\u_i\|_{\psi_\alpha}^\alpha \}
\le
2 \exp\{  - z^{\alpha} \}
.
\eeq
It can be shown from basic calculus that the function $g(t) = t^3 \exp\{ - t^{\alpha} \}$ is decreasing in $[\cB, +\infty)$ and is increasing in $[0, \cB]$, where $\cB$ was earlier defined in \eqref{beta} \citep{fan2012large}.
If $\cM \in [\cB, \infty)$ we have
\beq\label{ugebeta}
\begin{aligned}
&
\int_\cM^\infty t\exp\left\{ -t^{\alpha} \right\} ~dt
=
\int_\cM^\infty t^{-2} t^3\exp\left\{ -t^{\alpha} \right\} ~dt
\le
\int_\cM^\infty t^{-2}  ~dt \cdot \cM^3 \exp\left\{ -\cM^{\alpha} \right\}
\\&=
\cM^{-1} \cdot \cM^3 \exp\left\{ -\cM^{\alpha} \right\}
=
\cM^2 \exp\left\{ -\cM^{\alpha} \right\}
.
\end{aligned}\eeq
If $\cM \in (0, \cB)$, we have by setting $\cM$ as $\cB$ in above
\beq\label{ulebeta}
\begin{aligned}
&
\int_\cM^\infty t \exp\left\{ - t^{\alpha} \right\} ~dt
=
\int_\cM^\cB t \exp\left\{ - t^{\alpha} \right\} ~dt
+
\int_\cB^\infty t \exp\left\{ - t^{\alpha} \right\} ~dt
\\ &\le
\int_\cM^\cB ~dt \cdot \cB \exp\left\{ - \cM^{\alpha} \right\} 
+
\cB^2\exp\left\{ -\cB^{\alpha} \right\}
\\ &\le
(\cB - \cM) \cB \exp\left\{ - \cM^{\alpha} \right\} 
+
\cB^2\exp\left\{ -\cM^{\alpha} \right\}
\le
2 \cB^2 \exp\left\{ - \cM^{\alpha} \right\}
.
\end{aligned}\eeq
Combining \eqref{Xppmoment} with the two above displays \eqref{ugebeta} and \eqref{ulebeta} we obtain
$$\begin{aligned}
&
\EE (\u_i'')^2
\le
2 \cM^2 \|\u_i\|_{\psi_\alpha}^2 \exp\left\{ -\cM^{\alpha} \right\}
+
4 \|\u_i\|_{\psi_\alpha}^2 \int_\cM^\infty t \exp\{ - t^{\alpha} \} ~ dt
\\&\le
(6 \cM^2 + 8 \cB^2) \|\u_i\|_{\psi_\alpha}^2 \exp\left\{-\cM^{\alpha}\right\}
,
\end{aligned}$$
completing the proof of \eqref{Xpp_sq} and hence \eqref{Spp_ineq}.

\item
Putting the pieces together:
combining \eqref{Sk_sep}, \eqref{Sp_ineq} and \eqref{Spp_ineq} we obtain for an arbitrary $u\in (0,\infty)$ that
\beq\label{Sbound1}
\begin{aligned}
&
\PP\left(\max_{1\le n\le N} T_n \ge 2z\right)
\le
\PP\left(\max_{1\le n\le N} T_n'  \ge z\right)	
+
\PP\left(\max_{1\le n\le N} T_n'' \ge z\right)
\\ &\le
\exp\left\{ - \frac{z^2}{8\cM^2} \right\}
+
\frac{6 \cM^2 + 8 \cB^2}{z^2} \exp\left\{-\cM^{\alpha}\right\}
.
\end{aligned}\eeq

We choose $\cM$ as, by making the exponents equal in above,
$$
\cM = \left(\frac{z^2}{8} \right)^{\frac{1}{\alpha+2}} 
\quad\text{such that}\quad
\frac{z^2}{8\cM^2} = \cM^{\alpha}
=\left(\frac{z^2}{8} \right)^{
\frac{\alpha}{\alpha+2}
} 
.
$$
Plugging this $\cM$ back into \eqref{Sbound1} we obtain
\beq\label{Sbound2}
\begin{aligned}
&
\PP\left(\max_{1\le n\le N} T_n \ge 2z\right)
\le
\exp\left\{ - \left(\frac{z^2}{8} \right)^{\frac{\alpha}{\alpha+2} } \right\}
+
\frac{6 \cM^2 + 8 \cB^2}{z^2}
\exp\left\{ - \left(\frac{z^2}{8} \right)^{\frac{\alpha}{\alpha+2} } \right\}
\\ &\le
\left[
1
+
\left( \frac{1}{8}  \right)^{ \frac{2}{\alpha+2} } 
\frac{6}{z^{ \frac{2\alpha}{\alpha+2} } }
+
\left( \frac{3}{\alpha} \right)^{ \frac{2}{\alpha} }  \frac{8}{z^2}
\right]
\exp\left\{ - \left(\frac{z^2}{8} \right)^{\frac{\alpha}{\alpha+2} } \right\}
\end{aligned}\eeq
where we plugged in the expression of $\cB$ in \eqref{beta}.
We can further simplify the square-bracket prefactor in the last line of \eqref{Sbound2} which can be tightly bounded by
$$\begin{aligned}
&
1
+
\left( \frac{1}{8}  \right)^{ \frac{2}{\alpha+2} } 
\frac{6}{z^{ \frac{2\alpha}{\alpha+2} } }
+
\left( \frac{3}{\alpha} \right)^{ \frac{2}{\alpha} }  \frac{8}{z^2}
\le
1
+ \frac{6 \cdot \frac{2}{\alpha+2} }{(8)^{\frac{2}{\alpha+2}}} + \frac{6 \cdot \frac{\alpha}{\alpha+2} }{ (8)^{\frac{2}{\alpha+2}}  z^2 } 
+
\left( \frac{3}{\alpha} \right)^{ \frac{2}{\alpha} }  \frac{8}{z^2}
\\&\le
3
+
\left(
\frac{0.75\cdot \frac{\alpha}{\alpha+2} }{(8)^{\frac{2}{\alpha+2}}  }
+
\left( \frac{3}{\alpha} \right)^{ \frac{2}{\alpha} } 
\right) \frac{8}{z^2}
\le
3
+
\left(
0.75
+
\left( \frac{3}{\alpha} \right)^{ \frac{2}{\alpha} } 
\right) \frac{8}{z^2}
\le
3
+
\left( \frac{3}{\alpha} \right)^{ \frac{2}{\alpha} }  \frac{16}{z^2}
.
\end{aligned}$$
where we used an implication of Jensen's inequality: for $\gamma = \alpha/(\alpha+2) \in (0,1)$ one has $x^\gamma \le 1-\gamma + \gamma x$ for all $x \ge 0$ (where the equality holds for $x = 1$), as well as a few elementary algebraic inequalities, including $\gamma 8^{-\gamma} < 0.177$, $(1 - \gamma) 8^{-\gamma} < 1$, $\left( 3 / \alpha \right)^{ 2 /\alpha}  > 0.78$ for all $\alpha > 0$ and $0 < \gamma = 2 / (\alpha + 2) < 1$.
Thus, \eqref{Sbound_prime} is concluded by noticing the relation \eqref{Tndef} and setting $z/2$ in the place of $z$, which hence proves Theorem \ref{theo:concentration} via the argument in \eqref{rescaling} in our proof.
\end{enumerate}
\end{proof}

\begin{figure}[!tb]
\centering
\includegraphics[width=3.1in]{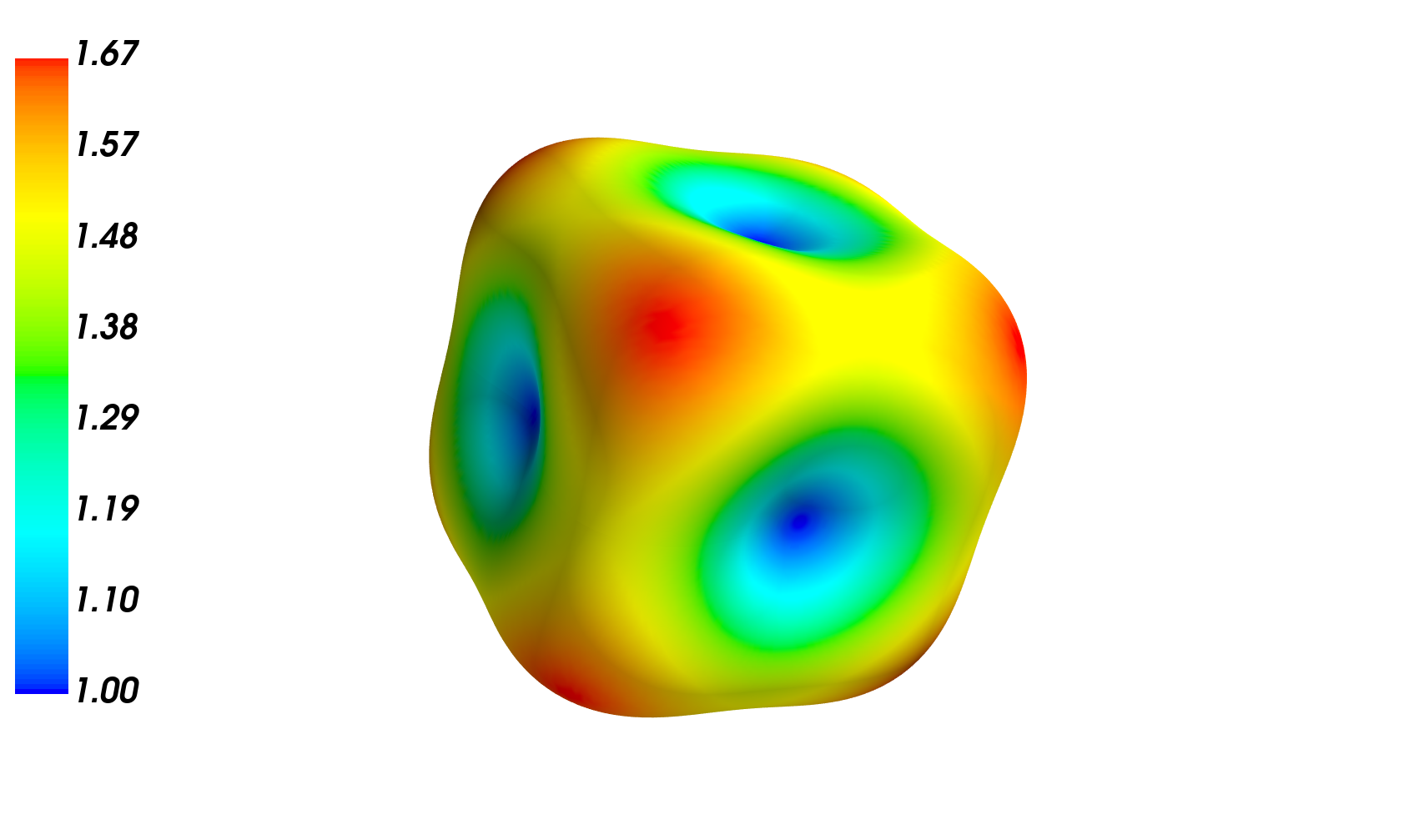}
\includegraphics[width=2.8in]{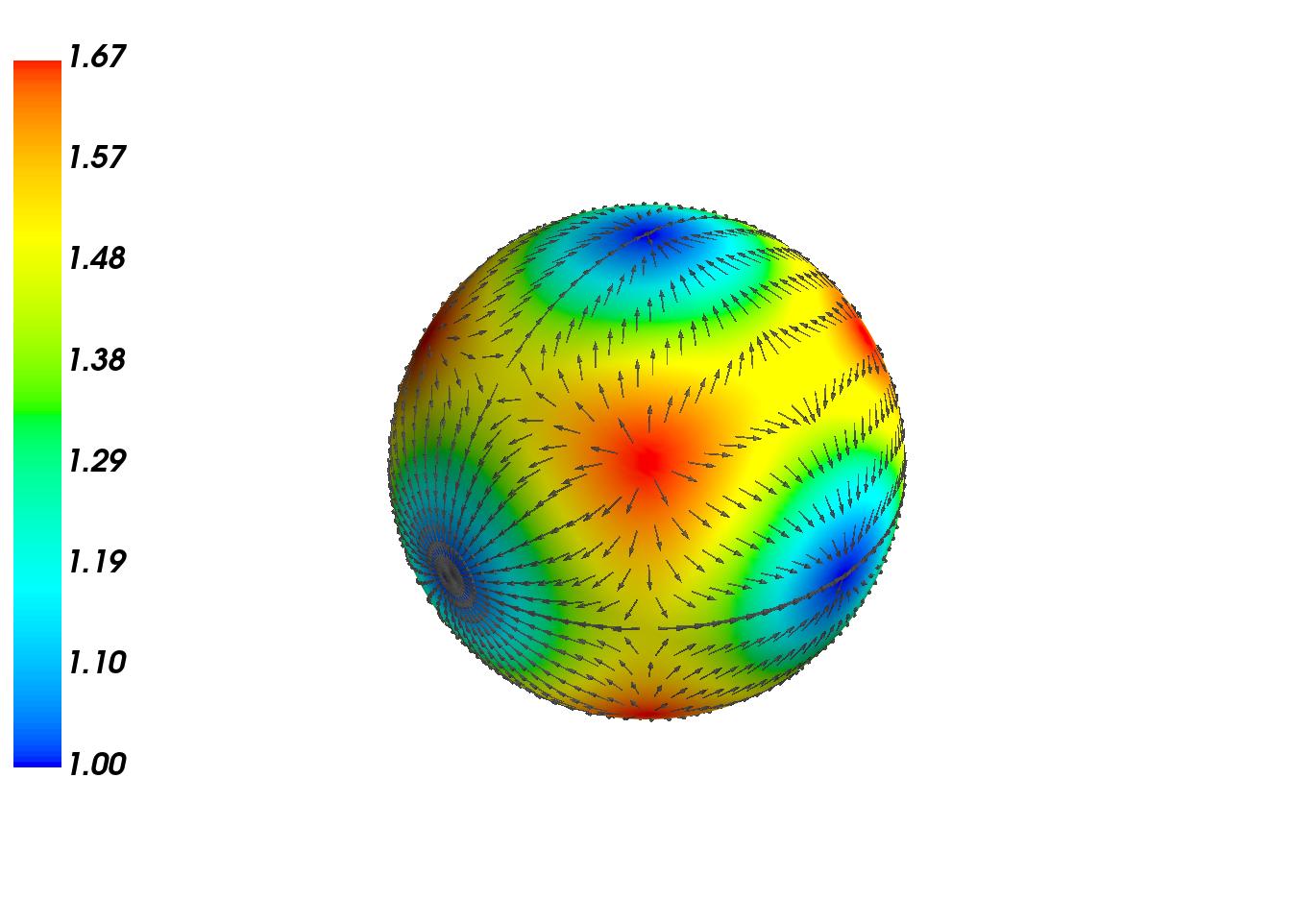}
\caption{Heatmap visualization of the tensorial ICA objective function.
Left: a graph plot where the radius denotes the corresponding function value (a global offset of 2 is added for illustration purposes). 
Right: a quiver plot with arrows denoting the negative gradient direction. 
}
\label{fig:1}
\end{figure}

\pb\section{Visualization and Numerical Experiments for Online Tensorial ICA}\label{sec:experiment}

\begin{figure}
\centering
\hspace{-.5in}
\begin{minipage}[t]{0.47\textwidth}
\includegraphics[width=\textwidth]{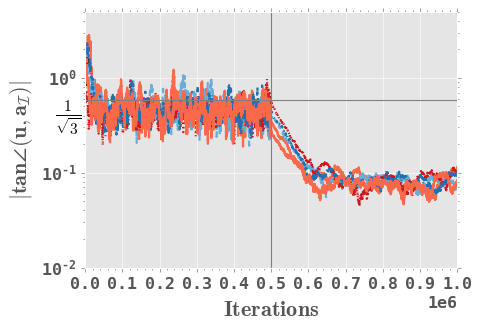}
\label{fig:y equals x}
\end{minipage}
\qquad
\begin{minipage}[t]{0.47\textwidth}
\includegraphics[width=\textwidth]{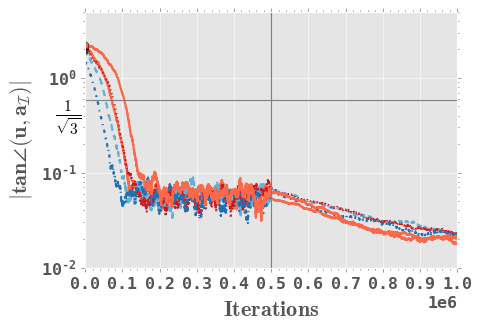}
\label{fig:five over x}
\end{minipage}
\caption{Convergence of our two-phase online tensorial ICA algorithm implemented under two settings.
Left: each independent component being a mixture-Gaussian distribution with $\mu_4 = 2.5$.
Right: each component being a Gaussian-Bernoulli distribution with $\mu_4 = 6$.
Both settings are plotted in the semilog-scale, where the horizontal axis is the iteration number, and the vertical axis is the distance to the closest independent component in metric $\min_{i \in [d]} \left|\tan\angle(\ub^{(T)},\ab_i)\right|$.
}
\label{Fig_1}
\end{figure}

In this section, we provide visualization and conduct a group of experiments in support of our theoretical results. 
In \S\ref{sec:exp_two_phase} we illustrate the two-phase convergence behavior of our algorithm on two most commonly seen ICA problem distributions: Mixture Gaussian and Gaussian-Bernoulli. 
In \S\ref{sec:exp_convergence} we empirically validate our sharp theoretical convergence rate (mainly in both $d$ and $T$) in our Corollary \ref{coro:two_phase}.

\subsection{Two-Phase Algorithm for Online Tensorial ICA}\label{sec:exp_two_phase}

In this subsection, we present an initial experimental result that demonstrates the two-phase convergence behavior of our Algorithm~\ref{algo:ICA}.
We arbitrarily set dimension as $d = 20$ and total sample size as $T = 1\times 10^6$, and the mixing matrix $\Ab$ randomly drawn from the Haar measure over the $d$-dimensional orthogonal group.
We conduct experiments on two separate instances of the one-dimensional distribution $Z_i$ of independent component in ICA: mixture Gaussian with $\mu_4 < 3$, and Gaussian-Bernoulli with $\mu_4 > 3$, to validate the convergence theory of our two-phase algorithm.

\begin{enumerate}[label=(\alph*)]
\item\textbf{Mixture Gaussian}
We adopt $Z = \delta Y_1 + (1-\delta) Y_2$ as the mixture Gaussian component of Gaussian variables $Y_1, Y_2$ with 
$$
Y_1 \sim N\left(-\frac{1}{\sqrt{2}}, \frac12\right)
,
\qquad
Y_2 \sim N\left(\frac{1}{\sqrt{2}}, \frac12\right)
,
\qquad
\delta \sim \text{Bernoulli}(1/2)
\quad
\text{independently}
.
$$
It is easy to verify that the distribution of $Z_i$ satisfies Assumption~\ref{assu:distribution} with $\mu_4$ = 2.5.
%
\item
\textbf{Gaussian-Bernoulli}
where we adopt $Z = \delta Y$ with
$$
Y \sim N(0, 2)
, 
\qquad
\delta \sim \text{Bernoulli}(1/2)
\quad
\text{independently}
.
$$
It is straightforward to verify that $\mu_4 = 6$ in this case.
\end{enumerate}
The algorithm is initialized at $\ub_0$ uniformly drawn from the unit sphere $\cD_1$, and we run our two-phase algorithm scheduled as in Corollary \ref{coro:two_phase}. 
That is, in the first half of the training when the number of iterates $\le T/2$ we choose stepsize as $\frac{8d}{|\mu_4 - 3|T}$ and in the second half choose stepsize as $\frac{9}{|\mu_4 - 3|T}$, both omitting logarithmic factors.
We measure the convergence by the tangent of the angle between $\ub^{(T)}$ and its closest independent component $\ab_i$.
The resulting 5 independent runs are shown in Figure~\ref{Fig_1}.
(The horizontal line represents $y = \frac{1}{\sqrt{3}}$ which is the barrier of the warm region, and the vertical line is $x = T/2$ which is the separation between two phases.)
Our result in Figure~\ref{Fig_1} exemplifies consistency with the main theoretical result of our paper, especially the algorithm's two-phase demonstration.
Regardless of the initialization all 10 independent runs share similar trajectories in terms of our measure: in Phase I the algorithm decays fast until oscillating into the warm region (the absolute tangent value below $\frac{1}{\sqrt{3}}$), and in Phase II the algorithm continue to converge linearly until reaching the desired accuracy.

\pb\subsection{Validating the Finite-Sample Convergence Rate}\label{sec:exp_convergence}
In this subsection, we validate our main convergence rate result, Corollary~\ref{coro:two_phase}, via simulations for a range of values $d$ and $T$ satisfying the scaling condition
\beq\label{eq1}
d\ge 2\sqrt{2\pi e} \log\epsilon^{-1} + 1
\quad\text{and} \quad
\frac{d^4}{T}\le C\epsilon^2
,
\eeq
so with probability $\ge 1-O(\epsilon)$ the following $\tilde{\mathcal{O}}(\sqrt{d/T})$-convergence rate result holds for all sufficiently large $d$ and $T$:
\beq\label{eq2}
|\tan\angle (\ub^{(T)}, \ab_{\mathcal{I}})|\le C\sqrt{\frac{d}{T}}
,
\eeq
where $C$ includes a polylogarithmic factor in $d, T, \epsilon$.

We continue to generate random data $\bZ$ as we did in the Gaussian-Bernoulli distributions case in~\S\ref{sec:exp_two_phase}, with parameter value of either $d$ or $T$ allowed to vary in each run while freezing the other.
By ergodicity we may empirically estimate $\EE|\tan \angle (\ub^{(T)}, \ab_{\mathcal{I}})|$ by averaging the last few iterates of $|\tan\angle (\ub^{(T)}, \ab_{\mathcal{I}})|$
---
the final $0.6T$ iterates are taken.
Then we scatter-plot the empirical $\EE|\tan\angle (\ub^{(T)}, \ab_{\mathcal{I}})|$ under a range of parameters in the figure, gauged by the theoretical rate of~\eqref{eq2}.
We first run for $T/2$ epochs with the stepsize propotional to $8d/T$ and run for $T/2$ epochs with the stepsize propotional to $9/T$, as shown in \S\ref{sec:exp_two_phase} and in Corollary~\ref{coro:two_phase}, forgoing the logarithmic factor.
The range of values one takes is $d\in \{7, 12, 20, 33, 54\}$, $T\in \{2\times 10^2, 2\times 10^3, 1\times 10^4, 5\times 10^4, 2\times 10^5, 1\times 10^6\}$.
The horizontal axis is the iteration number $T$ and the vertical axis is $|\tan\angle (\ub^{(T)}, \ab_{\mathcal{I}})|$, both are plotted in log-log scale.
We discuss the simulation results to illustrate the dependency as follows:

\begin{figure}[!tb]
\centering
\includegraphics[width=0.48\textwidth]{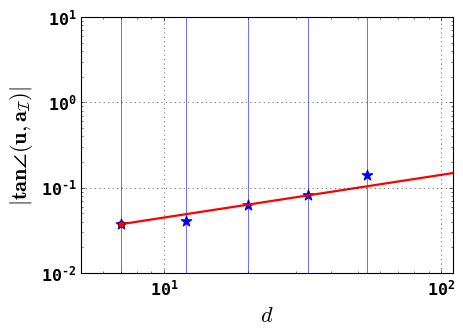}
\hfill
\includegraphics[width=0.48\textwidth]{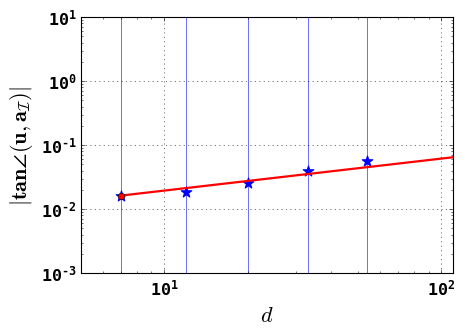}
\caption{The dependency of the convergence rate on $d$ for given $T$. Left: $T = 2 \times 10^5$. Right: $T = 1 \times 10^6$.}
\label{fig:3}
\end{figure}

\begin{figure}[!tb]
\centering
\includegraphics[width=0.48\textwidth]{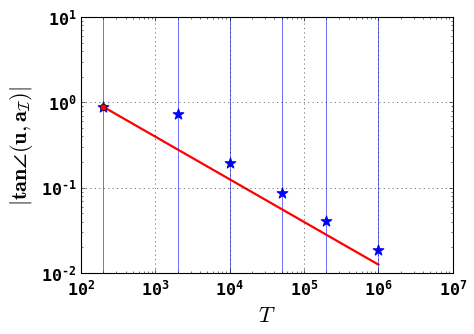}
\hfill
\includegraphics[width=0.48\textwidth]{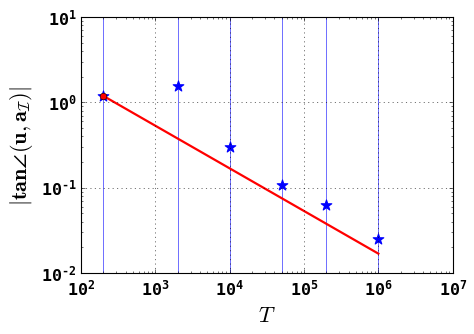}
\caption{The dependency of the convergence rate on $T$ for given $d$. Left: $d = 12$. Right: $d = 20$.}
\label{fig:4}
\end{figure}

\begin{enumerate}[label=(\alph*)]
\item
Figure~\ref{fig:3} shows the dependence of the convergence rate on $d$ for a given $T$, where the gauging red line passes through $d = 7$ and has a slope of $1/2$, representing the theoretical convergence dependency of $\tilde{O}(d^{1/2})$ in $d$.
We observe that smaller $d$s aligns well with the scaling condition (when $T = 2 \times 10^5$ and $d \leq 33$, and when $T = 1 \times 10^6$ and $d \leq 54$) and achieve the dependency of $\tilde{O}(d^{1/2})$ in $d$;
\item
Figure~\ref{fig:4} shows the dependence of the convergence rate on $T$ for a given $d$, where the gauging red line has a slope of $-1/2$, representing the theoretical convergence dependency of $\tilde{O}(T^{-1/2})$ in $T$.
We observe that larger $T$s aligns well with the scaling condition and achieve the dependency of $\tilde{O}(T^{-1/2})$ in $T$.
\end{enumerate}
Hence the Figures~\ref{fig:3} and~\ref{fig:4} together validate our $\tilde{O}(\sqrt{d/T})$-convergence rate of the online tensor ICA algorithm.

\tableofcontents
\end{document}